%% file: main.tex
\documentclass[11pt]{article}
\usepackage[utf8]{inputenc} 
\usepackage[T1]{fontenc}    

\usepackage{amsmath,amssymb,amsthm,mathrsfs,pgf,tikz,caption,subcaption,mathtools,mathabx}
\usepackage[margin=1in]{geometry}
\usepackage[ruled,vlined,linesnumbered]{algorithm2e}
\usepackage{natbib}
\usepackage[pagebackref,colorlinks=true,urlcolor=blue,linkcolor=blue,citecolor=blue,pdfstartview=FitH]{hyperref}
\usepackage{cleveref}
\usepackage{url}            
\usepackage{booktabs}       
\usepackage{amsfonts}       
\usepackage{nicefrac}       
\usepackage{microtype}      
\usepackage{times}
\usepackage{bbm}
\usepackage{enumitem}
\usepackage{xcolor}
\usepackage{mdframed}

\usepackage{makecell}

\usepackage{thmtools}
\usepackage{thm-restate}

\input{macro}

\input{specific_macro}
\title{Approximate Replicability in Learning}
\author{Max Hopkins\thanks{Institute for Advanced Study, Princeton. nmhopkin@ias.edu. Supported by NSF Award DMS-2424441}, Russell Impagliazzo\thanks{University of California, San Diego. rimpagliazzo@ucsd.edu. Supported by NSF Award AF: Medium 2212136},
Christopher Ye\thanks{University of California, San Diego. czye@ucsd.edu. Supported by NSF Award AF: Medium 2212136 and HDR TRIPODS Phase II grant 2217058 (EnCORE Institute).}}

\begin{document}

\pagenumbering{gobble}
\maketitle
\begin{abstract}
    \input{abstract}
\end{abstract}

\newpage
\tableofcontents
\newpage

\pagenumbering{arabic}
\input{intro}

\input{prelims}

\input{repl_prediction}

\input{approx_replicability}

\input{semi_supervised_replicability}

\section*{Acknowledgements}

We would like to thank Arsen Vasilyan for many fruitful discussions throughout the development of this work both regarding approaches to improving the sample complexity of our algorithms, as well as toward designing proper approximately replicable learners. We also thank Mark Bun, Rex Lei, Satchit Sivakumar, and Shay Moran for helpful conversations regarding relaxed notions of replicability and their connection to differential privacy.

Finally, we are deeply thankful to several anonymous reviewers for their many thoughtful suggestions and for pointing out an error in the analysis of an earlier version of \Cref{thm:predict-formal}.

\bibliographystyle{alpha}  
\bibliography{references} 

\newpage
\appendix

\input{appendix_apx_repl}

\input{appendix_omitted}

\end{document}

%% file: macro.tex
\newtheorem{theorem}{Theorem}[section]
\newtheorem{corollary}[theorem]{Corollary}
\newtheorem{proposition}[theorem]{Proposition}
\newtheorem{lemma}[theorem]{Lemma}
\newtheorem{definition}[theorem]{Definition}

\newtheorem{claim}[theorem]{Claim}

\newcommand{\R}{\mathbb{R}}

\newcommand{\N}{\mathbb{N}}

\newcommand{\distribution}{\mathcal{D}}
\newcommand{\domain}{\mathcal{X}}
\newcommand{\range}{\mathcal{Y}}
\newcommand{\hypotheses}{\mathcal{H}}

\newcommand{\calM}{\mathcal{M}}

\newcommand{\sign}{{\rm sign}}
\newcommand{\poly}{{\rm poly}}
\newcommand{\polylog}{{\rm polylog}}

\newcommand{\ind}{\mathbbm{1}}

\newcommand\norm[1]{\left\lVert#1\right\rVert}

\newcommand{\Acal}{\mathcal{A}}

\newcommand{\var}{\text{Var}}
\newcommand{\eps}{\varepsilon}
\newcommand{\E}{\mathbb{E}}

\newcommand{\dx}{{\rm d}x}

\newcommand{\err}{\mathrm{err}}
\newcommand{\errD}{\mathrm{err}_{\distribution}}
\newcommand{\errS}{\mathrm{err}_{S}}
\newcommand{\acc}{\mathrm{acc}}

\newcommand{\opt}{\mathrm{OPT}}
\newcommand{\tvd}[1]{\mathrm{tvd}(#1)}

\newcommand{\bigO}[1]{O\left( #1 \right)}
\newcommand{\bigOm}[1]{\Omega\left( #1 \right)}
\newcommand{\tO}[1]{\tilde{O}\left( #1 \right)}
\newcommand{\bigtO}[1]{\tilde{O}\left( #1 \right)}
\newcommand{\tOm}[1]{\tilde{\Omega}( #1 )}
\newcommand{\bigtOm}[1]{\tilde{\Omega}\left( #1 \right)}
\newcommand{\litO}[1]{o\left( #1 \right)}

\newcommand{\dist}{\mathrm{dist}}
\newcommand{\distD}{\dist_{\distribution}}
\newcommand{\cand}{\mathrm{cand}}
\newcommand{\given}{\textrm{ s.t. }}
\newcommand{\andT}{\textrm{ and }}

\newcommand{\set}[1]{\{#1\}}

\newcommand{\algT}[1]{\textsc{\bfseries \footnotesize #1}}

\newcommand{\innerAlg}{\Acal}

\newcommand{\iid}{i.i.d.\ }
\newcommand{\wrt}{w.r.t.\ }
\newcommand{\whp}{w.h.p.\ }

\newcommand{\BinomD}{{\rm Binom}}
\newcommand{\UnifD}{{\rm Unif}}

%% file: specific_macro.tex
\newcommand{\expMechScoreScale}{t}
\newcommand{\priv}{{\rm priv}}
\newcommand{\pub}{{\rm pub}}

\newcommand{\LDim}{{\rm LDim}}

\newcommand{\prHeavyHitters}{\algT{PointwiseReplicableHeavyHitters}}
\newcommand{\hhLabel}{\algT{HeavyHitterLabels}}
\newcommand{\basicNonUniformRep}{\algT{PointwiseReplicability}}
\newcommand{\rLearnerFinite}{\algT{PointwiseReplicabilityFiniteRandomness}}
\newcommand{\nonUniformtoBias}{\algT{BiasEstimatorPointwise}}

\newcommand{\hypothesisSelection}{\algT{HypothesisSelection}}

\newcommand{\apxReplRealizable}{\algT{RealizableApproximateReplicability}}
\newcommand{\thresholdLearner}{\algT{ThresholdLearner}}
\newcommand{\privSelection}{\algT{PrivateSelection}}
\newcommand{\semiRepltoSemiPriv}{\algT{SemiPrivateReduction}}
\newcommand{\semiPrivDimRed}{\algT{SemiPrivateHardnessAmplification}}
\newcommand{\apxReplDimRed}{\algT{ApxReplicabilityHardnessAmplification}}
\newcommand{\corrSamp}{\algT{CorrSamp}}

\newcommand{\game}{\mathcal{G}}

\newcommand{\leftT}{{\rm left}}
\newcommand{\rightT}{{\rm right}}
\newcommand{\up}{{\rm up}}
\newcommand{\down}{{\rm down}}

\newcommand{\Rad}{{\rm Rad}}

%% file: abstract.tex
Replicability, introduced by (Impagliazzo et al. STOC '22), is the notion that algorithms should remain stable under a resampling of their inputs (given access to shared randomness).
While a strong and interesting notion of stability, the cost of replicability can be prohibitive: there is no replicable algorithm, for instance, for tasks as simple as threshold learning (Bun et al. STOC '23).
Given such strong impossibility results we ask: under what approximate notions of replicability is learning possible?

In this work, we propose three natural relaxations of replicability in the context of PAC learning:
\begin{enumerate}
    \item \textbf{Pointwise:} the learner must be consistent on any \textit{fixed} input, but not across all inputs simultaneously.
    \item \textbf{Approximate:} the learner must output hypotheses that classify most of the distribution consistently.
    \item \textbf{Semi:} the algorithm is fully replicable, but may additionally use shared unlabeled samples.
\end{enumerate}
In all three cases, for constant replicability we obtain close to sample-optimal agnostic PAC learners: 1) and 2) are achievable using $O(d/\alpha^2 + 1/\alpha^{4})$ samples, while 3) requires $\Theta(d^2/\alpha^2)$ labeled samples.

%% file: intro.tex
\section{Introduction}


\textit{Replicability} is the notion that statistical methods should remain stable under fresh samples from the same population. 
Recently, \cite{ImpLPS22} introduced a formal framework of replicability in learning, stating that for (most) random seeds, running a \emph{replicable} algorithm over fresh samples should produce the same result:
\[
\forall \distribution: \Pr_{S,S'\sim \distribution^m, r\in R}[A(S;r)=A(S';r)] \geq 1-\rho.
\]
Unfortunately, \cite{ImpLPS22}'s replicability, while powerful, is highly restrictive. \cite{bun2023stability}, for instance, showed the notion is quantitatively equivalent to approximate differential privacy, ruling out a broad family of tasks as basic as learning 1D-Thresholds. Moreover, even for tasks where replicability is possible (such as mean estimation), achieving constant replicability often comes with quadratic overhead \cite{bun2023stability,hopkins2024replicability}, rendering the notion infeasible in high dimensional settings.

In this work, we introduce three weakened notions of replicability, \textit{pointwise} replicability, \textit{approximate} replicability, and \textit{semi}-replicability, that remain powerful enough to be of potential use in application, yet apply to a substantially broader domain of problems. In particular, we show in these settings that any standard (non-replicable) learning algorithm can be `boosted' to one satisfying weak replicability, along with corresponding lower bounds on the overhead required to do so.

\paragraph{Pointwise Replicability}
Our first notion of study is a basic relaxed variant of replicability we call \textit{pointwise} replicability,
which asks that the behavior of our algorithm is replicable on any \textit{fixed} point in the domain (rather than over all such points simultaneously):
\begin{definition}[Pointwise Replicability]
    \label{def:non-uniform-repl}
    We say an algorithm $\Acal$ is $\rho$-pointwise replicable if:
    \[
    \forall \distribution, \forall x \in X: \Pr_{S,S' \sim \distribution^m, r \sim R}[\Acal(S;r)(x) \neq \Acal(S';r)(x)] < \rho
    \]
\end{definition}


This definition is inspired by the notion of private prediction of Dwork and Feldman \cite{dwork2018privacy}, where privacy is only required over revealing the label of a single queried point in the domain (rather than revealing the full model). 
Similarly, one can think of our notion as \textit{replicable prediction}: given a test point $x$, such an algorithm will almost always give the same predicted label. 
We give a generic algorithm that turns any (possibly non-replicable) learner into a pointwise replicable predictor at the cost of running the original algorithm $O(\rho^{-2})$ times, and show that this is optimal up to polylogarithmic factors.\footnote{Formally, here and in the following approximate setting, we also incur an extra additive $O(\frac{1}{\alpha^4\rho^2})$ cost from separately learning heavy elements in the underlying distribution. This cost is comparatively negligible unless $\alpha$ is very small.} 

\paragraph{Approximate Replicability}
Our second notion, approximate replicability, asks that the outputs remain the same on ``most'' (rather than all) points in the output domain simultaneously.

\begin{definition}[Approximate Replicability]
    We say an algorithm $\Acal$ is $(\rho, \gamma)$-approximately replicable if:
    \[
    \forall \distribution: \Pr_{S,S' \sim \distribution^{m}, r \sim R}\left[ \Pr_{x \sim \distribution_{\domain}}[\Acal(S;r)(x) \neq \Acal(S';r)(x)] > \gamma\right] < \rho
    \]
    where $\distribution_{\domain}$ is the marginal distribution over unlabeled data.
    We say an algorithm is $\rho$-approximately replicable if it is $(\rho, \rho)$-approximately replicable.
\end{definition}

This definition is closest to other approximate notions proposed in the literature, and, in fact, was explicitly raised in \cite{chase2023replicability} as a potential direction of study. This notion can easily be generalized beyond PAC learning settings by equipping the output space with an appropriate metric (e.g. the same metric used to evaluate correctness). 
For example, if the task is $\ell_{2}$ mean estimation, an algorithm is $(\rho, \gamma)$-approximately replicable if it outputs two estimates within $\gamma$ in $\ell_{2}$-norm.
In this particular instance (and others with similarly unique solutions), it is clear one can achieve $(\rho, \alpha)$-approximate replicability essentially for ``free'' since any two estimates within $\alpha$ of the true mean are also within $2 \alpha$ of each other, yielding a large class of problems where the natural algorithm is immediately approximately replicable.

Approximate replicability becomes more interesting for complex tasks like PAC learning where two optimal hypotheses can disagree on much of the domain.\footnote{In fact, in this case even in the realizable PAC setting where all optimal hypotheses are close, approximate replicability is still not immediate from accuracy since we'd like to ensure replicability over \textit{all} input distributions (including non-realizable ones).} 
Nevertheless, building on our pointwise replicable algorithm, we show any learner can be transformed into a $\rho$-approximately replicable one by running it $\tO{\rho^{-2}}$ times.\footnote{We use $\tilde{O}$ to suppress polylogarithmic factors.}
Similarly to pointwise replicability, we show that our sample complexity is optimal up to polylogarithmic factors.

\paragraph{Semi-Replicable Learning}

Our final relaxation of replicability, the semi-replicable model, diverges from the above in still requiring identical outputs, but gives the algorithm access to an additional pool of \emph{shared, unlabeled} samples (in addition to shared randomness).

\begin{definition}[Semi-Replicability]
We say an algorithm is $\rho$-semi replicable if:
    \[
    \forall \distribution: \Pr_{S_U \sim \distribution_{\domain}^{m}, S,S' \sim \distribution^{m}, r\in R}[A(S; S_{U}, r)=A(S'; S_{U}, r)] \geq 1-\rho 
    \]
where $\distribution_{\domain}$ is the marginal distribution over unlabeled data.
\end{definition}

The semi-replicable model is inspired by the well-studied notion of \textit{semi-privacy} \cite{beimel2013private} and has several interesting motivations. First, the notion is reasonably practical; public unlabeled data is plentiful and can easily be shared between different research groups. For instance, fine-tuning machine learning models to a specific application (where we require algorithmic stability) can begin from a large, shared public model.


The model can also be thought of as a natural interpolation between \textit{distribution-dependent} PAC learning where the unlabeled data distribution is known (and every learnable class is easily shown to be replicably learnable by seminal results of \cite{benedek1991learnability} and \cite{ImpLPS22,bun2023stability}) and the standard \textit{distribution-free} PAC model where many learnable classes are not replicably learnable at all. We observe one does not need to know the entire underlying marginal --- just a few public unlabeled samples (shared data about the marginal distribution) suffices to allow replicable learning for every learnable class (albeit with a necessary quadratic blowup).


\subsection{Our Contributions}


We now give a more detailed exposition of our main contributions. We start with some brief background, and refer the reader to \Cref{sec:prelims} for further details. See \Cref{table:results} for a quick informal summary of our results.

We primarily study (binary) classification in the context of PAC learning \cite{valiant1984theory}.
Given sample access to a labeled distribution $\distribution$ over $(x, y) \in \domain \times \set{\pm 1}$, we are asked to produce a hypothesis $h: \domain \rightarrow \set{\pm 1}$.
The classification error of $h$, denoted $\errD(h)$, is the probability a point is mislabeled, i.e. $\Pr_{\distribution}(h(x) \neq y)$.
An algorithm is an (agnostic) $(\alpha, \beta)$-learner for a hypothesis class $\hypotheses$ if for every distribution $\distribution$, $\Pr(\errD(\innerAlg(S)) > \opt + \alpha) < \beta$ where $\opt := \inf_{h \in \hypotheses} \errD(h)$ is the optimal error of the class.
In the realizable setting, we are promised that $\distribution$ is supported in $(x, h(x))$ for some $h \in \hypotheses$.
A proper learner outputs $\hat{h} \in \hypotheses$, while an improper learner may output an arbitrary function $\hat{h}: \domain \rightarrow \set{\pm 1}$.
In this standard setting, learnability of a class is tightly parameterized by its VC Dimension.
In particular, any class $\hypotheses$ with VC Dimension $d$ has an agnostic $(\alpha, \beta)$-learner with $\bigO{(d + \log(1/\beta))\alpha^{-2}}$ samples \cite{vapnik1974theory, blumer1989learnability} and an (improper) realizable $(\alpha, \beta)$-learner with $\bigO{(d + \log(1/\beta))\alpha^{-1}}$ samples \cite{hanneke2016optimal, larsen2023bagging}.
Throughout the remainder of the introduction, we omit $\polylog((\alpha \beta \rho \gamma)^{-1})$ terms for simplicity.


\begin{table}[ht]
\renewcommand*{\arraystretch}{2.5}
\centering
\begin{tabular}{|c|c|c|c|}
\hline
 & \textbf{Upper Bound} & \textbf{Lower Bound} & \textbf{Notes} \\
\hline
\textbf{Pointwise} &
$\frac{d}{\rho^2 \alpha^2} + \frac{1}{\rho^2 \alpha^{4}}$ [\ref{thm:predict}] &
$\frac{d}{\rho^2 \alpha^2}$ [\ref{thm:non-uniform-repl-lb}] &
Agnostic \\[5pt]
\cline{2-4}
& 
$\frac{d}{\rho^2 \alpha} + \frac{1}{\rho^2 \alpha^2}$ [\ref{thm:predict}] & 
$\frac{d}{\rho^2 \alpha}$ [\ref{thm:non-uniform-repl-lb}] & 
\makecell{Realizable} \\[5pt]
\hline
\textbf{Approximate} & 
$\frac{d}{\gamma^2 \alpha^2} + \frac{1}{\poly(\rho \gamma \alpha)}$ [\ref{thm:apx-repl-informal}] & 
$\frac{d}{(\rho + \gamma)^2 \alpha^2}$ [\ref{thm:approx-repl-d-lb}] &
Agnostic
\\[10pt]
\cline{2-4}
& 
$\frac{d}{\rho^2 \min(\alpha, \gamma)}$ [\ref{thm:apx-repl-informal}] & 
& 
\makecell{Proper, Realizable} \\[5pt]
\cline{2-4}
& 
$\frac{1}{\rho^2 \alpha^2} + \frac{1}{\gamma^2}$ [\ref{prop:thresholds-proper-informal}] & $\frac{1}{(\rho + \gamma)^2 \alpha^2}$ [\ref{thm:approx-repl-d-lb}] 
& 
\makecell{Proper, Agnostic \\
Threshold Learner} \\[5pt]
\hline
\textbf{Semi-Replicable} & 
 \makecell{$\frac{d^2}{\rho^2 \alpha^2}$ labeled samples \\[5pt] $\frac{d}{\alpha}$ shared samples [\ref{thm:semi-supervised-repl-ub}]} & 
 \makecell{ $\frac{d^2}{\rho^2 \alpha^2}$ labeled samples [\ref{thm:semi-supervised-repl-lb}] \\[5pt]
 $\frac{d}{\alpha}$ shared samples [\ref{thm:semi-supervised-repl-shared-lb}]} &
Proper, Agnostic \\[10pt] 
\hline
\end{tabular}
\caption{\small{Table of our results. Polylogarithmic factors and dependence on $\beta$ omitted for simplicity. All algorithms incur an additional $\polylog((\alpha \beta \rho  \gamma)^{-1})$-factor in sample complexity. 
}}
\label{table:results}
\end{table}

\paragraph{Pointwise Replicability}

Our first main result is a generic procedure for turning any standard learner into a pointwise replicable one.\footnote{For measure theoretic reasons, \Cref{thm:predict} assumes that the underlying domain is countable. In particular, one step of our algorithm requires applying a Chernoff/Hoeffding bound over the domain.}

\begin{restatable}[Informal \Cref{thm:predict-formal}]{theorem}{ReplPrediction}
    \label{thm:predict}
    Let $\Acal$ be an agnostic $(\alpha, \beta)$-learner with $m(\alpha, \beta)$ samples.
    There exists an agnostic $\rho$-pointwise replicable $(\alpha, \beta)$-learner over countable domains with sample complexity $\tO{m(\alpha, \rho^2 \beta)\rho^{-2} + \rho^{-2} \alpha^{-4})}$.
    Our algorithm runs in time linear in sample complexity with $\bigO{\rho^{-2}}$ oracle calls to $\Acal$.
    
    Furthermore, in the realizable setting, there is a $\rho$-pointwise replicable $(\alpha, \beta)$-learner with sample complexity $\tO{m(\alpha, \rho^2 \beta) \rho^{-2} + \rho^{-2} \alpha^{-2})}$
    where $m(\alpha, \beta)$ is the sample complexity of a realizable $(\alpha, \beta)$-learner.
\end{restatable}

In the realizable setting, we obtain an algorithm with sample complexity $\tO{d \rho^{-2} \alpha^{-1} + \rho^{-2} \alpha^{-2}}$.
In the agnostic setting, we obtain an algorithm with sample complexity $\tO{d \rho^{-2} \alpha^{-2} + \rho^{-2} \alpha^{-4})}$ \cite{vapnik1974theory, blumer1989learnability}.
The following lower bound shows that both bounds are essentially optimal when the VC dimension $d$ is large and the accuracy parameter $\alpha$ is high (e.g.\ $d \gg 1/\alpha^2$).
In particular, our algorithm is sample optimal when the accuracy parameter $\alpha$ is constant.
We remark that our lower bound holds for any VC class with dimension $d$.






\begin{restatable}{theorem}{ReplPredictionLB}
    \label{thm:non-uniform-repl-lb}
    Let $\hypotheses$ be any class with VC dimension $d$.
    Any $\rho$-pointwise replicable $(0.01 \alpha, 0.0001)$-learner requires $\tOm{d \rho^{-2} \alpha^{-2}}$ samples in the agnostic setting and $\tOm{d \rho^{-2} \alpha^{-1}}$ samples in the realizable setting.
    Furthermore, any pointwise-replicable learner requires shared randomness.
\end{restatable}

It remains an interesting open question to determine whether the additive $\rho^{-2} \alpha^{-O(1)}$ term is necessary in the sample complexity, or whether it is possible to obtain a $\rho$-pointwise replicable learner with $O(m(\alpha, \beta) \rho^{-2})$ samples.


\paragraph{Approximate Replicability}

Building on our pointwise replicable procedure, we present a transformation turning any standard learner into an approximately replicable one: 

\begin{restatable}[Informal \Cref{thm:apx-repl-chernoff} and \Cref{thm:realizable-apx-repl}]{theorem}{ApxRepl}
    \label{thm:apx-repl-informal}
    Let $\innerAlg$ be an (agnostic) $(\alpha, \beta)$-learner on $m(\alpha, \beta)$ samples.
    There exists an (agnostic) $\rho$-approximately replicable $(\alpha, \beta)$-learner with sample complexity $\tO{m(\alpha, \beta)\rho^{-2} + \rho^{-5} + \rho^{-2} \alpha^{-4} + \rho^{-4} \alpha^{-2}}$.
    Our algorithm runs in time linear in sample complexity with $O(\rho^{-2})$ oracle calls to $\innerAlg$.

    Furthermore, there is a proper $\rho$-approximately replicable realizable $(\alpha, \beta)$-learner with sample complexity $\tO{d \rho^{-2} \alpha^{-1} + d \rho^{-3}}$.
\end{restatable}




When $\rho$ is constant, our proper realizable learner obtains the sample complexity $\tO{d \alpha^{-1}}$.
In particular, we obtain a $0.01$-approximately replicable realizable $(\alpha, \beta)$-learner with identical sample complexity as a generic realizable learner (both in the proper and improper cases).

Furthermore, we present a lower bound showing that our approximately replicable agnostic learner is sample-optimal whenever the VC dimension is sufficiently large (e.g.\ $d \gg \rho^{-3} \alpha^{-2}$).
As is the case with pointwise replicability, our lower bound holds for any VC class with dimension $d$.

\begin{restatable}{theorem}{ApproxReplicabilityLB}
    \label{thm:approx-repl-d-lb}
    Let $\rho < 0.001$ and $\hypotheses$ be an arbitrary class with VC dimension $d$.
    Any $(\rho, \gamma)$-approximately replicable$(\alpha, 0.001)$-learner for $\hypotheses$ requires sample complexity $\bigtOm{d \alpha^{-2} (\rho + \gamma)^{-2}}$.
    In particular, any $\rho$-approximately replicable $(\alpha, 0.001)$-learner requires sample complexity $\bigtOm{d \alpha^{-2} \rho^{-2}}$.
\end{restatable}

Determining the exact sample complexity of approximately replicable PAC learning remains an intriguing open question.
In particular, when the underlying hypothesis class is simple (e.g. $d = O(1)$), is there a $\rho$-approximately replicable learner with sample complexity $\tO{\rho^{-2} \alpha^{-2}}$?
Toward resolving this question, we give a simple direct algorithm for the fundamental class of $1$-dimensional thresholds on $\R$, that is hypotheses of the form $h(x) \mapsto \ind[x > a]$. 
Thresholds pose a fundamental obstacle to strict replicability: there is no replicable algorithm for learning thresholds \cite{bun2023stability, alon2019private}. Nevertheless, we give a proper, approximately replicable learner for the class with optimal sample complexity.



\begin{proposition}[Informal \Cref{prop:thresholds-proper-formal}]
    \label{prop:thresholds-proper-informal}
    Let $\rho, \gamma, \alpha > 0$ and $0 < \beta < \rho$.
    There is an agnostic proper $(\rho, \gamma)$-approximately replicable $(\alpha, \beta)$-learner for thresholds with sample complexity $\tO{\gamma^{-2} + \rho^{-2} \alpha^{-2}}$.
    In particular, there is a $\rho$-approximately replicable $(\alpha, \beta)$-learner with $\tO{\rho^{-2} \alpha^{-2}}$ samples.
\end{proposition}



\paragraph{Semi-Replicable Learning}

In our final model, semi-replicability, we give upper and lower bounds that are optimal (up to polylogarithmic factors) in both the shared and labeled sample complexities.
\begin{theorem}[Informal \Cref{thm:semi-supervised-repl-ub-formal}]
    \label{thm:semi-supervised-repl-ub}
    There is a $\rho$-semi-replicable $(\alpha, \beta)$-learner for with shared (unlabeled) sample complexity $\tO{d \alpha^{-1}}$ and labeled sample complexity $\tO{d^2 \rho^{-2} \alpha^{-2}}$.
\end{theorem}

We give a lower bound for the number of labeled samples (shared or otherwise) showing \Cref{thm:semi-supervised-repl-ub} is tight up to log factors regardless of the number of unlabeled samples.

\begin{restatable}{theorem}{SemiReplicableLB}
    \label{thm:semi-supervised-repl-lb}
    Any $\rho$-semi-replicable $(\alpha, 0.0001)$-learner for any class $\hypotheses$ with VC dimension $d$ requires sample complexity $\tOm{d^{2} \rho^{-2} \alpha^{-2}}$.
    This bound holds regardless of unlabeled sample complexity.
\end{restatable}

Finally, we give a lower bound for the shared sample complexity that holds regardless of the overall sample complexity.

\begin{restatable}{theorem}{SemiReplicableSharedLB}
    \label{thm:semi-supervised-repl-shared-lb}
    There exists a class $\hypotheses$ such that any $0.0001$-semi-replicable $(\alpha, 0.0001)$-learner for $\hypotheses$ requires $\tOm{d \alpha^{-1}}$ shared samples.
    The lower bound holds even if the shared samples are labeled.
\end{restatable}

On our way to proving this result, we also show that a similar lower bound holds for semi-private learners.
In particular, we exhibit a class where even semi-private learners requires $\bigOm{d \alpha^{-1}}$ public samples \cite{alon2019limits}.
It remains an interesting open question to determine whether \emph{every} class with infinite Littlestone Dimension and VC dimension $d$ requires $\bigOm{d \alpha^{-1}}$ shared (resp. public) samples or if there are some classes that can be learned with fewer shared (resp. public) samples.


\paragraph{Shared Randomness}

It is well known that replicability requires shared randomness: for tasks as simple as bias estimation it is not possible to achieve high probability replication without shared random bits \cite{ImpLPS22}, and this impossibility result is now known to extend to a broad class of problems \cite{chase2023replicability}. While our algorithms critically use shared randomness, a priori it is not clear this is necessary in our relaxed settings. Our final result is to show that all three models indeed require at least one shared random bit.

For pointwise and semi-replicability, the necessity of shared randomness is fairly straightforward, and follows already from the reductions we give to prove sample lower bounds, e.g., by encoding a (strictly replicable) bias estimation instance onto the label of a single element. For approximate replicability, we show directly using the Poincare-Miranda theorem that even an algorithm that is consistent on a $\frac{1}{2}$-fraction of the domain requires shared randomness.

\begin{restatable}{theorem}{SharedRandomnessApxRepl}
    \label{thm:shared-randomness-apx-repl-lb}
    There is no deterministic $(0.01, 0.5)$-approximately replicable $(0.05, 0.01)$-learner for any class with VC dimension $d \geq 2$.
\end{restatable}

Finally, we remark that our previous algorithms for pointwise replicable and approximately replicable learning use independent shared random strings for every element $x \in \domain$ in the underlying data domain.
For infinite domains, this leads to a potentially infinite amount of shared randomness (although the randomness is only required for the points $x$ which we wish to label, so any algorithm that only queries the output hypothesis on finitely many points will still use finite randomness).
The following theorem gives an algorithm that achieves both point-wise and approximate replicability using finite randomness over arbitrary (even uncountable) domains, at the cost of weaker dependence on $\alpha, \rho, \gamma$ parameters.
We remark that we maintain linear dependence on the VC dimension $d$.

\begin{theorem}[Informal \Cref{thm:apx-repl-finite-r}]
    \label{thm:apx-repl-finite-r-informal}
    There exists an (agnostic) $\rho$-pointwise replicable $(\alpha, \beta)$-learner with sample complexity $\tO{d \rho^{-4} \alpha^{-7}}$ and an (agnostic) $(\rho, \gamma)$-approximately replicable $(\alpha, \beta)$-learner with sample complexity $\tO{d \rho^{-2} \gamma^{-2} \alpha^{-5} + d \rho^{-4} \alpha^{-7}}$.
    Furthermore, both algorithms run in time linear in sample complexity with oracle calls to a (possibly non-replicable) $(\alpha, \beta)$-learner.
\end{theorem}


\subsection{Technical Overview}

We give an overview of our techniques.

\subsubsection{Replicable Prediction}

For simplicity of exposition, we will assume the underlying domain $\domain$ is finite.\footnote{Note technically any class $(\mathcal X, H)$ over a finite domain is fully replicably learnable in $\log^2(|H|)$ samples \cite{bun2023stability}, but many VC classes as simple as thresholds require a number of samples scaling with the size of the domain \cite{bun2020equivalence,bun2023stability}. We will obtain sample complexity linear in VC Dimension, establishing a significant separation even in the finite setting.}




\paragraph{A First Attempt: Randomized Averaging}

We begin by describing a simple (incorrect) attempt to turn any learning algorithm $\innerAlg$ into a replicable predictor we then modify to give a legimate such algorithm.
For any $x$, let $p(x) := \E_{S, \innerAlg}[\innerAlg(S)(x)]$ denote the expected label of $x$ under a random hypothesis learned by $\innerAlg$ (over both the randomness of the samples and the algorithm).
Our algorithm runs $\innerAlg$ on $T \gg 1/\rho^{2}$ fresh independent sample sets to obtain hypotheses $h_{1}, \dotsc, h_{T}$.
Now, given any $x$, let $\hat{p}(x) := \frac{1}{T} \sum_{t = 1}^{T} h_{t}(x)$ denote an empirical estimate of $p(x)$.
Our algorithm then draws an independent random threshold $r(x) \sim \UnifD[-1, 1]$ for each point and outputs the hypothesis 
\begin{equation}
    \label{eq:average-estimate-intro}
    g: x \mapsto \begin{cases}
        1 & \hat{p}(x) \geq r(x) \\
        -1 & \hat{p}(x) \leq r(x)
    \end{cases}
\end{equation}
Since $r(x)$ is drawn uniformly from $[-1, 1]$, we immediately get that the expected error of our algorithm $g$ on any $x$ over the choice of random threshold $r(x)$ is the average error of $\set{h_t}$ on $x$ (\Cref{lem:exp-error-ub-pred}).


First, let us confirm this procedure is indeed pointwise replicable. 
Fix an $x \in \domain$.
Since we run $\innerAlg$ on $T \gg 1/\rho^2$ fresh datasets, we can show that $\var(\hat{p}(x)) \leq \rho^2$.
Since $g$ fails to be replicable if and only if $r(x)$ falls between two independent estimates $\hat{p}(x)$, the probability that this occurs (over the randomness of $\hat{p}(x), r(x)$) is at most $\rho$ (\Cref{lem:p-estimate}).

Unfortunately, this approach runs into an issue with accuracy. 
While above we have argued the error of our hypothesis $g$ matches the average error of the $T$ hypotheses output by $\innerAlg$ in \emph{expectation}, we need to output an accurate hypothesis with high probability (at least $1 - \beta$). 
To prove this, it is natural to try to use the fact that the error of $g$ on each domain element $x$ is independent since each $r(x)$ is sampled independently. 
Thus, if the marginal distribution $\distribution$ is uniformly spread over many elements (each with small mass), we can apply Hoeffding's inequality to show the error of $g$ is concentrated around its expectation with high probability, and in particular at most $\opt + \alpha$, as desired.
We note that to obtain the final bound, we condition on the event that $\errD(h_{t}) \leq \opt + \alpha$ for all $t$.

Unfortunately, we have no control over $\distribution$, and it may well be the case that the distribution is not well spread out (i.e., it may have many heavy hitters, elements with large mass). In this case, Hoeffding will not imply any non-trivial concentration, and for good reason: the error of $g$ on a small set of heavy hitters may indeed fail to be concentrated, and we need another approach in this case.\footnote{A previous version of this paper incorrectly applied Markov's inequality to bound the error. An anonymous reviewer gave the following counterexample: $\domain = \set{1, 2}$, $\distribution = \UnifD(\domain) \times \set{1}$, and $H = \set{\delta_{1}, \delta_{2}}$ where $\delta_x(z) = -1$ if $x = z$, $1$ otherwise. Our algorithm will learn $\hat{p}(x) \sim 0$ and with $\Omega(1)$ probability map both elements to $-1$, incurring error $1/2$ beyond $\opt = 1/2$.}

\paragraph{Handling Heavy Hitters}

Thankfully, there is a (morally) simple fix to this problem: since any element $x$ with large mass can be sampled many times, we can just directly learn an $\alpha$-optimal label for $x$!
Formally, set a threshold $\nu$, and say $x$ is $\nu$-heavy if $\distribution(x) \geq \nu$ under the marginal distribution and $\nu$-light otherwise.
In order to learn heavy hitters and their labels, we introduce two useful tools: (1) a pointwise replicable algorithm for identifying heavy hitters, and (2) a pointwise replicable algorithm for learning the labels of heavy hitters. Once we have identified and learned labels of the set of heavy hitters, the rest of the distribution $\distribution$ is ``well-spread'', so Hoeffding will indeed give us the desired concentration on the remainder of the points.

Let's make this strategy slightly more formal. Our pointwise replicable algorithm for heavy-hitters (\Cref{prop:finding-heavy}) takes in $1/(\rho^2 \nu)$ samples, and outputs a set $S$ such that any $10 \nu$-heavy element is in $S$ and no $0.1\nu$-light element is in $S$.
Furthermore, over two runs, any $x \in \domain$ is either included in $S$ or excluded from $S$ in both runs with probability $1 - \rho$.
We remark that in contrast fully replicable algorithms for identifying heavy hitters require $1/(\rho^2 \nu^3)$ samples, while we crucially exploit the weaker notion of pointwise replicability to obtain improved sample complexity.
In particular, we use $1/\nu$ samples to identify a set of candidate heavy hitters, and then for every candidate element $x$, we can estimate its mass $\distribution(x)$ up to multiplicative error $\hat{\distribution}(x) \in (1 \pm \rho) \distribution(x)$ using $1/(\rho^2 \nu)$ samples.

Our next task is to replicably decide a labeling for these points. Our pointwise replicable labeling algorithm (\Cref{prop:learn-hh-label}) takes a set $S$ where each element $x \in S$ is guaranteed to be $\nu$-heavy and outputs $\ell(x)$ for all $x \in S$ such that $\ell(x) = \sign(y(x))$ whenever the label distribution is sufficiently biased (formally $|y(x)| \geq \alpha$ where $y(x) = \E[y|x]$ is the expected label of $x$ under the input distribution). This can be done via a simple replicable bias estimation procedure \cite{ImpLPS22} that replicably tests the bias of a single Rademacher distribution up to error $\alpha$ with $1/(\rho^2 \alpha^2)$ samples.
Since a sample of size $m$ contains $m \nu$ samples of a $\nu$-heavy element on average, it suffices to take $m \gg 1/(\rho^2 \alpha^2 \nu)$ samples.

\paragraph{The Final Algorithm}

Combining the tools above, we define our final algorithm to first (pointwise replicably) identify a set of $\nu$-heavy elements $S$ and output
\begin{equation*}
    h: x \mapsto \begin{cases}
        \ell(x) & x \in S \\
        g(x) & x \not\in S
    \end{cases} 
\end{equation*}
where $\ell$ is the label obtained from \Cref{prop:learn-hh-label} and we recall $g$ is the original labeling obtained by thresholding the averaging $T \gg \rho^{-2}$ hypotheses in \eqref{eq:average-estimate-intro}.

We argue that $h$ is still pointwise replicable.
For any $x$, its membership in $S$ is consistent since we identify heavy hitters in a pointwise replicable manner.
Given that membership in $S$ is consistent across two runs, it is enough to prove that $\ell$ and $g$ are themselves pointwise replicable. For $\ell$, this is guaranteed since $S$ contains only $\Omega(\nu)$-heavy elements for which the above bias estimation procedure is replicable, while $g$ is pointwise replicable by our initial discussion.

It is left to argue that $h$ produces $\alpha$-accurate hypotheses with high probability (for the appropriate setting of the threshold $\nu$). 
Let $S^*$ denote the set of $\nu$-heavy elements and $\distribution(S^*)$ denote their collective mass.
We consider two cases.
If $\distribution(S^*) \geq 1 - \alpha$, the error of $h$ on $\domain \setminus S^*$ is at most $\alpha$ since $\distribution(\domain \setminus S^*) \leq \alpha$.
Furthermore, since $S^* \subset S$, our labeling algorithm ensures that $\ell(x)$ learns the optimal label whenever $|y(x)| \geq \alpha$, $\ell(x)$ has error at most $\alpha$ in excess of any function on $S^*$. 
We thus conclude that $h$ is an $O(\alpha)$-accurate hypothesis.

Now, consider the second case, i.e. $\distribution(\domain \setminus S^*) \geq \alpha$. 
Our goal in this case is to ensure that the error of $h$ on $\domain \setminus S^*$ concentrates around its mean.
As a sum over independent random variables bounded by $\distribution(x)$ for every $x \in \domain \setminus S^*$, a standard application of Hoeffding's inequality allows us to bound the probability that the error exceeds its mean by $\alpha$ as
\begin{equation*}
    \exp\left( - \bigOm{\frac{\alpha^2}{\sum_{x \not\in S} \distribution(x)^2}} \right) < \exp\left( - \bigOm{\frac{\alpha^2}{\nu \sum_{x \not\in S} \distribution(x)}} \right) < \exp\left( - \bigOm{\frac{\alpha^2}{\nu}} \right)
\end{equation*}
where we have used any $x \not\in S$ is $\nu$-light as $S^* \subset S$.
Thus, we set $\nu \ll \frac{\log(1/\beta)}{\alpha^2}$ to ensure that the error is well concentrated around its mean (and therefore at most $\opt + \alpha$).

\paragraph{Lower Bound}

To obtain a lower bound, we reduce from replicable bias estimation: given sample access to a Rademacher distribution $\distribution^*$ (supported on $\set{\pm 1}$), determine if the mean is greater than $\alpha$ or less than $- \alpha$. \cite{ImpLPS22} show any $\rho$-replicable algorithm for $\alpha$-bias estimation (even with constant error probability) requires $\bigOm{\rho^{-2} \alpha^{-2}}$ samples. We will build a reduction that `hides' one such instance into a set of $d$ shattered points, improving the lower bound to the tight $\bigOm{d \rho^{-2} \alpha^{-2}}$.

Fix any class $\hypotheses$ with VC Dimension $d$ and let $\set{x_{1}, \dotsc, x_{d}}$ be a set of shattered elements.
Without loss of generality, let $d$ be odd.
We reduce bias estimation to PAC learning as follows.
Let $\distribution$ be a uniform distribution over $\set{x_{1}, \dotsc, x_{d}}$ that labels $x_{i}$ according to $\Rad(\alpha)$ for a randomly selected subset of $\frac{d - 1}{2}$ elements, $\Rad(-\alpha)$ for the remaining $\frac{d - 1}{2}$, and $\distribution^*$ for the remaining element $x^*$.
Given sample access to $\distribution^*$, we can easily create a dataset of $m$ samples from $\distribution$, by drawing $m$ elements from $\set{x_{1}, \dotsc, x_{d}}$ independently and uniformly, and labeling each one according to the element's corresponding Rademacher distribution.
Note that we only need (in expectation) $\frac{m}{d}$ samples from $\distribution^*$ to simulate $m$ samples from $\distribution$.
Let $\innerAlg$ be a $\rho$-pointwise replicable $(\alpha, \beta)$-learner for $\hypotheses$.
We construct $S$ as described above, and then given $h \gets \innerAlg(S)$, we return $h(x^*)$.

Since for any fixed $x$, and in particular $x^*$, our algorithm is consistent with probability $\rho$, we immediately obtain a $\rho$-replicable algorithm.
Thus, it suffices to argue correctness.
Assume without loss of generality $\distribution^* \sim \Rad(\alpha)$.
Suppose we obtain an $0.01\alpha$-accurate hypothesis $h$ on this distribution.
There are $\frac{d + 1}{2}$ (randomly chosen) elements in the domain whose labels are distributed according to $\Rad(\alpha)$.
For every element labeled $-1$, $h$ incurs an excess error of $\frac{\alpha}{d}$ over the optimal hypothesis.
In particular, if $h$ is $0.01\alpha$-accurate, $h$ can label at most $\frac{0.01\alpha}{\alpha/d} = 0.01d$ elements $-1$.
Since the elements appear identical to $\innerAlg$ and the position of $x^*$ is randomly selected, we argue this implies the probability that $h(x^*) = -1$ is at most $\frac{0.01d}{(d+1)/2} \leq 0.02$.
Thus, we obtain an algorithm that solves bias estimation with error probability at most $0.02$ and takes (with high probability) at most $O(m/d)$ samples from $\distribution^*$, violating \cite{ImpLPS22} if $m = o(d \rho^{-2} \alpha^{-2})$.


Finally, note that this reduction also immediately implies that pointwise replicability requires shared randomness, since replicable bias estimation is also well known to require shared randomness \cite{ImpLPS22,chase2023replicability,dixon2023list}.
In particular, we can take a single element $x \in \domain$ and use samples from a bias estimation instance to label $x$.
 
\subsubsection{Approximate Replicability}

At a high level, our approximately replicable learner follows the same framework as replicable prediction: for elements with high probability mass, we learn their labels directly; for elements with low probability mass, we average several independently learned hypotheses and argue that the error is well concentrated. As before, given a threshold $\nu$, we require procedures to (1) identify $\nu$-heavy elements and (2) learn their labels. Unlike the pointwise case where it sufficed to consider replicability of each heavy hitter individually, to ensure approximate replicability we now need to ensure  a large fraction of the identified heavy hitters and their labels are consistent simultaneously.

Our goal is to obtain a $(\rho, \gamma_0)$-approximately replicable learner.
Let $\gamma \gets \gamma_{0}/C$ for some sufficiently large $C$.
Our approximately replicable heavy hitter identification algorithm (\Cref{prop:apx-repl-heavy}) returns two sets: $S_{\gamma}$ containing all $2 \gamma$-heavy elements and no $\gamma/2$-light elements, and $S$ containing all $\min(10 \nu, 2\gamma)$-heavy elements and no $\min(0.1 \nu, \gamma/2)$-light elements. Without loss of generality, we may assume $S_{\gamma} \subseteq S$. Recall that for replicable prediction we only returned a single set $S$. Our two-level system here allows us to ensure $S_{\gamma}$, which contains all $\Omega(\gamma)$-heavy elements, is \textit{fully} replicable. This is necessary since \textit{any} disagreement on an element with mass $\Omega(\gamma)$ immediately rules out $\gamma$-closeness of the output hypothesis. On top of this, we require at most a $\gamma$-fraction of elements to placed in the larger set $S$ inconsistently (i.e., $\distribution(S^{(1)} \Delta S^{(2)}) < \gamma$ where $S^{(1)}, S^{(2)}$ are outputs over two runs).\footnote{We use $\Delta$ to denote the symmetric difference.}

As before, our approximately replicable label learning algorithm outputs $\ell(x)$ for all $x \in S$ such that $\ell(x) = \sign(y(x))$ for $|y(x)| \geq \alpha$. 
The replicability guarantee is as follows: we learn the labels of $S_{\gamma}$ fully replicably (again, since any disagreement on an element with mass $\Omega(\gamma)$ immediately rules out the output hypotheses being $\gamma$-close) and ensure that elements an all but $\gamma$-fraction of elements that are identified as $\nu$-heavy in both iterations (i.e. $x \in S^{(1)} \cap S^{(2)}$) are labeled consistently (i.e. $\ell^{(1)}(x) = \ell^{(2)}(x)$).

Finally, we label $h(x) = \ell(x)$ if $x \in S$ and $h(x) = g(x)$ otherwise (see \eqref{eq:average-estimate-intro}) as before, where $g$ is now obtained by averaging $T := O(\gamma^{-2})$ hypotheses (in contrast to $O(\rho^{-2})$) for some large $C$.
We remark the the accuracy arguments follow essentially identically to the replicable prediction algorithm, so we focus only on approximate replicability.
From our algorithm above, we see that an element is inconsistently labeled (i.e. $\tilde{h}^{(1)}(x) \neq \tilde{h}^{(2)}(x)$) if and only if one of the following occur:
\begin{enumerate}
    \item $x$ is classified as $\nu$-heavy inconsistently: $x \in S^{(1)} \Delta S^{(2)}$.
    \item $x$ is $\nu$-heavy but labeled inconsistently: $x \in S^{(1)} \cap S^{(2)}$ and $\ell^{(1)}(x) \neq \ell^{(2)}(x)$.
    \item $x$ is not $\nu$-heavy but labeled inconsistently: $x \not\in S^{(1)} \cup S^{(2)}$ and $g^{(1)}(x) \neq g^{(2)}(x)$.
\end{enumerate}
Our heavy hitter identification (\Cref{prop:apx-repl-heavy}) and labeling (\Cref{prop:apx-repl-hh-label}) algorithms ensure that the first two events occur only on a $O(\gamma)$-fraction of elements.
Consider the last event.
Following similar arguments to replicable prediction, over two runs of the algorithm, $g$ is pointwise replicable with probability $\ll \gamma$.
In particular, since $S$ contains $S_{\gamma}$ which contains all $2 \gamma$-heavy elements, all $x \not\in S$ are $2 \gamma$-light. Since we sample each $r(x)$ independently, the mass of elements labeled inconsistently under $g$ is a sum of independent random variables in $[0, 2 \gamma]$ whose sum has expectation $\ll \gamma$, and a standard Chernoff bound ensures that the mass exceeds $\gamma_0 \gg C \gamma$ with probability at most $\exp(-\Omega(C \gamma/\gamma)) \ll 1$.
In particular, setting $C = O(\log(1/\rho))$ suffices to ensure that the output hypotheses are $O(\gamma_0)$-close with probability at least $1 - \rho$.

\paragraph{Sample Lower Bound}

Like in the pointwise replicable setting, we'd like to prove a sample lower bound for approximate replicability via reduction from bias estimation, but doing so is no longer as straightforward as `planting' a single randomized domain point. In particular,
while pointwise replicability guaranteed our answer is replicable on \emph{every} $x \in \domain$ (and therefore also on whichever element we embed the bias estimation instance), in approximate replicability the algorithm only needs to be consistent on a $\gamma$-fraction of the domain, so it is entirely possible the algorithm is never consistent on the embedded instance.
In particular, under the pointwise guarantee, we only needed indistinguishability to guarantee correctness (and thus only when the label of $x^*$ is sufficiently biased), while here we need to ensure $x^*$ remains indistinguishable for \emph{every} input distribution to the bias estimation problem.

To do this, we will critically rely on the fact that replicable bias estimation is hard even in the following \textit{average-case} setting: even when the adversary selects $\Rad(p)$ with $p \in [-\alpha, \alpha]$ chosen uniformly at random (and this distribution is known to the learner), replicable bias estimation \textit{still} requires $\bigOm{{\rho^{-2} \alpha^{-2}}}$ samples.
Thus, assuming that our input distribution to the bias estimation is drawn from this meta-distribution,\footnote{Here we mean the adversary's distribution over distributions.} we can draw $d - 1$ other dummy input distributions from the same meta-distribution, thereby hiding the true distribution from the algorithm.

More formally, consider the following reduction.
Let $\distribution$ denote an (unknown) distribution over $\set{\pm 1}$ to which we have sample access.
Let $\set{x_{1}, \dotsc, x_{d}}$ denote a set of $d$ shattered elements, and consider a uniform distribution over the $d$ elements.
We define the label distribution as follows.
Randomly select an index $r \in [d]$ and generate labels according to $\distribution$.
For every other index, generate labels according to a (shared) distribution sampled from the hard meta-distribution.
Note that this process defines a meta-distribution over input distributions to our learning algorithm, and this meta-distribution is a $d$-wise product of the hard meta-distribution for bias estimation.

Suppose we have an $(\rho, \gamma)$-approximately replicable $(0.01 \alpha, \beta)$-learner taking $m$ samples.
If we have a randomized algorithm that is correct and replicable on every distribution, a standard minimax-style argument shows there is also a deterministic algorithm that is correct and replicable on a random distribution (drawn from the product meta-distribution).
In particular, even if every element's label (including the one embedding the bias estimation problem) is drawn from the hard meta-distribution, we have an algorithm that is correct and replicable with respect to ``most" of the label distributions.

Now, consider the random variable $\err_{i}$ denoting the error of the hypothesis on $x_{i}$.
Over the product meta-distribution, $\err_{i}$ are distributed identically, and therefore at most $\alpha$ on average.
In particular, $\err_{i} = O(\alpha)$ with high constant probability.
Furthermore, since (over the meta-distribution) all elements of the domain are indistinguishable, the approximately replicable learner is consistent on the relevant label with probability $1 - O(\rho + \gamma)$, by union bounding over the event that the algorithm produces a $\gamma$-consistent hypothesis, and that the relevant element lies in the consistent region of the hypotheses.
Thus, we obtain a $O(\alpha)$-accuracy algorithm for bias estimation with $O(\rho + \gamma)$-replicability.
Since our algorithm only takes $O(m/d)$ samples from the relevant element with high probability, we conclude $m = \bigOm{d(\rho + \gamma)^{-2} \alpha^{-2}}$.

\paragraph{Shared Randomness} Because the above reduction itself uses shared randomness, we cannot directly argue as before this implies necessity of shared randomness for approximate replicability. We will argue this directly using the Poincare-Miranda Theorem.
It suffices to consider an algorithm that accomplishes the following task which we can embed into any class with VC dimension $\geq 2$: 

\begin{center}
    Given sample access to $\Rad(p_{1}) \times \Rad(p_{2})$, output $v \in \set{\pm 1}^{2}$ such that 
    
    (1) $\max_{i} \{|p_{i}|\ind[v_{i} \neq \sign(p_{i})]\}  \leq \alpha$ and (2) $\Pr(\norm{v^{(1)} - v^{(2)}}_{0} > 1) < 0.01$.
\end{center}

Above, (1) corresponds to correctly deciding the bias of each distribution, while (2) corresponds to an $(0.01, 0.5)$-approximately replicable algorithm. The idea is now as follows. Suppose there is an approximately replicable deterministic algorithm $\innerAlg$.
We may define for every $y \in \set{\pm 1}^{2}$
the sets 
\begin{equation*}
    B_{y} := \set{p = (p_1, p_2) \in [0, 1]^{2} \given \Pr_{S \sim \Rad(p)^{m}} \left( \innerAlg(S) = y \right) \geq 0.1} \text{.}
\end{equation*}
That is, $B_{y}$ contain the set of distributions where $\innerAlg$ outputs $y$ with reasonable probability.
Since there is at least one output with probability $\geq \frac{1}{4}$, the collection of $B_{y}$ cover the cube $[-1,1]^{2}$.
Observe that if $B_{x} \cap B_{y}$ for $\norm{x - y}_{0} = 2$ then $\innerAlg$ is not approximately replicable for the distribution parameterized by $p \in B_{x} \cap B_{y}$, as with probability at least $0.02$ the algorithm produces $x, y$ on two distinct runs.
Thus, it suffices to show such an intersection exists.

Suppose we wish to cover the unit square with $4$ sets, one corresponding to each corner.
Our goal is to show that at least one pair of opposite sets must intersect.
\Cref{fig:poincare-miranda} gives an illustration.

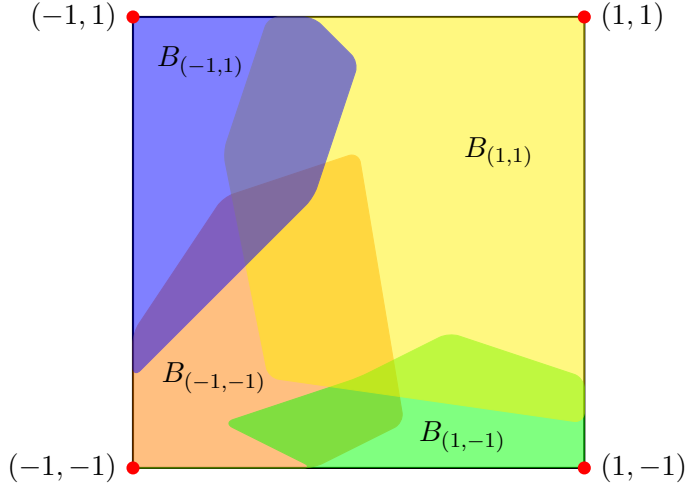
\begin{figure}
    \centering
    \input{tikz/poincare_miranda}
    \caption{An example of sets $B_{y}$ covering $[-1, 1]^{2}$. The four regions must cover the whole square and correctness constraints require $B_{\leftT}$ (which consists of the blue $B_{(-1, 1)}$ and orange $B_{(-1, -1)}$ regions) to contain the full left border. Similar constraints hold on $B_{\rightT}, B_{\up}, B_{\down}$.
    Then, regardless of how the square is colored, at least one pair of opposing regions (either yellow and orange or blue and green) must intersect. In the given example, yellow $B_{(1, 1)}$ and orange $B_{(-1, -1)}$ intersect.}
    \label{fig:poincare-miranda}
\end{figure}

Formally, consider the sets
\begin{align*}
    B_{\leftT} &= B_{(-1, -1)} \cup B_{(-1, 1)} \\
    B_{\rightT} &= B_{(1, -1)} \cup B_{(1, 1)} \\
    B_{\up} &= B_{(-1, 1)} \cup B_{(1, 1)} \\
    B_{\down} &= B_{(-1, -1)} \cup B_{(1, -1)}
\end{align*}
and the map
\begin{equation*}
    F: p \mapsto \left( \dist(p, B_{\leftT}) - \dist(p, B_{\rightT}), \dist(p, B_{\up}) - \dist(p, B_{\down}) \right) \text{.}
\end{equation*}
By correctness, we know $p \in B_{\leftT}$ on the left border of $[-1, 1]^{2}$, $p \in B_{\rightT}$ on the right border, $p \in B_{\up}$ on the up border, and $p \in B_{\down}$ on the down border.
Furthermore, as $F$ is continuous, the Poincare-Miranda Theorem states that there exists $p^*$ with $F(p^*) = (0, 0)$.
Without loss of generality, we can assume $p^* \in B_{(1, 1)}$ and therefore $p^* \not\in B_{(-1, -1)}$ (else we are done). 
Since $p^* \in B_{\rightT} \cap B_{\up}$ and $F(p^*)=(0,0)$, we must have $p^* \in B_{\leftT} \cap B_{\down}$. Finally, since $p^* \notin B_{(-1,-1)}$ by assumption, this can only be the case if $p^* \in B_{(-1, 1)} \cap B_{(1, -1)}$, and in particular the intersection $B_{(-1, 1)} \cap B_{(1, -1)}$ is non-empty so we are done.

\subsubsection{Semi-Replicable Learning}

Our algorithm for semi-replicable learning follows the same basic two-step framework introduced for semi-private learning \cite{alon2019limits,hopkins2022realizable}: first we use $O(d/\alpha)$ shared unlabeled samples to compute a small set of roughly $(d/\alpha)^{O(d)}$ hypotheses $C \subset \hypotheses$ that almost certainly contains a good classifier, and, second, we use a replicable learner for finite classes from \cite{bun2023stability} to learn $C$. 
The latter requires roughly $\log^2|C|$ labeled samples, which gives the desired bound. Thus perhaps the more interesting result lies in showing this simple procedure is sample-optimal, both in terms of the shared and unshared sample complexities.


\paragraph{Shared Sample Complexity}

Our first lower bound, on the shared sample complexity of semi-replicable learning, is actually proved via a reduction to the so-called \textit{semi-private} learning model, along with a new lower bound for this setting built on \cite{alon2019limits}. Intuitively, a semi-private algorithm has access to public dataset $D_{\pub}$ and private dataset $D_{\priv}$ and must adhere to privacy constraints only on the private dataset (the exact notion of approximate DP used below is standard, but not important for the proof overview so we omit the exact definition, see \Cref{def:private}):
\begin{center}
    An algorithm $\innerAlg$ is $(\varepsilon, \delta)$-semi-private if given any public dataset $D_{\pub}$, $\innerAlg(D_{\pub}, \cdot)$ is $(\varepsilon, \delta)$-private.
\end{center}

Following existing tools \cite{GhaziKM21,bun2023stability}, it is easy to transform any semi-replicable algorithm into a `semi-private' one with roughly matching sample complexity. Thus, if we can show that no semi-private algorithm exists using $o(d/\alpha)$-public samples, it will imply a corresponding lower bound for semi-replicable algorithms on $o(d/\alpha)$-shared samples. 
Toward this end, we will rely heavily on a result of \cite{alon2019limits} who prove any semi-private algorithm for a class with infinite Littlestone dimension must use $\bigOm{\alpha^{-1}}$ public samples. 
Our main goal is to improve this bound to $\bigOm{d/\alpha}$ for a specific class of infinite Littlestone Dimension and VC Dimension $d$: the class of conjunctions of $d$ thresholds, i.e. $x \mapsto \bigwedge_{i = 1}^{d} h_{i}(x_{i})$.


As in approximate replicability, the main idea is to hide one ``relevant'' threshold among $d$ dummy thresholds, forcing the algorithm to use an additional multiplicative factor of $d$ samples. 
In particular, given one hard distribution over datasets $(D_{\pub}, D_{\priv})$, we construct an instance that consists of a uniform distribution over $d$ distinct thresholds: thus any accurate hypothesis must be accurate on a large fraction of the $d$ thresholds.
Since the thresholds are drawn from the same hard distribution, the error of our algorithm on an average threshold (including our single ``relevant'' one) is $O(\alpha)$.
However, any semi-private algorithm for thresholds requires $\bigOm{\frac{1}{\alpha}}$ public samples, so that our original algorithm requires $\bigOm{\frac{d}{\alpha}}$ public samples to be semi-private and accurate on $\bigOm{d}$ thresholds.

\paragraph{Labeled Sample Complexity}

To obtain a lower bound on the labeled sample complexity, we take inspiration for the lower bound for replicably learning VC classes in \cite{bun2023stability} and reduce from the $d$-dimensional sign-one-way marginals problem.
Let $\distribution_{p}$ denote a Rademacher product distribution (supported on $\set{\pm 1}^{d}$ with mean $p$).
A solution $v \in \set{\pm 1}^{d}$ is $\alpha$-accurate if $\frac{1}{d} \sum_{i} \ind[v_{i} \neq \sign(p_{i})] |p_{i}| \leq \alpha$.
A standard reduction shows that $\hat{v} = (h(x_{1}), \dotsc, h(x_{d}))$ is $\alpha$-accurate if and only if $\errD(h) \leq \alpha$ \cite{bun2023stability, hopkins2025generative}.
Since our learning algorithm is replicable, we obtain a replicable algorithm for $d$-dimensional sign-one-way marginals, which \cite{hopkins2025generative} recently showed requires $\bigOm{d^{2} \rho^{-2} \alpha^{-2}}$ samples, giving the desired bound.

\subsection{Conclusion and Open Questions}

Replicability is a highly desirable property of learning algorithms, but comes at a major cost due to its extreme requirements (making many classes at least quadratically harder to learn, and others like 1-D thresholds simply impossible). In this work, we have shown that several natural relaxations of replicability are feasible for \textit{any} learnable class, often at very little cost for constant replicability parameters.

Our work leaves open several interesting questions, most obviously settling the sample complexity of pointwise and approximate replicability, but also whether either notion can be achieved \textit{properly} for any class. Below, we discuss a few challenges and potential directions toward resolving the former of these questions.

For replicable prediction, we obtain a $\rho$-pointwise replicable $(\alpha, \beta)$-learner on $\bigtO{d \rho^{-2} \alpha^{-2} + \rho^{-2} \alpha^{-4}}$ samples.
A careful inspection of our algorithm reveals that the additive term is required only in the step to learn the labels of $\nu \sim \alpha^{-2}$-heavy elements up to accuracy $\alpha$. 
It would be interesting to either remove this term or prove it is in fact necessary.
A similar question holds for approximate replicability.


Furthermore, the sample complexity of $(0.01, \gamma)$-approximately replicable learning remains unsettled (i.e.\ when $\rho \gg \gamma)$) .
In this regime, even a simpler variant remains unsolved: approximately replicable bias estimation.

\begin{center}
    Given samples to a binary product distribution with mean $p \in [-1, 1]^{n}$, return $v \in \set{\pm 1}^{n}$ satisfying,

    (1) $\max_{i} \ind[v_{i} \neq \sign(p_{i})] |p_{i}| \leq \alpha$ and (2) $\Pr(\norm{v^{(1)} - v^{(2)}}_{0} > \gamma) < 0.01$.
\end{center}

In other words, the algorithm must output hypotheses that disagree on at most a $\gamma$-fraction of coordinates.
For constant $\alpha$, \cite{hopkins2024replicability} show that $\bigOm{d\gamma^{-1}}$ samples are necessary, while the best algorithms for this problem need $\bigO{\min\left(d^2, d \gamma^{-2} \right)}$ samples, leaving a substantial gap in sample complexity.

\subsection{Further Related Work}

\paragraph{Replicable Learning}
Replicable learning was introduced by \cite{GhaziKM21, ImpLPS22} and has played a central role in recent years in the study of various notions of stability \cite{bun2023stability,kalavasis2023statistical,moran2023bayesian}, notably including differential privacy \cite{bun2023stability}. 
Replicable algorithms are now known for a variety of classical problems, including mean estimation \cite{ImpLPS22, hopkins2024replicability,vander2024replicability}, distribution testing \cite{liu2024replicable, diakonikolas2025replicable, aamand2025structure}, clustering \cite{esfandiari2023replicable}, PAC Learning \cite{bun2023stability,kalavasis2024replicable, larsen2025improved}, and reinforcement learning \cite{eaton2023replicable, karbasi2023replicability, hopkins2025generative, eaton2025replicable}. 
\cite{ImpLPS22, kalavasis2024computational} study replicable quantile estimation over finite domains, while our threshold algorithm requires approximately replicable quantile estimation over arbitrary domains.
Conversely, various works have also shown strong lower bounds implying replicability is either more costly than standard learning \cite{hopkins2024replicability,hopkins2025generative} or impossible (e.g., for learning infinite Littlestone classes) \cite{bun2023stability}).

Finally, related to our lower bounds on shared randomness, several recent works have studied the number of random bits needed to achieve replicability \cite{chase2023replicability, dixon2023list,vander2024replicability,chase2024local, hopkins2025role}. \cite{chase2023replicability} in particular also use Poincare-Miranda for this purpose to study an equivalent formulation called list-replicability \cite{dixon2023list}. List replicability is strictly stronger than approximate replicability (and, in particular, is statistically equivalent to \cite{ImpLPS22}'s original notion), and to the best of our knowledge the necessity of shared randomness for approximate replicability cannot be derived from their result.

\paragraph{Approximate Replicable Learning}

Several relaxations of replicability have been studied or proposed in the literature \cite{karbasi2023replicability, chase2023replicability, kalavasis2023statistical, hopkins2024replicability}. 
Our notion of approximate replicability, for instance, was proposed as an open direction in \cite{chase2023replicability}.
Concretely, approximately replicable reinforcement learning was studied in \cite{karbasi2023replicability} and bias estimation in \cite{hopkins2024replicability}.

Our work also relates more broadly to the study of notions of approximate stability in learning, starting with the seminal work of \cite{bousquet2002stability} who require output hypotheses to have similar error under small perturbations to the training set. These notions are typically incomparable to the ones we study --- they are noticeably weaker in the sense they only require similar loss rather than similar output hypotheses, but do not require shared randomness as we show our models must.






\paragraph{Private Prediction} Private Prediction was introduced by Dwork and Feldman \cite{dwork2018privacy} where they give private predictors for PAC learning and convex regression.
Subsequent works \cite{bassily2018model, dagan2020pac} improve the sample complexity of agnostic PAC learning and \cite{van2020trade} gave an experimental study of private prediction algorithms. Using known connections between differential privacy and replicability, it is possible to translate any private predictor into a replicable one at roughly quadratic overhead. Thus such bounds would scale quadratically with the VC dimension, while we achieve linear dependence directly.

The literature has also covered interesting extensions of the private prediction model such as `everlasting' prediction, which allows an infinite stream of queries preserving privacy \cite{naor2023private, stemmer2024private}. It would be very interesting to show an analogous result in the replicable setting that decays slower than the trivial union bound argument against multiple such prediction queries.

\paragraph{Semi-Private Learning}

Semi-private learning, introduced by \cite{beimel2013private}, studies how access to public samples may be helpful for private learning. Semi-privacy has been studied in various settings, including estimation and distribution learning \cite{bie2022private, ben2023private}, query release \cite{liu2021leveraging, pmlr-v119-bassily20a}, feature learning \cite{krichene2024private}, stochastic optimization \cite{ullah2024public}, and PAC learning \cite{bassily2018model, block2024oracle}. \cite{alon2019limits,hopkins2022realizable} settle the public sample complexity of semi-private learning for constant VC classes.

%% file: tikz/poincare_miranda.tex
\usetikzlibrary{decorations.pathmorphing}

\begin{tikzpicture}[scale=6]

  \draw[line width=0.8pt] (0,0) rectangle (1,1);

  \draw[fill, rounded corners, opacity=0.5, orange] (0, 0) -- (0, 0.3) -- (0.2, 0.6) -- (0.5, 0.7) -- (0.6, 0.1) -- (0.4, 0) -- (0, 0);

  \draw[fill, rounded corners, opacity=0.5, green] (1, 0) -- (1, 0.2) -- (0.7, 0.3) -- (0.5, 0.2) -- (0.2, 0.1) -- (0.4, 0) -- (0, 0);

  \draw[fill, rounded corners, opacity=0.5, yellow] (1, 1) -- (1, 0.1) -- (0.3, 0.2) -- (0.2, 0.7) -- (0.3, 1) -- (1, 1);

  \draw[fill, rounded corners, opacity=0.5, blue] (0, 1) -- (0, 0.2) -- (0.4, 0.6) -- (0.5, 0.9) -- (0.4, 1) -- (0, 1);

  \node[circle, fill, draw, red, thick, scale=0.4, label={left:$(-1, -1)$}] (gc4) at (0, 0) {};

  \node[circle, fill, draw, red, thick, scale=0.4, label={right:$(1, 1)$}] (gc4) at (1, 1) {};

  \node[circle, fill, draw, red, thick, scale=0.4, label={left:$(-1, 1)$}] (gc4) at (0, 1) {};

  \node[circle, fill, draw, red, thick, scale=0.4, label={right:$(1, -1)$}] (gc4) at (1, 0) {};

  \node[scale=0.5, label={right:$B_{(-1, -1)}$}] (gc4) at (0.03, 0.2) {};
  \node[scale=0.5, label={right:$B_{(-1, 1)}$}] (gc4) at (0.02, 0.9) {};
  \node[scale=0.5, label={right:$B_{(1, -1)}$}] (gc4) at (0.6, 0.07) {};
  \node[scale=0.5, label={right:$B_{(1, 1)}$}] (gc4) at (0.7, 0.7) {};

\end{tikzpicture}

%% file: prelims.tex
\section{Preliminaries}
\label{sec:prelims}

Let $[n] = \set{1, \dotsc, n}$.
We use $\wedge$ to denote logical AND and $\vee$ to define logical OR.
We let $A \Delta B$ denote the symmetric difference of sets $A, B$.
For any distribution $\distribution$, let $\distribution^{m}$ denote the product of $\distribution$ taken $m$ times (i.e. taking $m$ \iid samples from $\distribution$).
Let $\calM$ denote a meta-distribution, i.e. a distribution over distributions.
Let $\Rad(p)$ denote the Rademacher distribution (i.e.\ a distribution supported on $\set{\pm 1}$) with mean $p$ i.e. $\Pr(X = 1) = \frac{1 + p}{2}$.
If $p$ is a vector, $\Rad(p)$ denotes a product of Rademacher distributions where each coordinate has mean $p_{i}$.

\paragraph{Statistical Learning}

Let $f, g$ be functions $\domain \rightarrow \set{\pm 1}$.
We typically use $x$ to denote samples from the domain $\domain$ and $y$ to denote labels i.e. $\set{\pm 1}$.
For any distribution $\distribution$ over $\mathcal{X} \times \{\pm 1\}$, we let $\distribution(x, y)$ denote the probability of the sample $(x, y)$.
For $x \in \domain$, we let $\distribution(x)$ denote the mass of $x$ under the marginal distribution of $\distribution$.
For $x \in \domain$, we let $y(x) = \E_{(x', y') \sim \distribution}[y'|x' = x]$ be the expected label of $x$ under $\distribution$.

The distribution error of a function $f$ is $\errD(f) := \Pr_{(x, y) \sim \distribution}(f(x) \neq y)$.
The classification distance of $f, g$ is $\dist_{\distribution}(f, g) := \Pr_{x \sim \distribution}(f(x) \neq g(x))$.
For any sample $S \in (\domain \times \set{\pm 1})^{m}$, the empirical error of $f$ is $\errS(f) := \frac{1}{m} \sum_{i = 1}^{m} \ind[f(x_i) \neq y_i]$.
The empirical classification distance of $f, g$ is $\dist_{S}(f, g) := \frac{1}{m} \sum_{i = 1}^{m} \ind[f(x_{i}) \neq g(x_{i})]$.

\begin{definition}[VC Dimension]
    \label{def:vc-dimension}
    A hypothesis class $\hypotheses \subset \domain \rightarrow \set{\pm 1}$ has VC-dimension $d$ if $d$ is the largest integer such that there exists a set of $d$ samples $\set{x_{1}, \dotsc, x_{d}}$ for which any labeling of the $d$ samples can be realized by a hypothesis $h \in \hypotheses$. 
    That is, $|{(h(x_{1}), \dotsc , h(x_{d})) \given h \in \hypotheses}|= 2^{d}$.
    We say such a set is shattered.
\end{definition}

\begin{theorem}[Uniform Convergence, e.g. \cite{blumer1989learnability}]
    \label{thm:uniform-convergence}
    Let $\hypotheses$ be a binary class of functions with domain $\domain$ with VC Dimension $d$.
    Then, for any distribution $\distribution$ over $\domain$ and all $m > 0$,
    \begin{equation*}
        \Pr_{x_1, \dotsc, x_{m} \sim \distribution} \left( \sup_{h \in \hypotheses} \left| \frac{1}{m} \sum_{i = 1}^{m} \ind[h(x_i) = 1] - \Pr_{z \sim \distribution}\left(h(z) = 1 \right) \right| \geq \gamma \right) \leq 4 (2m)^{d} e^{-\gamma^2 m/8} \text{.}
    \end{equation*}
\end{theorem}

\begin{definition}[PAC Learning \cite{vapnik1971,valiant1984theory,haussler1992decision}]
    \label{def:pac}
    An algorithm $\innerAlg$ is an agnostic $(\alpha, \beta)$-learner for $\hypotheses$ if for any distribution $\distribution$ over $\domain \times \set{\pm 1}$, $\Pr_{S \sim \distribution^{m}}(\errD(\innerAlg(S)) > \opt + \alpha) < \beta$ where $\opt := \inf_{h \in \hypotheses} \errD(h)$ and $m$ is the sample complexity of $\innerAlg$. 

    We say that $\innerAlg$ is a realizable $(\alpha, \beta)$-learner for $\hypotheses$ if for any $\distribution$ over pairs $(x, h(x))$ where $x \sim \distribution_{\domain}$ and $h \in \hypotheses$, $\Pr_{S \sim \distribution^{m}}(\errD(\innerAlg(S)) > \alpha) < \beta$.
\end{definition}

We note that any class of VC dimension $d$ has an agnostic $(\alpha, \beta)$-learner with $\bigO{\frac{d+\log \frac{1}{\beta}}{\alpha^{2}}}$ samples \cite{haussler1992decision} and an (improper) realizable $(\alpha, \beta)$-learner with $\bigO{\frac{d + \log(1/\beta)}{\alpha}}$ samples \cite{hanneke2016optimal, larsen2023bagging}.

\paragraph{Replicability}
We use $\innerAlg$ to denote (possibly randomized) algorithms, while $\innerAlg(S; r)$ denotes executing $\innerAlg$ on samples $S$ with internal randomness $r$.

\begin{definition}[Replicability \cite{ImpLPS22}]
    \label{def:repl}
    A randomized algorithm is $\rho$-replicable if for every distribution $\distribution$,
    \begin{equation*}
        \Pr_{S_1, S_2 \sim \distribution^{m}, r} \left( \innerAlg(S_1; r) \neq \innerAlg(S_2; r) \right) < \rho \text{.}
    \end{equation*}
\end{definition}

\begin{lemma}[Correlated Sampling, Lemma 7.5 of \cite{rao2020communication}]
    \label{lemma:correlated-sampling}
    Let $\domain$ be a finite domain.
    There is a randomized algorithm $\corrSamp$ (with internal randomness $\xi \sim \distribution_{R}$) such that given any distribution over $\domain$ outputs a random variable over $\domain$ satisfying the following:
    \begin{enumerate}
        \item (Marginal Correctness) For all distributions $\distribution$ over $\domain$, 
        \begin{equation*}
            \Pr_{\xi \sim \distribution_{R}} (\corrSamp(\distribution, \xi) = x) = \Pr_{X \sim \distribution}(X = x) .
        \end{equation*}
        \item (Error Guarantee) For all distributions $\distribution, \distribution'$ over $\domain$,
        \begin{equation*}
            \Pr_{\xi \sim \distribution_{R}} \left(\corrSamp(\distribution, \xi) \neq \corrSamp(\distribution', \xi)\right) \leq 2 \cdot \tvd{\distribution, \distribution'} .
        \end{equation*}
    \end{enumerate}

    Furthermore, the algorithm runs in expected time $\tO{|\domain|}$.
\end{lemma}



The learnability of a class using replicable algorithms is characterized by what is known as Littlestone Dimension.
The Littlestone dimension is a combinatorial parameter that characterizes regret bounds in Online Learning \cite{littlestone1988learning, ben2009agnostic}.

\begin{definition}[Littlestone Dimension]
    \label{def:littlestone}
    Let $\hypotheses$ be a hypothesis class over $\domain$.
    A mistake tree is a binary decision tree whose internal nodes are labeled by elements of $\domain$.
    Any root-to-leaf path in a mistake tree can be described as a sequence of examples $(x_{1}, y_{1}), \dotsc, (x_{d}, y_{d})$ where $x_{i}$ is the label of the $i$-th node in the path and $y_{i} = 1$ if the $(i + 1)$-th node in the path is the right child of the $i$-th node and otherwise $y_{i} = -1$.
    We say that a tree $T$ is shattered by $\hypotheses$ if for any in $T$, there is $h \in \hypotheses$ such that $h(x_{i}) = y_{i}$, for all $i \leq d$.
    The Littlestone dimension of $\hypotheses$, denoted by $\LDim(\hypotheses)$, is the depth of largest complete tree that is shattered by $\hypotheses$.
\end{definition}

\paragraph{Differential Privacy}

We define differential privacy.

\begin{definition}
    \label{def:private}
    Let $\innerAlg: \domain^{m} \rightarrow \range$ be an algorithm.
    Two datasets $S, S' \in \domain^{m}$ are neighboring if are identical in all but one sample.
    $\innerAlg$ is $(\varepsilon, \delta)$-private if for every pair of neighboring datasets $S ,S'$ and every set of outputs $Y \subset \range$, 
    \begin{equation*}
        \Pr(\innerAlg(S) \in Y) \leq e^{\eps} \Pr(\innerAlg(S') \in Y) + \delta \text{.}
    \end{equation*}
\end{definition}

\paragraph{Distribution Divergence and Distance}

In the following, let $X, Y$ be random variables.
Unless otherwise specified, the random variables are over domain $\domain$.

\begin{definition}
    \label{def:tvd}
    The \emph{total variation distance} between $X$ and $Y$ is
    \begin{equation*}
        \tvd{X, Y} = \frac{1}{2} \sum_{x \in \domain} \left| \Pr(X = x) - \Pr(Y = x) \right| \text{.}
    \end{equation*}
\end{definition}

%% file: repl_prediction.tex
\section{Replicable Prediction}

In this section, we examine the special case of pointwise replicability in binary classification (over $\set{\pm 1}$), which we call \textit{replicable prediction}. 
In particular, we show that any standard PAC-learner can be amplified to a (pointwise) replicable predictor with only polynomial overhead in $\rho$.

\begin{theorem}[Formal \Cref{thm:predict}]
    \label{thm:predict-formal}
    Let $\Acal$ be an (agnostic) $(\alpha, \beta)$-learner on $m(\alpha, \beta)$ samples over a countable domain.
    There exists an (agnostic) $\rho$-pointwise replicable $(\alpha, \beta)$-learner with sample complexity
    \[
    \bigtO{\frac{m(\alpha, \rho^2 \beta)}{\rho^2} + \frac{\log^2(1/\beta)}{\rho^2 \alpha^2} + \frac{1}{\rho^2 \alpha^{4}} + \frac{\log^2(1/\min(\beta, \rho))}{\alpha^4}} \text{.}
    \]
    Furthermore, our algorithm runs in time linear in sample complexity with $\bigO{\frac{1}{\rho^2}}$ oracle calls to $\Acal$.
    In the realizable setting, the required sample complexity is \[
    \bigtO{\frac{m(\alpha, \rho^2 \beta)}{\rho^2} + \frac{\log^2(1/\beta)}{\rho^2 \alpha^2}}\text{.}
    \]
\end{theorem}

Our algorithm requires several new sub-routines.
First, we give a pointwise replicable procedure for selecting heavy hitters (i.e. elements in the domain with large mass under the input distribution), which has been previously studied in the fully replicable setting \cite{ImpLPS22, kalavasis2023statistical, hopkins2024replicability}.

\begin{restatable}[Pointwise Replicable Heavy Hitters]{proposition}{prHeavyHittersProp}
    \label{prop:finding-heavy}
    There is an algorithm $\Acal$ that given $\nu, \beta, \rho > 0$ and sample access to $\distribution$ over $\domain$, returns $S \subset \domain$ satisfying the following:
    \begin{enumerate}
        \item (Pointwise Replicability) Let $S, T$ denote the output of two runs of $\Acal$. 
        For every $x \in \domain$, 
        \begin{equation*}
            \Pr\left(\ind[x \in S] + \ind[x \in T] = 1 \right) < \rho \text{.}
        \end{equation*}
        \item (Completeness) With probability $1 - \beta/2$, $S$ contains all $x$ with $\distribution(x) > 10 \nu$.
        \item (Soundness) With probability $1 - \beta/2$, $S$ does not contain any $x$ with $\distribution(x) < \nu/10$.
    \end{enumerate}
    Furthermore, the algorithm has sample complexity $\bigO{\frac{\log(1/\min(\beta, \rho)v)}{\nu \rho^2}} = \bigtO{\frac{\log(1/\beta)}{\nu \rho^2}}$.
\end{restatable}

We note that our algorithm is similar to previous algorithm for replicably computing heavy hitters.
However, we crucially use the weaker condition of pointwise replicability to obtain an algorithm with $v^{-1}$ sample complexity (in contrast to the $v^{-3}$ sample complexity required by known fully replicable algorithms).
Next, we show that we can pointwise replicably learn accurate labels for heavy hitters.

\begin{restatable}[Labeling Heavy Hitters]{proposition}{LabelHeavyHitters}
    \label{prop:learn-hh-label}
    There is an algorithm $\Acal$ that given $S \subset \domain$ and $\nu, \alpha, \beta, \rho > 0$ satisfying $\distribution(x) > \nu$ for all $x \in S$, outputs $\ell(x)$ for all $x \in S$ satisfying the following:
    \begin{enumerate}
        \item (Pointwise Replicability) For fixed $x \in S$, let $\ell_1(x), \ell_2(x)$ denote the label of $x$ over two runs of $\Acal$. 
        Then, for each $x$, $\Pr(\ell_1(x) \neq \ell_2(x)) < \rho$.

        \item (Accuracy) With probability $1 - \beta$, $\ell(x) = \sign(y(x))$ whenever $|y(x)| \geq \alpha$.
        Here, $y(x) := \E_{(x', y') \sim \distribution}[y'|x' = x]$.
    \end{enumerate}
    Furthermore, the algorithm has sample complexity $\bigO{\frac{1}{\alpha^2 \rho^2 \nu} + \frac{\log(|S|/\min(\beta, \rho))}{\alpha^2 \nu}}$.
\end{restatable}

We defer the proofs of \Cref{prop:finding-heavy} and \Cref{prop:learn-hh-label} to \Cref{sec:heavy-hitters}.
Here, we show how to use these as sub-routines to obtain an algorithm for replicable prediction. We first present our algorithm:
    
    \IncMargin{1em}
    \begin{algorithm}[!ht]
    
    \SetKwInOut{Input}{Input}\SetKwInOut{Output}{Output}\SetKwInOut{Parameters}{Parameters}
    \Input{PAC-learner $\Acal$ with sample complexity $m(\alpha, \beta)$.}
    \Parameters{$\alpha$ accuracy, $\beta$ error probability, and $\rho$ replicability}
    \Output{$\rho$-pointwise replicable $(\alpha, \beta)$-accurate PAC-learner $\tilde{\Acal}$.}
    
    \caption{$\basicNonUniformRep(\Acal, \alpha, \beta, \rho)$}
    \label{alg:chernoff-pred}

    Let $\beta \leq \alpha, T \gets \frac{C}{\rho^2}, \nu \gets \frac{\alpha^2}{C \log(1/\beta)}$ where $C$ is a sufficiently large constant.

    Let $S \gets \prHeavyHitters(0.1 \nu, \beta, \rho)$ from \Cref{prop:finding-heavy}.

    Let $\set{\ell(x)}_{x \in S} \gets \hhLabel(S, 0.01 \nu, \beta, \rho)$ from \Cref{prop:learn-hh-label}.
    
    Run $\Acal$ on $T$ fresh samples of size $m(\alpha, \beta/T)$, obtaining hypotheses $h_{t}: \domain \rightarrow \set{\pm 1}$ for $t \in [T]$.

    For every $x \in \domain$, define $\hat{p}(x) = \frac{1}{T} \sum_{t = 1}^{T} h_{t}(x)$ and draw $r(x) \sim \UnifD[-1, 1]$.

    \Return hypothesis $\tilde{h}(x) = \begin{cases}
        \ell(x) & x \in S \\
        1 & x \not\in S \andT \hat{p}(x) \geq r(x) \\
        -1 & x \not\in S \andT \hat{p}(x) < r(x)
    \end{cases}$
    
    \end{algorithm}
    \DecMargin{1em}

\begin{proof}[Proof of \Cref{thm:predict-formal}]

    We begin by arguing that our algorithm is $2 \rho$-pointwise replicable.
    Increasing the sample complexity by a constant factor allows us to obtain $\rho$-pointwise replicability.
    For $i \in \set{1, 2}$, let $S^{(i)}, \ell^{(i)}, \hat{p}^{(i)}, \tilde{h}^{(i)}$ the corresponding values computed in two separate runs of our algorithm.
    Let $H = \set{h_{1}, \dots, h_{T}}$ denote the hypotheses obtained from running $\Acal$ on $T$ fresh samples and $H^{(1)}, H^{(2)}$ the corresponding set of hypotheses computed in two separate runs of our algorithm.
    Define $p(x) := \E_{h \sim \Acal}[h(x)]$ to be the expected value of $h(x)$ over the randomness of the samples given to $\Acal$ and the internal randomness of the algorithm.
    
    Fix a point $x \in \domain$.
    By \Cref{prop:finding-heavy}, we have with probability $1 - \rho$, either $x$ is in both $S^{(1)}, S^{(2)}$ or neither.
    If $x \in S^{(1)} \cap S^{(2)}$, \Cref{prop:learn-hh-label} guarantees that $\ell(x)$ is consistent with probability $1 - \rho$, thus ensuring $\tilde{h}^{(1)}(x) = \tilde{h}^{(2)}(x)$ with probability $1 - 2 \rho$ via a union bound.
    
    On the other hand, if $x$ is in neither $S^{(1)}, S^{(2)}$, we have that $\tilde{h}(x)$ is a function of $\hat{p}^{(1)}(x), \hat{p}^{(2)}(x)$ respectively.
    We bound the probability that $r(x)$ lies between $\hat{p}^{(1)}(x), \hat{p}^{(2)}(x)$.

    \begin{lemma}
        \label{lem:p-estimate}
        For any fixed $x \in \domain$ and $r \sim \UnifD[-1, 1]$,
        \begin{equation*}
            \Pr_{H^{(1)}, H^{(2)}, r} \left( \min(\hat{p}^{(1)}(x), \hat{p}^{(2)}(x)) \leq r \leq \max(\hat{p}^{(1)}(x), \hat{p}^{(2)}(x)) \right) < \rho \text{.}
        \end{equation*}
    \end{lemma}

    \begin{proof}
        For any fixed $x$, and observe that for large enough $C$, $\var(\hat{p}(x)) \leq \frac{\rho^{2}}{200}$ since $\hat{p}(x)$ is a sum of $T$ independent Rademacher random variables with mean $p(x)$.
        Then, 
        \begin{align*}
            \Pr_{H^{(1)}, H^{(2)}, r} \left( \min(\hat{p}^{(1)}(x), \hat{p}^{(2)}(x)) \leq r \leq \max(\hat{p}^{(1)}(x), \hat{p}^{(2)}(x)) \right)^{2} &\leq \E_{H^{(1)}, H^{(2)}} \left[ \left| \hat{p}^{(1)}(x) - \hat{p}^{(2)}(x) \right| \right]^{2} \\
            &\leq \E_{H^{(1)}, H^{(2)}} \left[ \left( \hat{p}^{(1)}(x) - \hat{p}^{(2)}(x) \right)^{2} \right] \\
            &= 2 \var \left[ \hat{p}(x) \right] \leq \frac{\rho^2}{100} \text{.}
        \end{align*}
       The first inequality follows from the fact that $r$ is distributed uniformly in $[-1, 1]$.
       In the second inequality, we have used Jensen's inequality, and the final equality is the definition of variance.
       We conclude by taking the square root on both sides.
    \end{proof}

    Since $\tilde{h}(x)$ depends only on $\ind[\hat{p}(x) > r(x)]$, we conclude that $\tilde{h}^{(1)}(x) = \tilde{h}^{(2)}(x)$ with probability $1 - \rho$ since $r$ does not lie between $\hat{p}^{(1)}(x), \hat{p}^{(2)}(x)$.
    Again, we obtain $2 \rho$-pointwise replicability via a union bound.

    In the remainder of the proof, we argue our algorithm produces an accurate hypothesis.
    We begin with some notation.
    Let $S^* := \set{x \in S \given \distribution(x) \geq \nu}$ denote the set of heavy hitters in the marginal distribution.
    Note that $|S^*| \leq 1/\nu$.
    Let $\distribution(S^*) = \sum_{x \in S^*} \distribution(x)$ denote the mass of all the heavy hitters.
    For a given $x$, recall that $y(x) := \E_{(x', y') \sim \distribution}[y'|x' = x]$ is the expected label of $x$ under the distribution $\distribution$.
    Let $h^*$ denote the optimal hypothesis in $\hypotheses$.

    \paragraph{Case 1: The heavy hitters are dominant.}
    In the first case, we assume that a significant fraction of the mass lies on the heavy hitters i.e.\ $\distribution(S^*) \geq 1 - \alpha$.
    For any hypothesis $h$, we write its error as
    \begin{equation*}
        \errD(h) = \Pr_{(x, y) \sim \distribution} (h(x) \neq y, x \in S^*) + \Pr_{(x, y) \sim \distribution}(h(x) \neq y, x \not\in S^*) \text{.}
    \end{equation*}
    In particular, under the assumption $\distribution(S^*) \geq 1 - \alpha$, we have 
    \begin{equation*}
        \Pr_{(x, y) \sim \distribution}(\tilde{h}(x) \neq y, x \not\in S^*)  \leq \alpha \text{.}
    \end{equation*}
    
    We now bound the first term.
    Here, we will argue that we explicitly learn the optimal labels of the heavy hitters up to error $\alpha$.
    Fix $x_0 \in S^*$.
    Any function $h: \domain \rightarrow \set{\pm 1}$ (including $h^*$) has error on $x_0$ at least
    \begin{align*}
        \Pr_{(x, y) \sim \distribution}(h(x) \neq y | x = x_0) &= \Pr_{(x, y) \sim \distribution}(h(x) = -1, y = 1 | x = x_0) + \Pr_{(x, y) \sim \distribution}(h(x) = 1, y = -1|x = x_0) \\
        &= \ind[h(x_0) = -1] \frac{y(x_0) + 1}{2} + \ind[h(x_0) = 1] \frac{1 - y(x_0)}{2} \\
        &= \frac{1}{2} \begin{cases}
            1 - |y(x_0)| & h(x_0) \neq \sign(y(x_0)) \\
            1 + |y(x_0)| & h(x_0) \neq \sign(y(x_0)) 
        \end{cases}\text{.}
    \end{align*}

    We now upper bound the error of $\tilde{h}(x)$ on $S^*$.  
    By the completeness property of \Cref{prop:finding-heavy}, with probability at least $1 - \beta$, we have $S^* \subseteq S$.
    On the other hand, the soundness property of \Cref{prop:finding-heavy} ensures that all $x \in S$ have $\distribution(x) \geq 0.01\nu$, so that \Cref{prop:learn-hh-label} guarantees that with probability at least $1 - \beta$, we have $\tilde{h}(x) = \ell(x) = \sign(y(x))$ for all $x \in S$ with $|y(x)| \geq \alpha$.
    Thus, we have
    \begin{equation*}
        \Pr_{(x, y) \sim \distribution}(\tilde{h}(x) \neq y \mid x = x_{0}) \leq \frac{1}{2} \begin{cases}
            1 - |y(x_0)| & |y(x_0)| \geq \alpha \\
            1 + |y(x_0)| & |y(x_0)| < \alpha
        \end{cases}
    \end{equation*}
    In particular, for any function $h$ and $x_{0} \in S^*$, we have
    \begin{equation*}
        \Pr_{(x, y) \sim \distribution}(\tilde{h}(x) \neq y \mid x = x_{0}) - \Pr_{(x, y) \sim \distribution}(h(x) \neq y \mid x = x_{0}) \leq \alpha 
    \end{equation*}
    since the two probabilities are equal whenever $|y(x_0)| > \alpha$ and otherwise the difference is at most $|y(x_0)| \leq \alpha$.
    Then, for any function $h$ (not necessarily in $\hypotheses$) we may bound the first term as
    \begin{align*}
        \Pr_{(x, y) \sim \distribution} (\tilde{h}(x) \neq y, x \in S^*) &= \sum_{x_0 \in S^*} \Pr_{(x, y) \sim \distribution}(\tilde{h}(x) \neq y, x = x_0) \\
        &= \sum_{x_0 \in S^*} \distribution(x_0) \Pr_{(x, y) \sim \distribution}(\tilde{h}(x) \neq y \mid x = x_0) \\
        &\leq \sum_{x_0 \in S^*} \distribution(x_0) \left( \Pr_{(x, y) \sim \distribution}(h(x) \neq y \mid  x = x_0) + \alpha \right) \\
        &\leq \Pr_{(x, y) \sim \distribution}(h(x) \neq y, x \in S^*) + \alpha \text{.}
    \end{align*}
    In particular, we obtain a hypothesis $\tilde{h}$ such that $\errD(\tilde{h}) \leq \errD(h^*) + 2 \alpha$ with probability $1 - 2 \beta$, since $\errD(h^*)$ is at least the error of $h^*$ on $S^*$, and above we upper bounded the error of $\tilde{h}$ with respect to any function $h$.
    As before, by scaling the sample complexity by a constant factor, we can obtain an $(\alpha, \beta)$-learner.
    
    \paragraph{Case 2: The heavy hitters are not dominant.}
    In the second case, we have that a non-trivial fraction of the distribution is not concentrated on heavy hitters, i.e., that $\distribution(S^*) < 1 - \alpha$.
    As before, since $S^* \subseteq S$ and every element $x \in S$ has $\distribution(x) \geq 0.01 \nu$, our above argument in fact yields 
    \begin{equation*}
        \Pr_{(x, y) \sim \distribution} (\tilde{h}(x) \neq y, x \in S) \leq \Pr_{(x, y) \sim \distribution}(f(x) \neq y, x \in S) + \alpha 
    \end{equation*}
    for \textit{any} function $f: S \rightarrow \set{\pm 1}$.
    We are thus interested in bounding the quantity
    \begin{equation*}
        \Pr_{(x, y) \sim \distribution}(\tilde{h}(x) \neq y, x \not\in S) \text{.}
    \end{equation*}
    Previously, we naively upper bounded this quantity by $\Pr(x \not\in S) \leq \Pr(x \not\in S^*) \leq \alpha$, which is no longer possible.
    Instead, we use concentration bounds to argue that with high probability the quantity is not much larger than $\Pr(h^*(x) \neq y, x \not\in S)$, for an optimal hypothesis $h^* \in \hypotheses$.
    
    Toward this end, we observe the expected error of our algorithm is the average of error of $h_1, \dots, h_T$.
    



    \begin{lemma}
        \label{lem:exp-error-ub-pred}
        \begin{equation*}
            \E_{r} \left[ \Pr_{(x, y) \sim \distribution} (\tilde{h}(x) \neq y, x \not\in S) \right] = \E_{t} \left[ \Pr_{(x, y) \sim \distribution} \Pr(h_t(x) \neq y, x \not\in S) \right] \text{.}
        \end{equation*}
    \end{lemma}

    \begin{proof}
        Recall that for any $x_0$ and any function $h$, $\Pr_{(x, y) \sim \distribution} (h(x) \neq y \mid x = x_0) = \frac{1 - h(x_0) y(x_0)}{2}$.
        Then we can write the expected error of $\tilde{h}$ in $\domain \setminus S$ over the random choice of $r$ as
        \begin{align*}
            \E_{r} \left[ \Pr_{(x, y) \sim \distribution} (\tilde{h}(x) \neq y, x \not\in S) \right] = \frac{1}{2} \E_{r} \left[ \sum_{x \not\in S} \distribution(x) (1 - \tilde{h}(x) y(x)) \right] \text{.}
        \end{align*}
        In the above equation, note that only $\tilde{h}(x)$ depends on $r$, and for $x \not\in S$, $\E_{r}[\tilde{h}(x)] = \hat{p}(x)$ since
        \begin{equation*}
            \E_{r}[\tilde{h}(x)] = \Pr_{r}(\tilde{h}(x) = 1) - \Pr_{r}(\tilde{h}(x) = - 1) = 2 \Pr(\tilde{h}(x) = 1) - 1 = 2 \frac{1 + \hat{p}(x)}{2} - 1 = \hat{p}(x) 
        \end{equation*}
        where in the third equality, we have $\tilde{h}(x) = 1$ if $r(x) \leq \hat{p}(x)$ for $r(x) \sim \UnifD[-1, 1]$.
        Thus, recalling $\hat{p}(x) = \frac{1}{T} \sum_{t} h_t(x)$,
        \begin{align*}
            \E_{r} \left[ \Pr_{(x, y) \sim \distribution} (\tilde{h}(x) \neq y, x \not\in S) \right] &= \frac{1}{2} \sum_{x \not\in S} \distribution(x) (1 - \hat{p}(x)) y(x) \\
            &= \frac{1}{2} \sum_{x \not\in S} \distribution(x) \left(1  - \frac{1}{T} \sum_{t = 1}^{T} h_t (x)\right) y(x) \\
            &= \frac{1}{2T} \sum_{x \not\in S} \distribution(x) \left(\sum_{t = 1}^{T} (1 - h_t (x)) \right) y(x) \\
            &= \frac{1}{2T} \sum_{t = 1}^{T} \sum_{x \not\in S}  \distribution(x) ( 1 - h_t (x)) y(x) \\
            &= \frac{1}{T} \sum_{t = 1}^{T} \frac{1}{2} \sum_{x \not\in S} \distribution(x) (1 - h_t(x)) y(x) \text{.}
        \end{align*}
        Finally, we observe that each summand is exactly the error of $h_t$ on $x \not\in S$, so that we may conclude
        \begin{align*}
            \E_{r} \left[ \Pr_{(x, y) \sim \distribution} (\tilde{h}(x) \neq y, x \not\in S) \right] &= \frac{1}{T} \sum_{t = 1}^{T} \Pr_{(x, y) \sim \distribution} \Pr(h_t(x) \neq y, x \not\in S) \\
            &= \E_{t} \left[ \Pr_{(x, y) \sim \distribution} \Pr(h_t(x) \neq y, x \not\in S) \right] \text{.} 
        \end{align*}
    \end{proof}
    Now, we argue that that for any \textit{fixed} choice of $h_1,\ldots,h_T$, the random variable 
    \[
    X := \Pr_{(x, y) \sim \distribution}\left(\tilde{h}(x) \neq y, x \not\in S\right)
    \]
    is well concentrated around its mean over the choice of $r$.
    Since we assume that the domain is countable, we have
    \begin{align*}
        X &= \sum_{x_0 \not\in S} \Pr_{(x, y) \sim \distribution} (\tilde{h}(x) \neq y, x = x_0) \text{.}
    \end{align*}
    Define $J(x_0) = \Pr_{(x, y) \sim \distribution} (\tilde{h}(x) \neq y, x = x_0)$ if $x_0 \not\in S$ and $0$ otherwise so that $X = \sum J(x_0)$.
    We argue $X$ is well concentrated around its mean using a variant of Hoeffding's inequality.
    
    \begin{restatable}{theorem}{CountableTail}
        \label{thm:countable-tail-bound}
        Let $S = \sum_{i = 1}^{\infty} X_{i}$ where $X_{i} \in [0, b_i]$ are independent.
        Suppose $\sum_{i = 1}^{\infty} b_i \leq 1$ and $\max_{i} b_{i} \leq \nu$.
        Let $\mu = \E[S]$.
        Then, for any $t > 0$,
        \begin{equation*}
            \Pr(|S - \mu| > t) < \exp\left( -\Omega(t^2/\nu) \right) \text{.}
        \end{equation*}
    \end{restatable}
    
    We defer the proof of \Cref{thm:countable-tail-bound} to \Cref{app:concentration}.
    By \Cref{prop:finding-heavy}, with probability $1 - \beta$, every $x_0 \not\in S$ satisfies $\distribution(x_0) < \nu$.
    Thus, each $J(x_0)$ is a random variable bounded between $[0, \distribution(x_0)]$.
    Furthermore, note that each $J(x_0)$ depends only on $\hat p(x_0)$ (which is fixed), and an independent choice of $r(x_0)$. In particular, conditioned on fixed $\hat{p}$, the $J(x_0)$ are independent and satisfy the assumptions of \Cref{thm:countable-tail-bound}. We may conclude that
    \begin{equation*}
        \Pr_{r}\left(\left|X - \E_{r}[X]\right| > \alpha\right) < \exp(-\Omega(\alpha^2/\nu)) < \beta \text{.}
    \end{equation*}
    whenever $\nu = \Theta(\alpha^2 /\log(1/\beta))$ for a sufficiently small constant.
    Thus, we have with probability $1 - 2 \beta$, the $|X - \E[X]| \leq \alpha$.
    Combined with \Cref{lem:exp-error-ub-pred}, we have
    \begin{equation*}
        X \leq \E_{r}[X] + \alpha \leq \E_{t} \left[\Pr_{(x, y) \sim \distribution} \left( h_{t} (x) \neq y, x \not\in S \right)\right] + \alpha \text{.}
    \end{equation*}
    We now condition on the event that for every $t \in [T]$, $\errD(h_t) \leq \opt+\alpha$, which occurs with probability $ 1- \beta$.
    Then, union bounding over the failure of the above event, \Cref{prop:finding-heavy}, \Cref{prop:learn-hh-label}, and \Cref{thm:countable-tail-bound}, we have with probability $1 - 4 \beta$ that
    \begin{align*}
        \errD(\tilde{h}) &\leq \left(\Pr_{(x, y) \sim \distribution} (f(x) \neq y, x \in S) + \alpha \right) + \left( \E_{t} \left[\Pr_{(x, y) \sim \distribution} \left( h_{t}(x) \neq y, x \not\in S \right)\right] + \alpha \right) \\
        &\leq \E_{t} \left[\Pr_{(x, y) \sim \distribution} \left( h_{t}(x) \neq y \right)\right] + 2 \alpha \\
        &\leq \opt + 3 \alpha \text{.}
    \end{align*}
    In the first inequality, we combined our bounds on error of $\tilde{h}$ from the two cases.
    In the second inequality, we observe that $f$ is an arbitrary function on $S \rightarrow \set{\pm 1}$ and must have error less than any hypothesis $h_{t}$ and thus also less than any average of such hypotheses.
    In the final inequality, we recall that we condition on the event that the error of any $h_{t}$ (and therefore the average error over $t$) is at most $\opt + \alpha$.
    Thus, we obtain a $(3 \alpha, 4 \beta)$-learner.
    As before, increasing the sample complexity by a constant factor completes the proof.

    We conclude by bounding the sample complexity.
    In the agnostic setting, we have
    \begin{align*}
        m(\alpha, \beta, \rho) &= \bigtO{Tm(\alpha, \beta/T) + \frac{\log(1/\beta)}{\nu \rho^2} + \frac{1}{\alpha^2 \rho^2 \nu} + \frac{\log(1/\min(\beta, \rho))}{\alpha^2 \nu}} \\
        &= \bigtO{\frac{m(\alpha, \beta/T)}{\rho^2} + \frac{\log^2(1/\beta)}{\rho^2 \alpha^2} + \frac{1}{\rho^2 \alpha^{4}} + \frac{\log^2(1/\min(\beta, \rho))}{\alpha^4}} \text{.}
    \end{align*}
    In the realizable setting, we apply \Cref{prop:learn-hh-label-realizable} instead of \Cref{prop:learn-hh-label} to obtain
    \begin{align*}
        m(\alpha, \beta, \rho) &= \bigtO{Tm(\alpha, \beta/T) + \frac{\log(1/\beta)}{\nu \rho^2} + \frac{1}{\rho^2 \nu} + \frac{\log(1/\min(\beta, \rho))}{\nu}} \\
        &= \bigtO{\frac{m(\alpha, \beta/T)}{\rho^2} + \frac{\log^2(1/\beta)}{\rho^2 \alpha^2}} \text{.}
    \end{align*}
\end{proof}

\subsection{Learning and Labeling Heavy Hitters}
\label{sec:heavy-hitters}

In this section, we give our algorithms for finding heavy hitters and learning their labels.

\prHeavyHittersProp*

\begin{proof}
    We give the following algorithm.
    \IncMargin{1em}
    \begin{algorithm}
    
    \SetKwInOut{Input}{Input}\SetKwInOut{Output}{Output}\SetKwInOut{Parameters}{Parameters}
    \Input{Sample access to $\distribution$.}
    \Parameters{Threshold $\nu$, error $\beta$, and replicability $\rho$.}

    Let $C$ be a sufficiently large constant. 
    
    Draw $m_1 = \frac{C \log(1/(\min(\beta, \rho) v))}{v}$ samples from $\distribution$ and let $S_{\cand}$ denote all unique elements in sample.

    Draw $m_2 = \frac{\log(m_1/\min(\beta, \rho))}{v \rho^2}$ samples and for all $x \in S_{\cand}$, let $\hat{p}(x)$ denote the frequency of $x$ in $m_2$ samples.

    Draw $r \sim \UnifD[v/2, 2v]$.

    \Return $S^* = \set{x \in S_{\cand} \given \hat{p}(x) > r}$.
    
    \caption{$\prHeavyHitters(\distribution, \nu, \beta, \rho)$}
    \label{alg:finding-heavy}
    
    \end{algorithm}
    \DecMargin{1em}

    We argue the correctness and replicability of our algorithm conditioned on the following events. 
    Let $S^*$ denote the output set of the algorithm.
    \begin{enumerate}
        \item With probability $1 - \beta$: for all $\distribution(x) \geq \nu$, $x \in S_{\cand}$.
        \item With probability $1 - \beta$: for all $x \in S_{\cand}$, either 
        \begin{enumerate}
            \item $\distribution(x) < 0.1\nu$ and $\hat{p}(x) \leq 0.2\nu$, or
            \item $\distribution(x) > 10\nu$ and $\hat{p}(x) \geq 5\nu$, or
            \item $0.1\nu \leq \distribution(x) \leq 10\nu$ and $(1 - 0.01 \rho) \distribution(x) \leq \hat{p}(x) \leq (1 + 0.01 \rho) \distribution(x)$
        \end{enumerate}
    \end{enumerate}

    For the first event, we observe the probability $x$ fails to appear is at most $(1 - \nu)^{m_1} < \exp(-\nu m_1) < \beta \nu$ by choosing $m_1 = \log(1/(\beta \nu))/\nu$.
    A union bound over (at most) $1/\nu$ elements with $\distribution(x) \geq \nu$ allows us to conclude that the first event holds with probability at least $1 - \beta$.

    For the second event, fix $x \in S_{\cand}$.
    Note that $\hat{p}(x) \sim \BinomD(m_2, \distribution(x))$.
    We now consider three cases.

    \paragraph{Case 1: $\distribution(x) < 0.1\nu$.}
    Suppose $\distribution(x) \leq 0.1\nu$.
    Then, $\hat{p}(x) \sim \BinomD(m_2, \distribution(x))$ is stochastically dominated by $Z \sim \BinomD(m_2, 0.1\nu)$.
    In particular, a Chernoff bound yields $\hat{p}(x) > 0.2\nu$ with probability at most $\Pr(Z > 0.2\nu) < \exp(-\Omega(\nu m_2)) < \frac{\beta}{m_1}$ by choosing $m_2 = \Theta(\nu^{-1}\log(m_1/\beta))$ for a sufficiently large constant. 

    \paragraph{Case 2: $\distribution(x) > 10\nu$.}
    Suppose $\distribution(x) > 10\nu$.
    Let $Z \sim \BinomD(m_2, 10\nu)$ so that a similar argument to above yields
    \begin{equation*}
        \Pr(\hat{p}(x) < 5\nu) < \Pr(Z < 5\nu) < \exp(-\Omega(\nu m_2)) < \frac{\beta}{m_1}
    \end{equation*}
    for sufficiently large $m_2 = \Theta(\nu^{-1}\log(m_1/\beta))$.

    \paragraph{Case 3: $0.1\nu \leq \distribution(x) \leq 10\nu$.}
    Suppose $\distribution(x) \in [0.1\nu, 10\nu]$.
    Then, a Chernoff bound yields 
    \begin{equation*}
        \Pr\left(\hat{p}(x) \not\in \left(1 \pm 0.01 \rho \right) \distribution(x) \right) < \exp(-\Omega(\rho^2 m_2 \nu)) < \frac{\beta}{m_1}
    \end{equation*}
    for sufficiently large $m_2 = \Theta(\nu^{-1} \rho^{-2} \log(m_1/\beta))$.
    Finally, we conclude with a union bound over all $|S_{\cand}| \leq m_1$, and observe that the second event holds with probability $1 - \beta$.

    We now prove the correctness of our algorithm, assuming that the above two events hold.
    Note that our sample complexity allows us to assume $\beta \leq \rho$, as we replace $\beta$ with $\min(\beta, \rho)$.
    We will obtain an algorithm with error $2 \beta$ and pointwise replicability $3 \rho$.
    By increasing the sample complexity by a constant factor, we can obtain an algorithm with the desired parameters.
    
    We begin with pointwise replicability.
    Fix a single $x$.
    First, if $\distribution(x) < 0.1\nu$, the second event ensures that $\hat{p}(x) < 0.2\nu < r$ so $x \not\in S^*$ with probability $1 - \beta$.
    In particular, the probability that $x \in S^*$ in either run of the algorithm is at most $2 \beta \leq 2 \rho$.
    Similarly, if $\distribution(x) > 10 \nu$, the probability $x \not\in S^*$ in either run is at most $2 \beta$.
    Finally, suppose $\distribution(x) \in [0.1\nu, 10\nu]$.
    The second event ensures that $\hat{p}(x)$ lies in an interval of width at most $0.02 \rho \distribution(x) \leq 0.2 \rho v$ with probability $1 - \beta$.
    Over two runs of the algorithm, both $\hat{p}(x)$ lie within this interval with probability $1 - 2 \beta$.
    Since $r \sim \UnifD[v/2, 2v]$, the probability that $r$ lies between the two estimates is at most the probability $r$ lies in this interval, which we can upper bound by $\frac{0.2 \rho v}{1.5 v} \leq \rho$.
    Observe that the membership of $x \in S^*$ is determined directly by $\ind[\hat{p}(x) > r]$, so the above argument implies that our decision to include $x \in S^*$ is consistent over two runs.
    A union bound then bounds the probability that $x$ is included in only one $S^*$ is at most $3 \rho$, concluding the proof of the first property.

    We proceed with completeness.
    Let $\distribution(x) \geq 10\nu$, so that the first event ensures $x \in S_{\cand}$.
    The second event then ensures that $\hat{p}(x) > 5 \nu > r$ as desired.
    A union bound ensures that both events hold with probability $2 \beta$.
    Finally, note that soundness follows immediately from the second event.
\end{proof}

Next, we give an algorithm for learning the labels of heavy hitters of the marginal distribution.

\LabelHeavyHitters*

We require the replicable bias testing algorithm of \cite{aamand2025structure}, which is optimal in all parameters up to constant factors.

\begin{lemma}[Theorem 6.15 of \cite{aamand2025structure}]
    \label{lem:replicable-coin-testing}
    For any $0 \leq \beta \leq \rho \leq 1$ and $-1 \leq p_0 < q_0 \leq 1$, there is a $\rho$-replicable algorithm that with sample access to $\Rad(p)$, satisfies the following with probability at least $1 - \min(\beta, \exp(-1/\rho))$:
    \begin{enumerate}
        \item If $p \geq q_0$, then output $1$.
        \item If $p \leq p_0$, then output $-1$.
    \end{enumerate}
    Furthermore, the algorithm uses samples $\bigO{\frac{q_0'}{\alpha^2 \rho^2} + \frac{q_0' \log(1/\beta)}{\alpha^2}}$ where $\alpha = q_0 - p_0$ and $q_0' = \frac{q_0 + 1}{2}$.
\end{lemma}

We now present our algorithm for labeling heavy hitters.
    \IncMargin{1em}
    \begin{algorithm}[h!]
    
    \SetKwInOut{Input}{Input}\SetKwInOut{Output}{Output}\SetKwInOut{Parameters}{Parameters}
    \Input{Sample access to $\distribution$, heavy hitter set $S$.}
    \Parameters{Density $\nu$, accuracy $\alpha$, error $\beta$, and replicability $\rho$.}

    Let $C$ be a sufficiently large constant. 
    
    Draw $m = \frac{C}{\alpha^2 \rho^2 \nu} + \frac{C \log(|S|/\min(\beta, \rho))}{\alpha^2 \nu}$ samples from $\distribution$, denoted $T = \set{(x_i, y_i)}_{i = 1}^{m}$.

    \For{$x \in S$}{
        Let $\ell(x)$ be the output of \Cref{lem:replicable-coin-testing} on $T(x) \gets \set{i \in [m] \given x_i = x}$ with $p_0 = - \alpha, q_0 = + \alpha$, replicability $\rho$, and error $\beta/|S|$.
    }

    \Return $\ell(x)$ for all $x \in S$.
    
    \caption{$\hhLabel(\distribution, \nu, \alpha, \beta, \rho)$}
    \label{alg:learn-hh-label}
    
    \end{algorithm}
    \DecMargin{1em}
\begin{proof}

    We begin by arguing that for all $x \in S$, $T(x)$ is not too small and therefore $\hat{y}(x)$ is a good estimate of $y(x) := \E_{(x', y') \sim \distribution}[y'|x' = x]$.
    Note $|T(x)| \sim \BinomD(m, \distribution(x))$.
    Since $\distribution(x) \geq \nu$ for all $x \in S$, we have $\E[|T(x)|] \geq m \nu$ and a Chernoff bound ensures
    \begin{equation*}
        \Pr(|T(x)| < m \nu /2) < \exp(-\Omega(m \nu)) < \frac{\beta}{|S|}
    \end{equation*}
    as long as $m = \Theta(\log(|S|/\beta)/\nu)$ for some sufficiently large constant.
    A union bound ensures that with probability $1 - \beta$, $|T(x)| \geq m_0 := m \nu / 2$ for all $x \in S$.

    Then, since $T(x)$ consists of \iid samples from $\distribution$, we have at least $m_0$ \iid samples from the distribution $(x', y') \sim \distribution$ conditioned on $x' = x$.
    In particular, we have $m_0$ \iid samples from the distribution $\Rad(y(x))$.
    By \Cref{lem:replicable-coin-testing}, we need $m_0 = \bigO{\frac{1}{\alpha^2 \rho^2} + \frac{\log(|S|/\min(\beta, \rho))}{\alpha^2}}$.
    In particular, the overall sample complexity is $\bigO{\frac{1}{\alpha^2 \rho^2 \nu} + \frac{\log(|S|/\min(\beta, \rho))}{\alpha^2 \nu}}$

    We now prove the correctness of our algorithm.
    As before, our sample complexity allows us to assume $\delta \leq \rho$ by setting $\delta \gets \min(\delta, \rho)$.
    Towards pointwise replicability, fix an $x \in S$.
    With probability $1 - 2 \beta$, we guarantee that over two runs of the algorithm, both sets $T(x)^{(1)}, T(x)^{(2)}$ are sufficiently large, i.e., greater than $m_0$.
    Then, conditioned on this event, \Cref{lem:replicable-coin-testing} ensures that $\ell(x)$ is $\rho$-replicable over two runs of the algorithm.
    A union bound ensures that the same label $\ell(x)$ is output over two runs with probability at most $3 \rho$.

    Towards accuracy, we again condition on the event that $|T(x)| \geq m_0$.
    Then, with a union bound over all $|S|$, with probability $1 - \beta$, we have $\ell(x) = \sign(y(x))$ whenever $|y(x)| \geq \alpha$, as desired.
    Thus, we have obtained an algorithm that is $\alpha$-accurate with error $2 \beta$, and $3 \rho$-pointwise replicable.
    As before, we may increase the sample complexity by a constant factor to complete the proof of the proposition.
\end{proof}

In the realizable setting, we only need to obtain the accuracy guarantee when $y(x) \in \set{\pm 1}$. 
In this case, we obtain a slightly stronger sample complexity bound.

\begin{proposition}[Labelling Realizable Heavy Hitters]
    \label{prop:learn-hh-label-realizable}
    There is an algorithm $\Acal$ that given $S \subset \domain$ and $\nu, \alpha, \beta, \rho > 0$ satisfying $\distribution(x) > \nu$ for all $x \in S$, outputs $\ell(x)$ for all $x \in S$ satisfying the following:
    \begin{enumerate}
        \item (Pointwise Replicability) For fixed $x \in S$, let $\ell_1(x), \ell_2(x)$ denote the label of $x$ over two runs of $\Acal$. 
        Then, for each $x$, $\Pr(\ell_1(x) \neq \ell_2(x)) < \rho$.

        \item (Accuracy) With probability $1 - \beta$, $\ell(x) = \sign(y(x))$ whenever $y(x) \in \set{\pm 1}$.
        Here, $y(x) := \E_{(x', y') \sim \distribution}[y'|x' = x]$.
    \end{enumerate}
    Furthermore, the algorithm has sample complexity $\bigO{\frac{1}{\rho^2 \nu} + \frac{\log(|S|/\min(\beta, \rho))}{\nu}}$.
\end{proposition}

\begin{proof}
    The proof closely follows that of \Cref{prop:learn-hh-label-realizable}.
    In fact, we use an identical algorithm, with the following two modifications:
    \begin{enumerate}
        \item Set $m = \frac{C}{\rho^2 \nu} + \frac{C \log(|S|/\beta)}{\nu}$.
        \item Run \Cref{lem:replicable-coin-testing} with $p_0 = -1, q_0 = 1$, replicability $\rho$ and error $\beta/|S|$.
    \end{enumerate}
    The proof of correctness and replicability follows identically. 
    Notice however that we now only require $m_0 = \bigO{\frac{1}{\rho^2} + \log(|S|/\min(\beta, \rho))}$ and thus $m = \bigO{\frac{1}{\rho^2 \nu} + \frac{\log(|S|/\min(\beta, \rho))}{\nu}}$ as desired.
\end{proof}

\subsection{Lower Bounds for Replicable Prediction}

In this section, we give lower bounds for replicable prediction showing that pointwise replicability has at least quadratic overhead in the replicability parameter $\rho$ over standard learning, similar to the overhead seen in classical bias estimation \cite{ImpLPS22}. Our lower bounds apply to any VC class.

\ReplPredictionLB*

We begin with the agnostic lower bound. 
See \Cref{thm:repl-prediction-real-lb} for the realizable lower bound.
Our lower bound follows by reduction to bias estimation, also known as one-way marginals.

\begin{definition}[One-Way Marginals]
    \label{def:one-way-marginals}
    Let $\distribution_{p}$ be a product of $d$ Rademacher distributions with expectations $p = (p_1, \dotsc, p_{d})$.
    A vector $v \in \set{\pm 1}^{d}$ is an $\alpha$-accurate solution to the one-way marginals problem for $\distribution_{p}$ if $\max_{i = 1}^{d} (\sign(p_{i}) - v_{i}) p_i \leq 2 \alpha$.

    An algorithm $(\alpha, \beta)$-accurately solves the sign-one-way marginals problem with sample complexity $m$ if given any product of $d$ Rademacher distributions $\distribution$ and $m$ \iid samples from $\distribution$, with probability at least $1 - \beta$ the algorithm outputs an $\alpha$-accurate solution to the sign-one-way marginals problem for $\distribution$.
\end{definition}

We will use bias estimation to refer to the $1$-dimensional one-way marginals problem i.e. given sample access to $\Rad(p)$, return $1$ if $p \geq \alpha$ and $-1$ if $p \leq - \alpha$.

\begin{proof}
    Since $\hypotheses$ is has VC dimension $d$, there is a set of $d$ elements shattered by $\hypotheses$.
    Denote them $\set{x_{1}, \dotsc, x_{d}}$.
    We will use the following family of input distributions, parameterized by mean vectors $p \in [-1, +1]^{d}$.
    Fix a mean vector $p \in [-1, +1]^{d}$.
    First, a domain element $x_{i}$ is sampled uniformly.
    Then, if the domain element $x_{i}$ is sampled, the label is distributed according to $\Rad(p_{i})$.
    Suppose there is a pointwise replicable PAC-learner $\innerAlg$ for $\hypotheses$ using $m$ samples.
    Assume without loss of generality that $m \geq d$.
    In the proof, we in fact assume $\innerAlg$ learns a class with VC dimension $2d + 1$.
    
    Given $\innerAlg$, we design a $\rho$-replicable algorithm for bias estimation.
    To this end, suppose we are given sample access to a Rademacher distribution with mean $p \in [-1, +1]$ and are required to output $1$ if $p \leq - \alpha$ and $-1$ if $p \geq \alpha$.
    From \cite{ImpLPS22}, it is known that any algorithm for this task (that succeeds with error probability $\beta \leq 0.1$) requires $\bigOm{\frac{1}{\rho^2 \alpha^2}}$ samples and at least 1 bit of shared randomness, even when it is known that $p \in [-\alpha, \alpha]$.
    We give an algorithm for bias estimation in \Cref{alg:non-uniform-bias-estimator}.

    \IncMargin{1em}
    \begin{algorithm}
    
    \SetKwInOut{Input}{Input}\SetKwInOut{Output}{Output}\SetKwInOut{Parameters}{Parameters}
    \Input{$\rho$-pointwise replicable $(\alpha, 0.0001)$-accurate PAC-learner $\innerAlg$. Sample access to unknown Rademacher distribution $\distribution$ with mean $p \in [-\alpha, \alpha]$.}
    \Parameters{$\alpha$ accuracy and $\rho$ replicability}
    \Output{$2\rho$-replicable algorithm for bias estimation.}
    
    \caption{$\nonUniformtoBias(\innerAlg, \alpha, \rho)$}
    \label{alg:non-uniform-bias-estimator}

    Initiate $S$ as an empty multi-set.

    Initiate counter $c \gets 0$.

    Sample $r \in [2d + 1]$ uniformly.
    \label{alg:non-uniform:plant-elements}
    
    Split $[2d + 1] \setminus \set{r}$ into two random equal subsets $B_{+}, B_{-}$.
    \label{alg:non-uniform:split-elements}

    \For{$j \in [m]$}{
        Sample $x_{j} \in \domain$ uniformly.
        \label{alg:non-uniform:domain-sample}

        \If{$x_{j} = x_{r}$}{
            Draw a sample $y_{j} \sim \distribution$ and add $(x_{j}, y_{j})$ to $S$.

            $c \gets c + 1$.

            \If{$c > 200 \frac{m}{d} \log(1/\rho)$}{
                \label{alg:non-uniform:sample-limit}
                
                \Return $1$.
            }
        }\Else{
            \If{$x_{j} \in B_{+}$}{
                Draw $y_{j} \sim \Rad(\alpha)$.
            }
            \If{$x_{j} \in B_{-}$}{
                Draw $y_{j} \sim \Rad(-\alpha)$.
            }
            Add $(x_{j}, y_{j})$ to $S$.
        }
    }

    \Return $\innerAlg(S)(x_{r})$.
    
    \end{algorithm}
    \DecMargin{1em}

    We claim our algorithm is $\rho$-replicable.
    We condition on the event that the sample limit is not reached.
    Let $m_{r}$ denote the number of times $x_{r}$ is sampled so that $\mu := \E[m_{r}] = \frac{m}{2d + 1}$.
    By a standard Chernoff bound,
    \begin{equation*}
        \Pr\left( m_{r} > \frac{m}{d} \left(1 +  100 \log(1/\rho) \right) \right) < \rho \text{.}
    \end{equation*}
    Since $B_{-}, B_{+}$ are sampled with shared randomness, in both runs of the algorithm $\innerAlg$ is given sample access to the same distribution over $(\set{x_{1}, \dotsc, x_{d}} \times \set{\pm 1})^{m}$ where labels are parameterized according to some mean vector $\overline{p} \in [-1, 1]^{2d + 1}$.
    In particular, the label of $x_{r}$ is then identical with probability at least $1 - \rho$ by pointwise replicability, so algorithm is $2\rho$-replicable in total via a union bound.

    Next, we claim our algorithm is accurate.
    Suppose without loss of generality that $p = \alpha$. 
    (A similar argument holds when $p = - \alpha$ and when $p \in (-\alpha, \alpha)$ there is no correctness constraint).
    We analyze the error of any hypothesis $h: \domain \rightarrow \set{\pm 1}$.
    \begin{align*}
        \errD(h) &= \frac{1}{2d + 1} \sum_{i = 1}^{2d + 1} \Pr_{(x, y)}(y \neq h(x) | x = x_{i}) \\
        &= \frac{1}{2d + 1} \sum_{i = 1}^{2d + 1} \ind[h(x_{i}) = -1] \cdot \frac{1 + p_{i}}{2} + \ind[h(x_{i}) = 1] \cdot \frac{1 - p_{i}}{2} 
    \end{align*}
    where $p_{i}$ is the mean of the labels of $x_{i}$.
    Thus, the excess error of $h$ against the optimal hypothesis $h^*$ is
    \begin{align*}
        \errD(h) - \errD(h^*) &= \frac{\alpha}{2d + 1} \left( \sum_{i \in B_{+} \cup \set{r}} \ind[h(x_{i}) = -1] +\sum_{i \in B_{-}} \ind[h(x_{i}) = 1] \right) \text{.}
    \end{align*}
    We condition on the event that $h \gets \innerAlg(S)$ is $0.01 \alpha$-accurate and the sample limit is not reached.
    The first occurs with probability $0.9999$ by the correctness of $\innerAlg$, and by the above we then have
    \begin{equation*}
        \frac{1}{2d + 1} \sum_{i \in B_{+} \cup \set{r}} \ind[h(x_{i}) = -1] \leq 0.01 \text{.}
    \end{equation*}
    In particular, out of the $d + 1$ elements in $B_{+} \cup \set{r}$, at most a $0.02$-fraction of them have $h(x) = -1$.
    Now, for any fixed $B_{-}$, over the random choice of $r, B_{+}$, the samples are identically distributed so that $\innerAlg$ cannot distinguish different choices of $r$ in $B_{+} \cup \set{r}$.
    In particular, define $E_{+} := \set{i \in B_{+} \cup \set{r} \given h(x_{i}) = -1}$ so that $\frac{|E_{+}|}{d + 1} \leq 0.02$.
    Then, over the random choices of our algorithm and the samples observed from the unknown Rademacher distribution $\distribution$, we have
    \begin{equation*}
        \Pr(r \in E_{+}) \leq \max_{B_{-}} \Pr_{r, B_{+}} (r \in E_{+} | B_{-}) \leq 0.02 \text{.}
    \end{equation*}
    Thus, conditioned on the event that the output hypothesis $h$ is $0.01\alpha$-accurate, the probability that $h(x_{r}) = -1$ is at most $0.02$.
    By a union bound, we output $-1$ with probability at most $0.03$.

    Finally, the sample complexity of our algorithm is immediate.
    In particular, we have $\frac{m}{d} \log(1/\rho) = \bigOm{\rho^{-2} \alpha^{-2}}$ which implies the desired lower bound.

    To see that the algorithm requires shared randomness, note that the only stage \Cref{alg:non-uniform-bias-estimator} uses shared randomness is in sampling $B_{-}, B_{+}$.
    However, to learn $\hypotheses$ with respect to a distribution which is supported over a single element, there is no need to sample these sets, so the reduction uses no shared randomness.
    Therefore, $\innerAlg$ must use shared randomness to learn $\hypotheses$.
\end{proof}

\input{realizable_repl_pred_lb}

%% file: realizable_repl_pred_lb.tex
Above, we have shown a lower bound for pointwise replicable algorithms that learn in the agnostic setting.
Below, we give a lower bound that applies in the realizable setting.

\begin{theorem}
    \label{thm:repl-prediction-real-lb}
    Any $\rho$-pointwise replicable $(\alpha, 0.0001)$-learner in the realizable setting requires sample complexity
    \begin{equation*}
        \bigOm{\frac{d}{\rho^2 \alpha \log(1/\rho)}} \text{.}
    \end{equation*}
\end{theorem}

We give a reduction from a simpler version of bias estimation, where correctness is only required on the two constant distributions over $\set{\pm 1}$.

\begin{definition}
    \label{def:0-1-bias-estimation}
    An algorithm solves the $\set{\pm 1}$-bias estimation problem if given sample access to $\Rad(p)$, the algorithm returns $1$ with probability $2/3$ when $p = 1$ and returns $-1$ with probability $2/3$ when $p = -1$.
\end{definition}

Without replicability, there is a simple $1$-sample algorithm for $\set{\pm 1}$-bias estimation.
First, we note the following lower bound for replicable algorithms solving this problem.

\begin{theorem}
    \label{thm:0-1-bias-estimation}
    Any $\rho$-replicable algorithm for $\set{\pm 1}$-bias estimation requires $\bigOm{\frac{1}{\rho^2}}$ samples.
\end{theorem}

We defer the proof to \Cref{sec:omitted}, as it is a standard modification of known replicable bias estimation lower bounds \cite{ImpLPS22, hopkins2024replicability}.

\begin{proof}[Proof of \Cref{thm:repl-prediction-real-lb}]
    We reduce realizable learning to $\set{\pm 1}$-bias estimation.
    Assume without loss of generality that $\hypotheses$ has VC Dimension $d + 1$ and $d$ is odd.
    As before, since $\hypotheses$ has VC dimension $d + 1$, let $\set{x_{0}, x_{1}, \dots, x_{d}}$ denote $d + 1$ shattered elements.
    In contrast to the agnostic setting, we use a family of input distribution, parameterized by mean vectors $p \in \set{\pm 1}^{d + 1}$.
    First, we describe the domain distribution $\distribution_{\domain}$ which gives $\distribution_{\domain}(x_{0}) = 1 - 100 \alpha$ and $\distribution_{\domain}(x_{i}) = \frac{100 \alpha}{d}$ for $i > 0$.
    As before, when $x_{i}$ is sampled, the label is distributed according to $\Rad(p_{i})$.
    Suppose there is a pointwise replicable realizable PAC-learner $\innerAlg$ for $\hypotheses$ using $m$ samples.
    Assume without loss of generality that $m \geq d$.

    Given $\innerAlg$, we describe our algorithm for $\rho$-replicable $\set{\pm 1}$-bias estimation.
    Our algorithm is analogous to \Cref{alg:non-uniform-bias-estimator} except in Line \ref{alg:non-uniform:plant-elements}, we sample $r \in [d]$ uniformly; in Line \ref{alg:non-uniform:split-elements}, we split $[d] \setminus r$ into random equal subsets; in \Cref{alg:non-uniform:domain-sample} we sample from $\distribution_{\domain}$ given above instead of uniformly; in Line \ref{alg:non-uniform:sample-limit}, we instead set the sample limit to $200 \frac{\alpha m}{d} \log(1/\rho)$.
    Finally, we consistently assign the label $-1$ to element $x_{0}$, and assume the algorithm is aware of this (i.e. without loss of generality, assume the algorithm outputs $h(x_{0}) = -1$).

    We claim our algorithm is replicable, by again conditioning on the event that the sample limit is not reached.
    As before if $m_{r}$ is the number of times $x_{r}$ is sampled, $\E[m_{r}] = \frac{100 \alpha m}{d}$ so that 
    \begin{equation*}
        \Pr\left(m_{r} > \frac{10 \alpha m}{d} \left(1 + 100 \log(1/\rho) \right) \right) < \rho \text{,}
    \end{equation*}
    so that the sample limit is not reached.
    Then, since $\innerAlg$ is pointwise-replicable and both samples are sampled from the same distribution (recall $B_{+}, B_{-}$ are sampled with shared randomness), we ensure $\innerAlg(S)(x_{r})$ is consistent with probability $1 - \rho$, thus guaranteeing $2\rho$-replicability.

    Accuracy follows from a similar argument. 
    Without loss of generality, assume we have sample access to $\distribution \sim \Rad(1)$.
    As before, conditioned on any choice of $B_{-}$, the choice of $r$ in $B_{+} \cup \set{r}$ is indistinguishable to $\innerAlg$, so that the probability that $r \in E_{+} := \set{i \in B_{+} \cup \set{r} \given h(x_{i}) = -1}$ is at most $\frac{|E_{+}|}{(d - 1)/2}$.
    If we condition on $\innerAlg$ outputting an accurate hypothesis, we have $\frac{100 \alpha}{d} |E_{+}| < \alpha$ or $|E_{+}| \leq \frac{d}{100}$.
    In particular, the probability that $x_{r}$ is classified incorrectly is at most $0.03$.
    Thus, we obtain a $2 \rho$-replicable algorithm that solves the $\set{\pm 1}$-bias estimation problem with error probability $0.03$ and $\bigO{\frac{m \alpha \log(1/\rho)}{d}}$
    By \Cref{thm:0-1-bias-estimation}, this implies 
    \begin{equation*}
        m = \bigOm{\frac{d}{\alpha \rho^2 \log(1/\rho)}} \text{.}
    \end{equation*}
\end{proof}

%% file: approx_replicability.tex
\section{Approximate Replicability}

In this section, we give algorithms and lower bounds for approximately replicable PAC learning.

\subsection{From Pointwise to Approximate}

We show that any $\gamma$-pointwise replicable learner is automatically $(\rho, \frac{\gamma}{\rho})$-approximately replicable.

\begin{proposition}
    \label{prop:non-uniform-to-approx}
    Suppose $\innerAlg$ is a $\gamma$-pointwise replicable learner.
    Then, $\innerAlg$ is also $(\rho, \frac{\gamma}{\rho})$-approximately replicable.
\end{proposition}

\begin{proof}
    Suppose $\innerAlg$ is $\gamma$-pointwise replicable and fix an arbitrary distribution $\distribution$.
    Let $n$ denote the sample complexity of $\innerAlg$.
    By $\gamma$-pointwise replicability, we have
    \begin{align*}
        \E_{S, S' \sim \distribution^{n}, r \sim R} \left[ \Pr_{x \sim \distribution} \left[ A(S; r)(x) \neq A(S'; r)(x) \right] \right] &= \E_{S, S' \sim \distribution^{n}, r \sim R} \left[ \E_{x \sim \distribution} \left[ \ind \left( A(S; r)(x) \neq A(S'; r)(x) \right) \right] \right] \\
        &= \E_{x \sim \distribution} \left[ \E_{S, S' \sim \distribution^{n}, r \sim R} \left[ \ind \left( A(S; r)(x) \neq A(S'; r)(x) \right) \right] \right] \\
        &= \E_{x \sim \distribution} \left[ \Pr_{S, S' \sim \distribution^{n}, r \sim R} \left( A(S; r)(x) \neq A(S'; r)(x) \right) \right] \\
        &\leq \gamma \text{.}
    \end{align*}
    Thus, by Markov's inequality,
    \begin{equation*}
         \Pr_{S, S' \sim \distribution^{n}, r \sim R} \left( \Pr_{x \sim D} \left( A(S; r)(x) \neq A(S'; r)(x) \right) > \frac{\gamma}{\rho} \right) < \rho \text{.}
    \end{equation*}
\end{proof}

Combining \Cref{thm:predict-formal} and \Cref{prop:non-uniform-to-approx}, we immediately obtain an approximately replicable PAC-learner. 
In particular, we take a $\rho \gamma$-pointwise replicable learner to obtain a $(\rho, \gamma)$-approximately replicable learner.

\begin{corollary}
    \label{cor:approx-repl}
    Let $\innerAlg$ be an agnostic $(\alpha, \beta)$-learner on $m(\alpha, \beta)$ samples.
    There is an agnostic $(\rho, \gamma)$-approximately replicable $(\alpha, \beta)$--learner using 
    \[ 
    \bigtO{\frac{m(\alpha, \beta)}{\rho^2 \gamma^2} + \frac{\log^2(1/\beta)}{\rho^2 \gamma^2 \alpha^2} + \frac{1}{\rho^2 \gamma^2 \alpha^{4}} + \frac{\log^2(1/\min(\beta, \rho, \gamma))}{\alpha^4}}
    \] 
    samples.
    
    Furthermore, there is a realizable $(\rho, \gamma)$-approximately replicable $(\alpha, \beta)$--learner using 
    \[ 
    \bigtO{\frac{m(\alpha, \beta)}{\rho^2 \gamma^{2}} + \frac{\log^2(1/\beta)}{\rho^2 \gamma^2 \alpha^{2}}}
    \] 
    samples, where $m(\alpha, \beta)$ is the sample complexity of a realizable $(\alpha, \beta)$-learner.
\end{corollary}

\subsection{Improved Approximate Replicable Learning}

The above reduction, while simple, does not obtain optimal parameters. 
Below, we show that a more careful modification of our algorithm for pointwise replicable learning allows us to obtain a sample-optimal learner when the VC dimension of the target class is large.

\begin{theorem}[Formal \Cref{thm:apx-repl-informal}]
    \label{thm:apx-repl-chernoff}
    Let $\innerAlg$ be an (agnostic) $(\alpha, \beta)$-learner on $m(\alpha, \beta)$ samples.
    There exists an (agnostic) $(\rho, \gamma)$-approximately replicable $(\alpha, \beta)$-learner with sample complexity
    \begin{equation*}
        \bigtO{\frac{m(\alpha, \gamma^2 \beta)}{\gamma^2} + \frac{\log(1/\beta)}{\rho^2 \gamma^3} + \frac{\log^3(1/(\rho \beta \gamma))}{\alpha^{4} \gamma^2} + \frac{\log^{3}(1/(\rho \beta))}{\alpha^{2} \rho^{2}  \gamma^2}} \text{.}
    \end{equation*}
    Furthermore, our algorithm runs in time linear in sample complexity with $O(1/\gamma^2)$ oracle calls to $\innerAlg$.
\end{theorem}

Note here we've removed the assumption that $\innerAlg$ learns over a countable domain: we justify this in \Cref{app:domain-red}.
Our algorithm for agnostic learning will largely follow the framework of \Cref{thm:predict-formal}.
As a result, we begin by describing how to modify our algorithms for finding heavy hitters (\Cref{prop:finding-heavy}) and learning their labels (\Cref{prop:learn-hh-label}) to the approximately replicable setting.

\paragraph{Approximately Replicable Heavy Hitters}

We begin with an algorithm for heavy hitters. 
To ensure approximate replicability, we require that sufficiently heavy elements are identified \emph{fully replicably}, while light elements can be identified pointwise replicably. 
By sampling random strings independently for all domain elements, we can prove that the mass of light elements that are classified inconsistently concentrates around its mean with high probability.

\begin{proposition}[Approximately Replicable Heavy Hitters]
    \label{prop:apx-repl-heavy}
    Let $\nu, \beta, \rho, \gamma > 0$.
    There is an algorithm that given sample access to $\distribution$ over $\domain$, returns $S_{\gamma}, S, \gamma'$ with $\gamma' = \Theta(\gamma/\log(1/\rho))$ and $S_{\gamma} \subseteq S \subseteq \domain$ satisfying the following:
    \begin{enumerate}
        \item (Replicability) Let $S_{\gamma}^{(1)}, S^{(1)}, \gamma'^{(1)}$ and $S_{\gamma}^{(2)}, S^{(2)}, \gamma'^{(2)}$ denote the output of the algorithm over two runs with independent samples.
        With probability at least $1 - \rho$, $S_{\gamma}^{(1)} = S_{\gamma}^{(2)}$, $\gamma'^{(1)} = \gamma'^{(2)}$ and $\distribution(S^{(1)} \Delta S^{(2)}) < \gamma$.
        \item (Completeness) With probability $1 - \beta/2$, $S_{\gamma}$ contains all $x$ with $\distribution(x) > 2 \gamma'$ and $S$ contains all $x$ with $\distribution(x) > \min(2 \gamma', 10 \nu)$.
        \item (Soundness) With probability $1 - \beta/2$, $S_{\gamma}$ does not contain any $x$ with $\distribution(x) < \gamma'/2$ and $S$ does not contain any $x$ with $\distribution(x) < \min(\gamma'/2, \nu/10)$.
    \end{enumerate}
    Furthermore, the algorithm takes $\bigtO{\frac{\log(1/\beta)}{\rho^2 \gamma^3} + \frac{\log(1/\beta)}{\nu \gamma^2}}$ samples.
\end{proposition}

\begin{proof}
    Our algorithm proceeds in two stages.
    First, we invoke the fully replicable heavy hitters identification algorithm of \cite{kalavasis2023statistical, hopkins2024replicability}.

    \begin{lemma}[Replicable Heavy Hitters]
        \label{lem:r-heavy-hitters}
        Let $\nu, \beta, \rho > 0$. 
        There is a $\rho$-replicable algorithm that given sample access to $\distribution$ over $\domain$, returns $S \subset \domain$ satisfying the following with probability $1 - \beta$:
        \begin{enumerate}
            \item (Completeness) $S$ contains all $x$ with $\distribution(x) > 2 \nu$.
            \item (Soundness) $S$ does not contains any $x$ with $\distribution(x) < \nu/2$
        \end{enumerate}
        Furthermore, the algorithm has sample complexity $\bigtO{\frac{\log(1/\beta)}{\rho^2 \nu^3}}$.
    \end{lemma}

    Let $C$ be some sufficiently large constant.
    We invoke the above algorithm with parameters $\nu \gets \gamma' :=  \gamma/ C \log(1/\rho)$, $\beta \gets \beta$, $\rho \gets \rho$ and replicably obtain an output set $S_{\gamma}$.
    Next, we run \Cref{alg:finding-heavy} of \Cref{prop:finding-heavy} with one modification: instead of drawing a single random string $r$, we draw independent random strings $r(x) \sim \UnifD[\nu/2, 2\nu]$ for every $x \in S_{\cand} \setminus S_{\gamma}$.
    Then, we return $S_{\nu} = \set{x \in S_{\cand} \setminus S_{\gamma} \given \hat{p}(x) > r(x)}$.
    We invoke this algorithm with parameters $\nu \gets \nu$, $\beta \gets \beta$, $\rho \gets \gamma/C$.
    Finally, we output $\tilde{S} \gets S_{\gamma} \cup S_{\nu}$.
    
    We can bound the sample complexity as
    \begin{equation*}
        \bigtO{\frac{\log(1/\beta)}{\rho^{2} \gamma^{3}} + \frac{\log(1/\beta)}{\nu \gamma^{2}}} \text{.}
    \end{equation*}
    
    We prove that our algorithm achieves the desired guarantees.
    First, by \Cref{lem:r-heavy-hitters} and \Cref{prop:finding-heavy}, a union bound implies that with probability $1 - 2 \beta$, our output set $S_{\gamma} \cup S_{\nu}$ contains all elements with $\distribution(x) > \min(10 \nu, 2 \gamma')$ and no elements with $\distribution(x) < \min(\nu/10, \gamma'/2)$, proving the completeness and soundness guarantees.

    We now proceed to replicability.
    It is clear that $\gamma'$ is replicable, since it does not depend on any samples.
    Consider two runs of the algorithm on independent samples with outputs $\tilde{S}^{(1)} = S_{\cand}^{(1)} \cup S_{\nu}^{(1)}$ and $\tilde{S}^{(2)} = S_{\cand}^{(2)} \cup S_{\nu}^{(2)}$.
    Recall that for $i \in \set{1, 2}$, 
    \begin{equation*}
        S_{\nu}^{(i)} = \set{x \given \hat{p}^{(i)}(x) \geq r(x)}
    \end{equation*}
    is the set of elements whose empirical frequencies exceed shared random string $r(x)$.
    We condition on the following events:
    \begin{enumerate}
        \item $S_{\gamma}^{(1)} = S_{\gamma}^{(2)}$ contains all $x$ with $\distribution(x) > 2 \gamma'$ and no element with $\distribution(x) < \gamma'/2$.
        \item For all $x$ with $\distribution(x) < \nu/10$, $\hat{p}^{(1)}(x), \hat{p}^{(2)}(x) < \nu/5$.
        \item For all $x$ with $\distribution(x) > 10 \nu$, $\hat{p}^{(1)}(x), \hat{p}^{(2)}(x) > 5 \nu$.
        \item For all $x$ with $\distribution(x) \in [\nu/10, 10\nu]$, $(1 - 0.01 \gamma/C) \distribution(x) < \hat{p}^{(1)}(x), \hat{p}^{(2)}(x) < (1 + 0.01 \gamma/C) \distribution(x)$.
    \end{enumerate}
    The first event holds with probability $1 - \rho - \beta$ by the correctness and replicability of \Cref{lem:r-heavy-hitters}. 
    The remaining events hold with probability $1 - \beta$ from the proof of \Cref{prop:finding-heavy} (see Cases 1, 2, and 3).
    Thus, by a union bound, all of the above events hold with probability $1 - 2 \beta - \rho$.
    
    Following the first event, we have
    \begin{equation*}
        \tilde{S}^{(1)} \Delta \tilde{S}^{(2)} = (S_{\nu}^{(1)} \Delta S_{\nu}^{(2)}) \setminus S_{\gamma} \text{.}
    \end{equation*}
    By the completeness condition, we have $\distribution(x) < \min(10 \nu, 2 \gamma') \leq 2 \gamma'$ for all $x \not\in S_{\gamma}$.

    Conditioned on the above events, we bound $\distribution(\tilde{S}^{(1)} \Delta \tilde{S}^{(2)})$.
    Fix $x \not\in S_{\gamma}$.
    From the second and third events, we guarantee that any $x$ with $\distribution(x) > 10 \nu$ will be in both $S^{(1)}_{\nu}, S^{(2)}_{\nu}$ and any $x$ with $\distribution(x) < \nu/10$ will be in neither.
    From the fourth event, we have that for any $x \not\in S_{\gamma}$ and $\distribution(x) \in [0.1 \nu, 10 \nu]$,
    \begin{align*}
        \Pr \left( x \in S^{(1)}_{\nu} \Delta S^{(2)}_{\nu}\right) &\leq \Pr \left( r(x) \in (\min(\hat{p}^{(1)}(x), \hat{p}^{(2)}(x)), \max(\hat{p}^{(1)}(x), \hat{p}^{(2)}(x))) \right) \\
        &\leq \frac{0.02 \gamma \distribution(x)}{9 C \nu} \leq \frac{0.2 \gamma \nu}{9 C \nu} \leq \frac{\gamma}{C} \text{.}
    \end{align*}
    Furthermore, $x \in S^{(1)}_{\nu} \Delta S^{(2)}_{\nu} \setminus S_{\gamma}$ is independent for all $x$ as $r(x)$ are all sampled independently.
    In particular,
    \begin{equation*}
        \distribution(S^{(1)}_{\nu} \Delta S^{(2)}_{\nu} \setminus S_{\gamma}) = \sum_{x \not\in S_{\gamma}} \distribution(x) \cdot \ind\left[x \in S^{(1)}_{\nu} \Delta S^{(2)}_{\nu} \right]
    \end{equation*}
    is a sum of independent random variables with expectation $\leq \gamma/C$ each in the range $[0, 2 \gamma']$, so that a Chernoff bound yields
    \begin{align*}
        \Pr_{r} \left( \distribution(S^{(1)}_{\nu} \Delta S^{(2)}_{\nu}) \setminus S_{\gamma} > \gamma \right) &= \Pr_{r} \left( \distribution(S^{(1)}_{\nu} \Delta S^{(2)}_{\nu} \setminus S_{\gamma}) > C \mu \right) \\
        &< \exp \left( - \Omega\left(\frac{C \mu}{\gamma'} \right)\right) \ll \rho
    \end{align*}
    where $\mu := \E[\distribution(S^{(1)}_{\nu} \Delta S^{(2)}_{\nu} \setminus S_{\gamma})] \leq \gamma/C$ and $\gamma' \leq \gamma/C \log(1/\rho)$ for sufficiently large constant $C$.
    
    Thus, a union bound allows us to conclude that with probability at least $1 - 2(\beta + \rho) \geq 1 - 4 \rho$, $\distribution(S^{(1)}_{\nu} \Delta S^{(2)}_{\nu} \setminus S_{\gamma}) < \gamma$, as desired.
    By setting $\beta \gets \min(\beta, \rho)$ and increasing the sample complexity by a constant factor, we conclude the proof.
\end{proof}

\paragraph{Approximately Replicable Label Learning}

We now give an approximately replicable algorithm for learning the labels of heavy hitters.
As in \Cref{prop:apx-repl-heavy}, we require that all labels of $\gamma$-heavy elements are learned fully replicably, while lighter elements can be learned pointwise replicably. 
As before, we will sample independent random strings to ensure that the mass of light elements concentrates around its mean.

\begin{proposition}
    \label{prop:apx-repl-hh-label}
    There is an algorithm $\innerAlg$ that given $S_{\gamma} \subset S \subset \domain$ and $\nu, \alpha, \beta, \rho, \gamma > 0$ satisfying (1) $\distribution(x) > \gamma$ for all $x \in S_{\gamma}$, (2) $\distribution(x) < 2 \gamma$ for all $x \not\in S_{\gamma}$, and (3) $\distribution(x) > \nu$ for all $x \in S$ satisfying the following:
    \begin{enumerate}
        \item (Replicability)
        Let $S_{\gamma}^{(1)}, S_{\gamma}^{(2)}$ and $S^{(1)}, S^{(2)}$ denote the the inputs and $\ell^{(1)}, \ell^{(2)}$ denote the labels output over two runs of $\innerAlg$.
        Assume $S_{\gamma} := S_{\gamma}^{(1)} = S_{\gamma}^{(2)}$.
        Then, with probability $1 - \rho$, $\ell^{(1)}(x) = \ell^{(2)}(x)$ for all $x \in S_{\gamma}$ and $\distribution(\set{x \in S^{(1)} \cap S^{(2)} \given \ell^{(1)}(x) \neq \ell^{(2)}(x)}) < \gamma$.
        
        \item (Accuracy) With probability $1 - \beta$, $\ell(x) = \sign(y(x))$ whenever $|y(x)| \geq \alpha$.
    \end{enumerate}
    Furthermore, the algorithm has sample complexity 
    \begin{equation*}
        \bigtO{\frac{\log^2(1/\rho)}{\alpha^2 \gamma^2 \nu} + \frac{\log^3(1/\beta)}{\alpha^2 \rho^2 \gamma^2} + \frac{\log (\log(1/\rho)/(\beta \gamma))}{\alpha^2 \nu}} \text{.}
    \end{equation*}
\end{proposition}

\begin{proof}
    Our algorithm proceeds in two stages.
    We label $S_{\gamma}$ fully replicably and $S \setminus S_{\gamma}$ approximately replicably.

    \paragraph{Labeling $S_{\gamma}$.}
    First, we would like to invoke the fully replicable labeling algorithm of \cite{hopkins2024replicability} on $S_{\gamma}$.

    \begin{lemma}
        \label{lem:r-linf-extimate}
        There is a $\rho$-replicable algorithm, given sample access to $\distribution \sim \Rad(p)$ for $p \in [-1, 1]^{d}$, with probability $1 - \beta$ outputs $v$ such that $v_i = \sign(p_i)$ for all $i$ with $|p_i| \geq \alpha$.
        Furthermore, the algorithm requires $\bigtO{\frac{d\log^3(1/\beta)}{\rho^2 \alpha^2}}$ samples from $\Rad(p)$.
    \end{lemma}

    To do so, we apply a similar algorithm and analysis as \Cref{prop:learn-hh-label}.
    \begin{enumerate}
        \item Let $m_0 := \bigtO{\frac{|S_{\gamma}| \log^3(1/\beta)}{\rho^2 \alpha^2}}$ be the sample complexity of \Cref{lem:r-linf-extimate}.
        \item Draw $m = \frac{C \log(|S_{\gamma}|)}{\gamma}$ samples where $C$ is a sufficiently large constant, denoted $\set{(x', y')}$.
        \item For each $x \in S_{\gamma}$, recall that $T(x)$ denote the set of samples where $x' = x$.
        \item Return $\ell(x)$ computed by \Cref{lem:r-linf-extimate} using $\set{T(x)}_{x \in S_{\gamma}}$ for all $x \in S_{\gamma}$.
    \end{enumerate}
    We argue that this procedure satisfies the following guarantees:
    \begin{enumerate}
        \item With probability $1 - 2 \beta$, $\ell(x) = \sign(y(x))$ for all $x \in S_{\gamma}$ with $|y(x)| \geq \alpha$.
        \item With probability $1 - \beta - \rho$, the labels $\ell^{(1)}(x) = \ell^{(2)}(x)$ for all $x \in S_{\gamma}$ over two runs of the algorithm. 
    \end{enumerate}
    Since each element $x \in S_{\gamma}$ is assumed to satisfy $\distribution(x) \geq \gamma/2$, we have $\E[|T(x)|] \geq m \gamma$ and a Chernoff bound similar to \Cref{prop:learn-hh-label} ensures that as long $m = \Theta(\log(|S_{\gamma}|/\beta) / \gamma)$ for a sufficiently large constant, we have $|T(x)| \geq m \gamma / 4$ for all $x \in S_{\gamma}$ with probability $1 - \beta$.
    In particular, we fix $m = \frac{C m_0}{\gamma}$ for a sufficiently large constant to ensure that $|T(x)| \geq m_0$ for all $x \in S_{\gamma}$.
    By \Cref{lem:r-linf-extimate}, we obtain a replicable algorithm for computing labels for all $S_{\gamma}$.
    By a union bound, we have $\ell(x) = \sign(y(x))$ for all $|y(x)| \geq \alpha$ with probability $1 - 2 \beta$ and the labels are replicable with probability $1 - \beta - \rho$.

    \paragraph{Labeling $S \setminus S_{\gamma}$.}
    Next, to label $x \in S \setminus S_{\gamma}$, we will run \Cref{alg:learn-hh-label} with parameters $\nu \gets \nu$, $\alpha \gets \alpha$, $\beta \gets \beta$, $\rho \gets \gamma/C$ for a sufficiently large $C$ and output the label $\ell(x)$ obtained for all $x \in S \setminus S_{\gamma}$.
    When executing \Cref{alg:learn-hh-label}, we ensure that we use independent shared randomness for every $x$ when running \Cref{lem:replicable-coin-testing}. 
    
    We claim that our algorithm satisfies the required guarantees.
    We begin with correctness. 
    A union bound implies that with probability $1 - 3 \beta$, $\ell(x) = \sign(y(x))$.
    We now proceed with replicability.
    Conditioned on the event that $\ell^{(1)}(x) = \ell^{(2)}(x)$ for all $x \in S_{\gamma}$, we have 
    \begin{equation*}
        \distribution\left(\set{x \in S^{(1)} \cap S^{(2)} \given \ell^{(1)}(x) \neq \ell^{(2)}(x)} \right) = \distribution\left(\set{x \in S^{(1)} \cap S^{(2)} \setminus S_{\gamma} \given \ell^{(1)}(x) \neq \ell^{(2)}(x)} \right) \text{.}
    \end{equation*}
    Fix $x \in S^{(1)} \cap S^{(2)} \setminus S_{\gamma}$.
    By \Cref{prop:learn-hh-label}, we have that $\Pr(\ell^{(1)}(x) \neq \ell^{(2)}(x)) \leq \gamma/C$.
    Then,
    \begin{equation*}
        Z := \distribution\left(\set{x \in S^{(1)} \cap S^{(2)} \setminus S_{\gamma} \given \ell^{(1)}(x) \neq \ell^{(2)}(x)} \right) = \sum_{x \in S^{(1)} \cap S^{(2)} \setminus S_{\gamma}} \distribution(x) \cdot \ind\left[ \ell^{(1)}(x) \neq \ell^{(2)}(x) \right]
    \end{equation*}
    is a sum of random variables in the range $[0, 2 \gamma]$ by the second input assumption and expectation at most $\rho$.
    By a Chernoff bound (\Cref{thm:chernoff}), we have 
    \begin{equation}
        \label{eq:chernoff-application-apx-repl}
        \Pr \left( Z > 6 \log(1/\rho) \gamma \right) < \exp \left( - \frac{6 C \log(1/\rho) (\gamma/C)}{3 (2 \gamma)} \right) < \rho \text{.}
    \end{equation}
    Then, by a union bound, we have with probability at least $1 - \beta - 2 \rho \geq 1 - 3 \rho$ that
    \begin{equation*}
        \distribution\left(\set{x \in S^{(1)} \cap S^{(2)} \given \ell^{(1)}(x) \neq \ell^{(2)}(x)} \right) \leq 6 \log(1/\rho) \gamma \text{.}
    \end{equation*}
    Setting $\beta \leq \rho$ and $\gamma \leq \gamma/\log(1/\rho)$ and increasing our sample complexity by a constant factor, we conclude the proof of the algorithm.
    
    Finally, we bound the sample complexity.
    Beginning with labeling $S_{\gamma}$, we use
    \begin{equation*}
        \bigtO{\frac{|S_{\gamma}| \log^3(1/\beta)}{\rho^2 \alpha^2 \gamma}} = \bigtO{\frac{\log^3(1/\beta)}{\rho^2 \alpha^2 \gamma^2}}
    \end{equation*}
    where we observe $|S_{\gamma}| = O(1/\gamma)$.
    Next, to label $S \setminus S_{\gamma}$ we use \Cref{prop:learn-hh-label} to bound the sample complexity as
    \begin{equation*}
        \bigO{\frac{1}{\alpha^2 \rho^2 \nu} + \frac{\log (|S|/(\beta \rho))}{\alpha^2 \nu}} = \bigtO{\frac{\log^2(1/\rho)}{\alpha^2 \gamma^2 \nu} + \frac{\log (\log(1/\rho)/(\beta \gamma))}{\alpha^2 \nu}}
    \end{equation*}
    where we have used $|S| \leq O(1/\nu)$.
    Summing the above two bounds proves the desired sample complexity.
\end{proof}

\paragraph{Sample-Efficient Approximately Replicable PAC Learning}

We are now ready to present our approximately replicable (agnostic) PAC learner with improved sample complexity. 

\begin{proof}[Proof of \Cref{thm:apx-repl-chernoff}]
    It suffices to prove that our algorithm is an approximately replicable PAC learner over finitely supported distributions from the following lemma, whose proof we defer to \Cref{app:domain-red}.

    \begin{proposition}
        \label{prop:domain-reduction-apx-repl}
        Let $\domain_0$ be an arbitrary domain.
        Let $\Acal$ be a $(\rho, \gamma)$-approximately replicable $(\alpha, \beta)$-learner on every finitely supported distribution $\distribution$ over $\domain_0$ with sample complexity $m(\alpha, \beta, \rho, \gamma)$ independent of the distribution.
        
        Then, there is an algorithm that is a $(3 \rho, 2 \gamma)$-approximately replicable $(4 \alpha, 4 \beta)$-learner on arbitrary distributions over $\domain_{0}$ with $m(\alpha, \beta, \rho, \gamma)$ samples.
    \end{proposition}

    Our algorithm will be similar to \Cref{alg:chernoff-pred}, except that we make the following modifications.
    \begin{enumerate}
        \item Set $T \gets \frac{C}{\gamma^2}$ (instead of $T \gets \frac{C}{\rho^2}$).
        \item Compute $S, S_{\gamma}, \gamma'$ with \Cref{prop:apx-repl-heavy} (instead of \Cref{prop:finding-heavy}) with parameters $\nu \gets 0.1 \nu$, $\beta \gets \beta, \rho \gets \rho, \gamma \gets \gamma$.
        \item Compute $\set{\ell(x)}_{x \in S}$ with \Cref{prop:apx-repl-hh-label} (instead of \Cref{prop:learn-hh-label}) with parameters $\nu \gets 0.01 \nu$, $\beta \gets \beta$, $\rho \gets \rho$, $\gamma \gets \gamma'$.
    \end{enumerate}
    Following \Cref{thm:predict-formal}, we will argue that the resulting algorithm is a $(3 \rho, 3\gamma)$-approximately replicable $(3 \alpha, 4 \beta)$-learner in the agnostic setting.
    Increasing the sample complexity by a constant factor concludes the proof of the theorem. 

    We begin with the more interesting property: approximate replicability.
    Let $\tilde{h}^{(1)}, \tilde{h}^{(2)}$ denote hypotheses output by our algorithm over two independent runs.
    Similarly, denote $S_{\gamma}^{(i)}, S^{(i)}, \ell^{(i)}$ the outputs of \Cref{prop:apx-repl-heavy} and \Cref{prop:apx-repl-hh-label} over two runs of the algorithm.
    We condition on the following events:
    \begin{enumerate}
        \item $S_{\gamma}^{(1)} = S_{\gamma}^{(2)}$ and denote $S_{\gamma} := S_{\gamma}^{(1)} = S_{\gamma}^{(2)}$ and $\ell^{(1)}(x) = \ell^{(2)}(x)$ for all $x \in S_{\gamma}$.
        \item $\distribution\left(S^{(1)} \Delta S^{(2)}\right) < \gamma$.
        \item $\distribution\left(\set{x \in S^{(1)} \cap S^{(2)} \given \ell^{(1)}(x) \neq \ell^{(2)}(x)}\right) < \gamma$.
        \item $\distribution\left(\set{x \not\in S^{(1)} \cup S^{(2)} \given \tilde{h}^{(1)}(x) \neq \tilde{h}^{(2)}(x)}\right) < \gamma$.
    \end{enumerate}
    First, let us see that our algorithm produces approximately replicable hypotheses under the desired events.
    In particular,
    \begin{align*}
        \distD\left(\tilde{h}^{(1)}, \tilde{h}^{(2)}\right) &= \distribution\left( \set{x \given \tilde{h}^{(1)}(x) \neq \tilde{h}^{(2)}(x)} \right) \\
        &= \distribution\left( \set{x \given \tilde{h}^{(1)}(x) \neq \tilde{h}^{(2)}(x), x \notin S_{\gamma}} \right) \\
        &= \distribution\left( \set{x \given \tilde{h}^{(1)}(x) \neq \tilde{h}^{(2)}(x), x \in S^{(1)} \cap S^{(2)}} \right) \\
        &\quad+ \distribution\left( \set{x \given \tilde{h}^{(1)}(x) \neq \tilde{h}^{(2)}(x), x \in S^{(1)} \Delta S^{(2)}} \right)\\
        &\quad+ \distribution\left( \set{x \given \tilde{h}^{(1)}(x) \neq \tilde{h}^{(2)}(x), x \not\in S^{(1)} \cup S^{(2)}} \right) \\
        &< 3 \gamma \text{.}
    \end{align*}
    The first equality is the definition of $\distD$.
    The second equality applies the first event.
    The third equality we observe that $x \not\in S_{\gamma}$ is either in both $S^{(1)}, S^{(2)}$, just one, or neither.
    To bound the first term, we apply the third event and observe $\tilde{h}(x) = \ell(x)$; to bound the second term, we apply the second event; to bound the third term, we apply the fourth event.

    We now bound the probability that all four events hold.
    The first event follows from the replicability of \Cref{prop:apx-repl-heavy} and \Cref{prop:apx-repl-hh-label} (note that if $S_{\gamma}$ is replicable, this satisfies the input assumption of \Cref{prop:apx-repl-hh-label}). 
    The second event follows from the replicability of \Cref{prop:apx-repl-heavy}. 
    The third event follows from the replicability of \Cref{prop:apx-repl-hh-label}. 
    We now analyze the fourth event.
    Fix $x \not\in S^{(1)} \cup S^{(2)}$ so that $\tilde{h}^{(i)}(x) \gets \ind[\hat{p}^{(i)}(x) \geq r(x)]$ in \Cref{alg:chernoff-pred}.
    Since the distribution has finite support, we define $Z$ to be the mass of elements not in $S^{(1)} \cup S^{(2)}$ where $\tilde{h}^{(1)}, \tilde{h}^{(2)}$ disagree:
    \begin{align*}
        Z &:= \distribution\left( \set{x \given \tilde{h}^{(1)}(x) \neq \tilde{h}^{(2)}(x), x \not\in S^{(1)} \cup S^{(2)}} \right) \\
        &= \sum_{x \not\in S^{(1)} \cup S^{(2)}} \distribution(x) \cdot \ind\left[ \tilde{h}^{(1)}(x) \neq \tilde{h}^{(2)}(x) \right] \\
        &= \sum_{x \not\in S^{(1)} \cup S^{(2)}} \distribution(x) \cdot \ind\left[ r(x) \in (\min(\hat{p}^{(1)}(x), \hat{p}^{(2)}(x)), \max(\hat{p}^{(1)}(x), \hat{p}^{(2)}(x))) \right] \text{.}
    \end{align*}
    By the completeness guarantee of \Cref{prop:apx-repl-heavy}, we have that $\distribution(x) \leq 2 \gamma'$.
    By \Cref{lem:p-estimate}, $\E[Z] \leq \frac{\gamma}{C}$ (due to our setting of $T$).
    Then, applying a Chernoff bound (\Cref{thm:chernoff}) as in \eqref{eq:chernoff-application-apx-repl} we obtain the desired bound.
    Union bounding over all events, we have that all four events hold with probability $1 - 3 \rho$.

    We now argue for correctness.
    Recall that \Cref{alg:chernoff-pred} is a $(3 \alpha + \beta, 3 \beta)$-learner in the agnostic setting.
    Furthermore, the proof of accuracy relies only on the completeness and soundness property of finding heavy hitters (\Cref{prop:finding-heavy}), the accuracy property of labeling heavy hitters (\Cref{prop:learn-hh-label}), \Cref{lem:exp-error-ub-pred}, and \Cref{thm:countable-tail-bound}.
    Note that we can apply \Cref{thm:countable-tail-bound} as in \Cref{thm:predict-formal} by the assumption that the data distribution is finitely supported.
    All of these properties hold identically in the approximately replicable setting, with the exception of the completeness and soundness property of \Cref{prop:finding-heavy}. 
    Thus, in this section we will only describe how to modify the argument to handle this change.

    \paragraph{Accuracy on Heavy Hitters.}
    Recall that we define $S^* = \set{x \given \distribution(x) \geq \nu}$.
    As before, by the completeness property of \Cref{prop:apx-repl-heavy}, we have $S^* \subseteq S$.
    The soundness property of \Cref{prop:apx-repl-heavy} again ensures that $\distribution(x) \geq 0.01 \nu$ for all $x \in S$.
    The remaining proof follows identically.

    \paragraph{Accuracy on Non-Heavy Hitters.}
    As before, the completeness property of \Cref{prop:apx-repl-heavy} ensures that every $x \not\in S$ satisfies $\distribution(x) \leq \nu$.
    The remaining proof follows identically.

    We conclude by bounding the sample complexity.
    Combining \Cref{prop:apx-repl-heavy}, \Cref{prop:apx-repl-hh-label}, and our setting of $T$, we obtain
    \begin{align*}
        \bigtO{\frac{m(\alpha,\gamma^2 \beta)}{\gamma^2} + \frac{\log(1/\beta)}{\rho^2 \gamma^3} + \frac{\log^3(1/(\rho \beta))}{\alpha^{4} \gamma^2} + \frac{\log^3(1/(\rho \beta))}{\alpha^{2} \rho^{2}  \gamma^2} + \frac{\log^{2}(1/(\rho \beta \gamma)}{\alpha^{4}}} \text{.}
    \end{align*}
    Simplifying, we bound the sample complexity as
    \begin{equation*}
        m(\alpha, \beta, \rho, \gamma) = \bigtO{\frac{m(\alpha, \gamma^2 \beta)}{\gamma^2} + \frac{\log(1/\beta)}{\rho^2 \gamma^3} + \frac{\log^3(1/(\rho \beta \gamma))}{\alpha^{4} \gamma^2} + \frac{\log^{3}(1/(\rho \beta))}{\alpha^{2} \rho^{2}  \gamma^2}} \text{.}
    \end{equation*}
    This concludes the proof of \Cref{thm:apx-repl-chernoff}.
\end{proof}

\subsection{Lower Bounds for Approximate Replicability}\label{sec:apx-lower}

In this section, we argue approximate replicable learning of VC classes can be reduced to the bias estimation problem, a useful distribution estimation task used in a number of prior replicability lower bounds. 
As in replicable prediction, our goal is to ``plant" one instance of bias estimation at a single randomized point. However, we now need to ensure the instance remains indistinguishable for \emph{every} input distribution to the bias estimation problem.

\ApproxReplicabilityLB*

As discussed in the technical overview, we will take advantage of the fact that replicable bias estimation is still hard in the average-case setting. 
In particular, even when the adversary selects $\Rad(p)$ with $p \in [- \alpha, \alpha]$ uniformly at random (and the algorithm knows this distribution), the sample complexity lower bound $\bigOm{\frac{1}{\rho^2 \alpha^2}}$ still holds.
We will use this to obtain a lower bound on approximately replicable PAC learning.

Before proceeding to the proof, we define some relevant terms.
Let $\set{x_{1}, \dotsc, x_{d}}$ denote a set of $d$ shattered elements.
Given any meta-distribution $\calM$ over Rademacher distributions (i.e. a distribution over distributions over $\set{\pm 1}$), define $\calM^{\otimes d}$ to be the meta-distribution that samples distributions as follows:

\begin{enumerate}
    \item Sample $d$ distributions $\distribution_{1}, \dotsc, \distribution_{d} \sim \calM$ independently.
    \item Define $\distribution^{\otimes d}$ to be the following distribution:
    \begin{enumerate}
        \item Sample $i \in [d]$ uniformly.
        \item Sample $y \sim \distribution_{i}$ and return $(x_{i}, y)$.
    \end{enumerate}
\end{enumerate}

For the rest of the proof, the reader can assume $\calM$ is the meta-distribution given by sampling $p \in [-\alpha, \alpha]$ uniformly and providing samples from $\Rad(p)$, although our reduction will hold for any meta-distribution $\calM$.
We will prove lower bounds for algorithms that are correct and replicable against a meta-distribution. Let's first be careful in defining what this means.

\begin{definition}[Distributional Accuracy]
    \label{def:dist-acc}
    Let $\calM$ consist of a distribution over input labeled distributions.
    An algorithm $\innerAlg$ is $(\alpha, \beta)$-accurate with respect to $\calM$ if
    \begin{equation*}
        \Pr_{\distribution \sim \calM, S \sim \distribution^{m}} \left( \errD(\innerAlg(S)) > \alpha \right) < \beta \text{.}
    \end{equation*}
\end{definition}

While the above notion is defined for PAC learners, they naturally extend to distributions over $\set{\pm 1}$ as follows.

\begin{definition}[Distributional Accuracy]
    \label{def:dist-acc-bias}
    Let $\calM$ consist of a distribution over distributions over $\set{\pm 1}$.
    An algorithm $\innerAlg: \set{\pm 1}^{m} \rightarrow \set{\pm 1}$ is $(\alpha, \beta)$-accurate with respect to $\calM$ if
    \begin{equation*}
        \Pr_{\distribution \sim \calM, S \sim \distribution^{m}} \left( (v - \sign(p)) p > 2 \alpha \right) < \beta \text{.}
    \end{equation*}
    where $\distribution \sim \Rad(p)$ and $v = \innerAlg(S)$ is the output of $\innerAlg$.
\end{definition}

Next, we define distributional replicability, and for PAC learners, distributional approximate replicability.

\begin{definition}[Distributional (Approximate) Replicability]
    \label{def:dist-repl}
    Let $\calM$ consist of a distribution over input labeled distributions.
    An algorithm $\innerAlg$ is $\rho$-replicable with respect to $\calM$ if
    \begin{equation*}
        \Pr_{\distribution \sim \calM, S_1, S_2 \sim \distribution^{m}} \left( \innerAlg(S_1) \neq \innerAlg(S_2) \right) < \rho \text{.}
    \end{equation*}
    An algorithm is $(\rho, \gamma)$-approximately replicable with respect to $\calM$ if
     \begin{equation*}
        \Pr_{\distribution \sim \calM, S_1, S_2 \sim \distribution^{m}} \left( \distD(\innerAlg(S_1), \innerAlg(S_2)) > \gamma \right) < \rho \text{.}
    \end{equation*}
\end{definition}


Our goal is now to relate distributional replicability and accuracy of algorithms to replicability and accuracy via the MiniMax theorem.
Consider the following games, defined $\game_{m, \gamma}, \game_{m, \alpha}, \game_{m, \alpha, \gamma}$.
\begin{enumerate}
    \item The algorithm player (randomly) chooses an algorithm: $\innerAlg: (\domain \times \set{\pm 1})^{m} \rightarrow \set{\domain \rightarrow \set{\pm 1}}$.
    \item The adversary player (randomly) chooses a distribution $\distribution$ over $\domain \times \set{\pm 1}$.
    \item Dataset $T, S_1, S_2$ of size $m_{u}, m_{s}, m_{s}$ are drawn \iid from $\distribution_{\domain}$.
    \item The algorithm player wins $\game_{m, \gamma}$ if $\distD(\innerAlg(T, S_1), \innerAlg(T, S_2)) \leq \gamma$, $\game_{m, \alpha}$ if $\errD(\innerAlg(T, S_1)) \leq \alpha$, and $\game_{m, \alpha, \gamma}$ if both occur.
    Otherwise, in each game, the adversary player wins.
    The value of the game, denoted $\game_{m, \alpha, \gamma}(\innerAlg, \distribution)$, is $1$ when the algorithm player wins and $0$ when the adversary player wins.
\end{enumerate}
When the sample complexity of the algorithm $m$ and parameters $\alpha, \gamma$ are clear, we simply write $\game$.
Note that any randomized algorithm can be viewed as a randomized strategy for the algorithm player while a meta-distribution can be viewed as a randomized strategy for the adversary player.
We are now ready to prove \Cref{thm:approx-repl-d-lb}.

\begin{proof}[Proof of \Cref{thm:approx-repl-d-lb}]
    Our overall proof will follow two steps: (1) from an approximately replicable learner for $\hypotheses$, we obtain a distributionally approximately replicable learner for $\hypotheses$, and (2) from the distributionally approximately replicable learner we obtain a distributionally replicable algorithm for bias estimation, which requires $\bigOm{\frac{1}{\rho^2 \alpha^2}}$ samples.
    We begin with step (1).

    Suppose we have an $(\rho, \gamma)$-approximately replicable $(\alpha, \beta)$-learner for $\hypotheses$ with sample complexity $m$.
    Then, from the easy direction of the MiniMax Theorem, we have
    \begin{equation*}
        1 - (\rho + \beta) \leq \max_{\innerAlg} \min_{\distribution} \E_{\innerAlg}[\game(\innerAlg, \distribution)] \leq \min_{\calM} \max_{A_{0}} \E_{\distribution \sim \calM}[\game(\innerAlg_{0}, \distribution)] \leq \min_{\calM^{\otimes d}} \max_{\innerAlg_{0}} \E_{\distribution \sim \calM}[\game(\innerAlg_{0}, \distribution)] \text{.}
    \end{equation*}
    Note that the final inequality follows since restricting the choice of the adversary to meta-distributions of the form $\calM^{\otimes d}$ can only help the algorithm player.
    Thus, we claim that for any meta-distribution $\calM^{\otimes d}$, we have a deterministic algorithm $\innerAlg_{0}$ that is $(\alpha, \beta)$-accurate and $(\rho, \gamma)$-approximately replicable with respect to $\calM^{\otimes d}$.
    This follows as 
    \begin{align*}
        1 - (\rho + \beta) &\leq \min_{\calM^{\otimes d}} \max_{\innerAlg_{0}} \E_{\distribution \sim \calM^{\otimes d}} \left[ \game(\innerAlg_{0}, \distribution) \right]\\
        &= \min_{\calM^{\otimes d}} \max_{\innerAlg_{0}} \\
        &\quad \Pr_{\distribution \sim \calM^{\otimes d}, S_1, S_2 \sim \distribution^{m}} \left( \left( \max_{i} \errD(\innerAlg_{0}(T, S_{i})) < \alpha \right) \wedge \left( \distD(\innerAlg_{0}(T, S_1), \innerAlg_{0}(T, S_2)) < \gamma \right) \right) \text{.}
    \end{align*}
    In the above, we take the maximum over strategies for the algorithm player, denoted $\innerAlg_{0}$, which are is the set of all deterministic algorithms.
    This completes step (1).
    We now proceed with step (2) and argue that from $\innerAlg_0$
    we can obtain an algorithm that is accurate and replicable with respect to $\calM$.
    
    \begin{lemma}
        \label{lem:dim-reduction-semi-replicable}
        Let $\calM$ be any meta-distribution over Rademacher distributions.
        Suppose $\innerAlg_{0}$ is a $(\rho, \gamma)$-approximately replicable $(\alpha, \beta)$-accurate learner with respect to $\calM^{\otimes d}$ with sample complexity $m$.
        Then, there is a $(2\rho + \gamma)$-replicable $(100 \alpha, 0.1)$-accurate algorithm $\innerAlg_{1}$ with respect to $\calM$ with sample complexity $\bigO{\frac{m}{d}\log(1/\rho)}$.
    \end{lemma}

    In the above lemma, we assume that $\calM$ is explicitly known. We reiterate that for replicable bias estimation, the hard meta-distribution $\calM$ is explicitly known, i.e. sample from $\Rad(p)$ where $p \in [-\alpha, \alpha]$ is chosen uniformly.

    
    \begin{proof}[Proof of \Cref{lem:dim-reduction-semi-replicable}]
        We describe our algorithm in \Cref{alg:apx-repl-dim-reduction}.
    
        \IncMargin{1em}
        \begin{algorithm}
        
        \SetKwInOut{Input}{Input}\SetKwInOut{Output}{Output}\SetKwInOut{Parameters}{Parameters}
        \Input{Approximately replicable PAC learner $\innerAlg_{0}$. Sample access to $\calM$ and $\distribution$.}
        \Parameters{$\rho$ replicability, $\alpha$ accuracy, $\beta$ error probability.}
        \Output{$\rho$-replicable $(\alpha, \beta)$-accurate algorithm with respect to $\calM$.}
    
        Sample $r \in [d]$ uniformly.

        \For{$i \in [d] \setminus r$}{
            Sample $\distribution_{i} \sim \calM$ independently.
        }

        Initiate $S$ as an empty multi-set and counter $c \gets 0$.

        \For{$j \in [m]$}{
            Sample $x_{j} \sim \domain$ uniformly.

            \If{$x_{j} = x_{r}$}{
                Sample $y_{j} \sim \distribution$ and add $(x_{j}, y_{j})$ to $S$. 

                Increment $c \gets c + 1$.

                \If{$c \geq 10\frac{m}{d}\log(1/\rho)$}{
                    \Return $h: x \mapsto 1$.
                }
            }\Else{
                Sample $y_{j} \sim \distribution_{i}$ where $x_{j} = x_{i}$ and add $(x_{j}, y_{j})$ to $S$.
            }
        }

        \Return $h(x_{r})$ where $h \gets \innerAlg_{0}(S)$.
        
        \caption{$\apxReplDimRed(\innerAlg_{0}, \calM)$}
        \label{alg:apx-repl-dim-reduction}
        
        \end{algorithm}
        \DecMargin{1em}

        The sample complexity of our algorithm is immediate.

        By construction, since $\distribution \sim \calM$, the dataset $S$ consists of $m$ samples drawn from a distribution sampled from the meta-distribution $\calM^{\otimes d}$.
        Since $\innerAlg_{0}$ is distributionally approximately replicable, with probability at least $1 - \rho$, $\innerAlg_{0}$ outputs $\gamma$-close hypotheses (i.e. hypotheses that disagree on at most a $\gamma$-fraction of $\set{x_{1}, \dotsc, x_{d}}$).
        Since $\innerAlg_{0}$ is distributionally accurate, with probability at least $1 - \beta$, $\innerAlg_{0}$ outputs $\alpha$-accurate hypotheses.
        We condition on these events.
        In addition, we condition on the event that the sample limit is not reached.
        Let $m_{r}$ denote the number of samples $x_{r}$ drawn.
        Since the expectation $\E[m_{r}] = \frac{m}{d}$, a standard Chernoff bound yields that $m_{r} > 10\frac{m}{d} \log(1/\rho)$ with probability at most $\rho$.

        Conditioned on the above, we argue our algorithm is accurate and replicable with respect to $\calM$.
        We begin with accuracy.
        Let us analyze the error of the output hypothesis $h \gets \innerAlg(S)$.
        Let $\distribution'$ be the distribution over labeled samples $(x, y)$ where the marginal is given by the uniform distribution on $\set{x_{1}, \dotsc, x_{d}}$, each domain element $x_{i}$ is labeled according to $\distribution_{i}$, and $x_{r}$ is labeled according to $\distribution$, so that $S$ consists of $m$ samples from $\distribution'$.
        For all $i$, let $p_{i} = \E_{\distribution'}[y|x = x_{i}] = \E[\distribution_{i}]$ be the expected label of $x_{i}$ under $\distribution'$, i.e. the mean of $\distribution_{i}$.
        Let $h^*$ be the optimal hypothesis for $\distribution'$.
        Then, if $h$ is $\alpha$-accurate,
        \begin{align*}
            \alpha &\geq \err_{\distribution'}(h) - \err_{\distribution'}(h^*) \\
            &= \frac{1}{d} \sum_{i = 1}^{d} \ind[h(x_{i}) \neq h^*(x_{i})] |p_{i}| \text{.}
        \end{align*}
        Define $\err_{i} := \ind[h(x_{i}) \neq h^*(x_{i})] |p_{i}|$. 
        Then, for any fixed $\distribution'$, we note that $r \sim [d]$ is still uniformly distributed, as $\distribution_{i}$ for $i \neq r$ and $\distribution$ are all sampled \iid from $\calM$.
        Since $\innerAlg_{0}$ only observes samples $S \sim \distribution'$, we note that $\err_{i}$ are identically distributed. 
        In particular, conditioned on $\frac{1}{d} \sum_{i = 1}^{d} \err_{i} \leq \alpha$, we have
        \begin{align*}
            \E_{\distribution \sim \calM, \distribution' \sim \calM^{\otimes d}, S \sim \distribution'} \left[ \err_{r} \right] &\leq \alpha \text{.}
        \end{align*}
        By Markov's Inequality,
        \begin{align*}
            \Pr_{\distribution \sim \calM, \distribution' \sim \calM^{\otimes d}, r \sim [d]} \left( \err_{r} > 100 \alpha \right) < 0.01 \text{.}
        \end{align*}
        Union bounding with the above events, we obtain a $(100 \alpha, 0.1)$-accurate algorithm with respect to $\calM$.

        Finally, we argue replicability.
        Consider two outputs of the algorithm $h_{1}, h_{2}$.
        Conditioned on the success of $\innerAlg$, $h_{1}, h_{2}$ disagree in at most $\gamma d$ domain elements $\set{x_{i}}$.
        Let $E = \set{i \given h_{1}(x_{i}) \neq h_{2}(x_{i})}$ be the set of domain elements where $h_{1}, h_{2}$ disagree.
        Again, since for any fixed $\distribution'$, $r$ is uniformly distributed in $[d]$, the probability that $r \in E$ is at most $\gamma$.
        Thus, we conclude that our algorithm is $(2 \rho + \gamma)$-replicable with respect to $\calM$ by a union bound (since $\innerAlg$ outputs $\gamma$-close hypotheses with probability $\rho$ and we do not exceed the sample bound with probability $\rho$).
    \end{proof}

    It remains to prove a lower bound for algorithms that are distributionally replicable and accurate with respect to some meta-distribution $\calM$.
    Towards this, we fix $\calM$ to be the specific hard meta-distribution mentioned above.
    For any fixed accuracy threshold $\alpha$ and replicability parameter $\rho$, define $\calM_{\alpha}$ to be the meta-distribution that uniformly samples a mean $p$ from $[-\alpha, \alpha]$ and produces samples from $\Rad(p)$.
    Then, $\calM_{\alpha}^{\otimes d}$ is the meta-distribution defined above by picking label distributions from $\calM_{\alpha}$ independently for each domain element.
    From \cite{ImpLPS22, hopkins2024replicability}, it is known that any algorithm that is $(\alpha, 0.1)$-accurate and $\rho$-replicable with respect to $\calM_{\alpha}$ requires $\bigOm{\frac{1}{\alpha^2 \rho^2}}$ samples.
    In particular, we have
    \begin{equation*}
        \frac{m}{d} \log(1/\rho) = \bigOm{\frac{1}{\alpha^2 (\rho + \gamma)^2}}
    \end{equation*}
    which yields the desired lower bound.
\end{proof}

\subsection{Shared Randomness}

We have seen that relaxing the definition of replicability to approximate replicability removes the necessity of shared randomness for basic tasks such as mean estimation, but our more involved methods for PAC learning rely heavily on a shared random string. Moreover, our sample lower bound from \Cref{thm:approx-repl-d-lb} does not show this is necessary, as the reduction itself uses shared randomness.
In this section, we give a direct argument showing shared randomness is provably necessary for approximately replicable agnostic PAC learning.

To obtain a lower bound against deterministic algorithms, we instead give a \emph{deterministic reduction} from sign-one-way marginals. 


\begin{definition}[Sign-One-Way Marginals]
    \label{def:sign-one-way-marginals}
    Let $\distribution_{p}$ be a product of $d$ Rademacher distributions with expectations $p = (p_1, \dotsc, p_{d})$.
    A vector $v \in \set{\pm 1}^{d}$ is an $\alpha$-accurate solution to the sign-one-way marginals problem for $\distribution_{p}$ if $\frac{1}{d} \sum_{i}^{d} v_i p_i \geq \frac{1}{d} \sum_{i = 1}^{d} |p_i| - \alpha$.

    An algorithm $(\alpha, \beta)$-accurately solves the sign-one-way marginals problem with sample complexity $m$ if given any product of $d$ Rademacher distributions $\distribution$ and $m$ \iid samples from $\distribution$, with probability at least $1 - \beta$ the algorithm outputs an $\alpha$-accurate solution to the sign-one-way marginals problem for $\distribution$.
\end{definition}

In order to make this reduction useful, we give the first impossibility result for deterministic approximately replicable algorithms.

\begin{definition}
    \label{def:apx-repl-sign-one-way}
    An algorithm $\innerAlg$ with $m$ samples solving the one-way marginals problem (\Cref{def:sign-one-way-marginals}) is $(\rho, \gamma)$-approximately replicable if for every distribution $\distribution$, 
    \begin{equation*}
        \Pr_{S_1, S_2 \sim \distribution^{m}, r}\left( \norm{\innerAlg(S_1; r) - \innerAlg(S_2; r)}_{0} > \gamma d \right) < \rho \text{.}
    \end{equation*}
\end{definition}

First, we show that no deterministic approximately replicable algorithm can solve the sign-one-way marginals problem.
In fact, we argue this even for the case $d = 2$.

\begin{theorem}
    \label{thm:shared-randomness-lb}
    There is no deterministic algorithm that $(0.01, 0.5)$-approximately replicable and $(0.1, 0.01)$-accurately solves the sign-one-way marginals problem.
\end{theorem}

\begin{proof}
    Suppose such an algorithm $\innerAlg$ exists.
    Note that $\innerAlg$ has range $\range = \set{(-1, -1), (-1, 1), (1, -1), (1, 1)}$.
    Consider a distribution $\distribution$ with mean $p = (p_1, p_2)$ which consists of a product of two Rademacher distributions.
    For all $p$, define $S_{p} := \set{y \in \range \given \Pr_{S}(\innerAlg(S) = y) \geq 0.1}$ to be the set of canonical outcomes of $\innerAlg$, i.e., $\innerAlg$ outputs $y$ with probability at least $0.1$.
    Since $|\range| \leq 4$, $S_{p}$ is non-empty for all $\distribution$.
    Now, suppose $S_{p}$ has two distinct elements $x, y$. 
    We claim $\norm{x - y}_{0} \leq 1$.
    Otherwise,
    \begin{equation*}
        \Pr_{T_1, T_2} \left( \norm{\innerAlg(T_1) - \innerAlg(T_2)}_{0} > 1\right) \geq 2\Pr(\innerAlg(T) = x) \Pr(\innerAlg(T) = y) \geq 0.02 \text{,}
    \end{equation*}
    violating the approximate replicability constraint.
    Then, we have $|S_{p}| \leq 2$ for all $p$ and the two elements can disagree in at most one coordinate.
    We claim that there is some $p^*$ for which $S_{p^*}$ contains $x, y$ where $\norm{x - y}_{0} = 2$, which violates approximate replicability.

    Towards this, for each $y \in \range$, let $P_{y}: [-1, 1]^{2} \rightarrow [0, 1]$ denote the function $\Pr_{T \sim \distribution_{p}}(\innerAlg(T) = y)$.
    Note that 
    \begin{equation*}
        P_{y}(p) = \sum_{T \given \innerAlg(T) = y} \Pr_{T \sim \distribution_{p}} (T)
    \end{equation*}
    is continuous since $\Pr(T)$ is a product of $p_{1}, (1 - p_{1}), p_{2}, (1 - p_{2})$ and therefore continuous in $p$.
    In particular, if we denote $B_{y} := \set{p \given P_{y}(p) \geq 0.1}$,
    we have that $B_{y}$ is closed for all $y$. 
    Define the sets,
    \begin{align*}
        B_{\leftT} &= B_{(-1, -1)} \cup B_{(-1, 1)} \\
        B_{\rightT} &= B_{(1, -1)} \cup B_{(1, 1)} \\
        B_{\up} &= B_{(-1, 1)} \cup B_{(1, 1)} \\
        B_{\down} &= B_{(-1, -1)} \cup B_{(1, -1)} \text{.}
    \end{align*}
    Then, define the continuous map 
    \begin{align*}
        F: p \mapsto (\dist(p, B_{\leftT}) - \dist(p, B_{\rightT}), \dist(p, B_{\up}) - \dist(p, B_{\down})) \text{.}
    \end{align*}
    We would like to apply the Poincare-Miranda Theorem.
    
    \begin{theorem}
        \label{thm:poincare-miranda}
        Consider $n$ continuous, real-value functions $f_{1}, \dotsc, f_{n}: [-1, 1]^{n} \rightarrow \R$.
        Assume for each $x_{i}$, $f_{i}$ is non-positive when $x_{i} = -1$ and non-negative when $x_{i} = +1$.
        Then, there is a point $p \in [-1, 1]^{n}$ when all $f_{i}$ are equal to $0$.
    \end{theorem}

    Note that when $p_{1} = -1$, then since $\innerAlg$ is $(\alpha, \beta)$-accurate, with probability $1 - \beta$, we have
    \begin{equation*}
        \alpha \geq \frac{1}{2} (1 + v_{1}) \text{.}
    \end{equation*} 
    Rearranging, $v_{1} \leq 2 \alpha - 1 < 0$ i.e. $v_{1} = -1$ so that $\innerAlg(T) \in \set{(-1, -1), (-1, 1)}$ with probability at least $1 - \beta \geq 0.99$ so that $F(p)_{1} < 0$.
    Similarly, when $p_{1} > 0$, we have $\innerAlg(T) \in \set{(1, -1), (1, 1)}$ with probability at least $1 - \beta \geq 0.99$ so that $F(p)_{1} > 0$.
    A similar argument holds for the second coordinate, so we may apply the Poincare-Miranda theorem and obtain a point $p^*$ where $F(p^*) = 0$.

    Now, since $S_{p^*}$ is not empty, assume without loss of generality that $(-1, -1) \in S_{p^*}$. 
    Then, since $F(p^*) = 0$ we have that $\dist(p^*, B_{\rightT}) = \dist(p^*, B_{\up}) = 0$.
    Assume without loss of generality that $p^* \not\in B_{(1, 1)}$, otherwise we are done.
    Then, since $B_{(1, 1)}$ is closed, we have $\dist(p^*, B_{(1, -1)}) = \dist(p^*, B_{(1, -1)}) = 0$ and therefore $p^* \in B_{(1, -1)} \cap B_{(-1, 1)}$.
    In particular, $(-1, 1), (1, -1) \in S_{p^*}$, as desired.
\end{proof}

Using this lower bound, we show that any approximately replicable PAC learner must also use shared randomness.

\SharedRandomnessApxRepl*

\begin{proof}
    Consider a set of two shattered elements $\set{x_{1}, x_{2}}$ and the uniform distribution over $\domain$.
    We reduce from $2$-dimensional sign-one-way marginals.
    Given $p \in [0, 1]^2$, the label of $x_{i}$ is $1$ with probability $p_{i}$ and $-1$ otherwise.
    Suppose there is a deterministic approximately replicable learner, denoted $\innerAlg$ and we have sample access to a $2$-coin problem instance.
    We construct samples for $\innerAlg$ by uniformly sampling $x$ from $\domain$ and labeling using samples from the corresponding coin.
    Given the output $h \gets \innerAlg$, we return $v \gets (h(x_{1}), h(x_{2}))$.
    The excess error of $h \gets \innerAlg$ is exactly $\frac{1}{2} \sum_{i} 2 |p_{i}| \ind[h(x_{i}) \neq h^*(x_{i})| = \sum_{i} |p_{i}| \ind[h(x_{i}) \neq h^*(x_{i})]$ which is (up to a factor of 2) the error of the solution $v$.
    Furthermore, whenever $h$ is approximately replicable, either $h(x_{1})$ or $h(x_{2})$ is consistent, so that $v$ is approximately replicable as well.
    Thus, we obtain a deterministic $(0.01, 0.5)$-approximately replicable $(0.1, 0.01)$-accurate algorithm for the sign-one-way marginals problem.
    Furthermore, our reduction is deterministic, so $\innerAlg$ must be randomized from \Cref{thm:shared-randomness-lb}.
\end{proof}

%% file: semi_supervised_replicability.tex
\section{Semi-Replicable Learning}

In this section, we give our algorithms and lower bounds for semi-replicable learning.

\begin{definition}[Replicable Semi-Supervised Learning]
    Let $\hypotheses$ be a concept class.
    A (randomized) algorithm $\innerAlg$ is a $\rho$-semi-replicable $(\alpha, \beta)$-learner for $\hypotheses$ with shared unlabeled sample complexity $m_u$ and labeled sample complexity $m_s$ if for any distribution $\distribution$, $\innerAlg$ satisfies the following:
    \begin{enumerate}
        \item Let $S \sim \distribution^{m_s}$, $U \sim \distribution_{\domain}^{m_u}$ and $r$ denote the internal randomness of $\innerAlg$.
        Then
        \begin{equation*}
        \Pr_{S, U, r}\left( \errD(\innerAlg(S; U, r)) \geq \opt_{\hypotheses}(\distribution) + \alpha \right) < \beta.
        \end{equation*}
        \item Let $S_1, S_2 \sim \distribution^{m_s}$ be drawn independently and $U \sim \distribution_{\domain}^{m_u}$. 
        Let $r$ denote the internal randomness of $\innerAlg$. 
        Then
        \begin{equation*}
            \Pr_{S_1, S_2, U, r} \left( \innerAlg(S_1; U, r) \neq \innerAlg(S_2; U, r) \right) < \rho.
        \end{equation*}
    \end{enumerate}
    For convenience, we say $\innerAlg$ has sample complexity $(m_{u}, m_{s})$.
\end{definition}


\begin{theorem}[Formal \Cref{thm:semi-supervised-repl-ub}]
    \label{thm:semi-supervised-repl-ub-formal}
    There is a $\rho$-semi-replicable $(\alpha, \beta)$-learner for with shared (unlabeled) sample complexity $\bigO{\frac{d + \log(1/\beta)}{\alpha}}$ and labeled sample complexity $\bigO{\frac{d^2}{\alpha^2 \rho^2} \left(\log^2 \frac{d \log(1/\beta)}{\rho \beta \alpha} \right) \log^3 \frac{1}{\rho}}$.
\end{theorem}

We will need the following replicable learner for finite hypothesis classes of \cite{bun2023stability}. 

\begin{theorem}[Theorem 5.13 of \cite{bun2023stability}]
    \label{thm:finite-class-replicable-learner}
    Let $\hypotheses$ be a finite concept class.
    There is a $\rho$-replicable $(\alpha, \beta)$-learner for $\hypotheses$ with sample complexity
    \begin{equation*}
        \bigO{\frac{\log^2 |\hypotheses| + \log \frac{1}{\rho \beta}}{\alpha^2 \rho^2} \log^3 \frac{1}{\rho}} .
    \end{equation*}
\end{theorem}

\begin{proof}[Proof of \Cref{thm:semi-supervised-repl-ub}]
    The argument is essentially the same as the standard upper bound for the semi-private model \cite{alon2019limits,hopkins2022realizable}. We first use the shared unlabeled samples to create a small shared subset of hypotheses $\hypotheses_0$ that contains a hypothesis $\alpha/2$-close to optimal with high probability, then learn from this class replicably using \Cref{thm:finite-class-replicable-learner}. 
    
    {\bf Step 1: Constructing an $\alpha$-non-uniform cover with unlabeled samples.}

    Let $m(\alpha, \beta) = \bigO{\frac{d+\log \frac{1}{\beta}}{\alpha}}$ denote the sample complexity of an (improper) realizable learner $\innerAlg$ for hypothesis class $\hypotheses$.
    By taking $m_u = m(\alpha/2, \beta/2)$ unlabeled samples $U$ and constructing the set of hypotheses
    \begin{equation*}
        C(U) = \set{\innerAlg(h(U)) \given h \in \hypotheses}
    \end{equation*}
    obtained by running $\innerAlg$ on all possibly labellings of $U$ under $\hypotheses$, Claim 2.3 of \cite{hopkins2022realizable} shows that with probability at least $1 - \beta / 2$, there exists $h' \in C(U)$ such that $\errD(h') \leq \opt_{\hypotheses}(\distribution) + \alpha / 2$.

    {\bf Step 2: Replicably learning from the $\alpha$-non-uniform cover.}
    
    Consider now the finite hypothesis class $C(U)$.
    Note that since the concept class $\hypotheses$ has VC dimension $d$, there are at most $m_u^{O(d)}$ possible labellings of $U$ (see \Cref{lem:sauer}) so that
    \begin{equation*}
        |C(U)| \leq m_u^{O(d)}.
    \end{equation*}
    Then, we draw sufficiently many samples so that applying \Cref{thm:finite-class-replicable-learner}, we have a $\rho$-replicable $(\alpha/2, \beta/2)$-learner for $C(U)$.
    Thus the supervised sample complexity is
    \begin{equation*}
        m_s = \bigO{\frac{\log^2 |C(U)| + \log \frac{1}{\rho \beta}}{\alpha^2 \rho^2} \log^{3} \frac{1}{\rho}} = \bigO{\frac{d^2 \log^2\frac{d+\log(1/\beta)}{\alpha} + \log \frac{1}{\rho \beta}}{\alpha^2 \rho^2} \log^3 \frac{1}{\rho}}.
    \end{equation*}
    We now show the algorithm has the claimed correctness and replicability properties.

    {\bf Correctness.}
    By a union bound, with probability at least $1 - \beta$ we have that the $(\alpha/2, \beta/2)$-learner of $C(U)$ outputs a hypothesis $h$ such that
    \begin{equation*}
        \errD(h) \leq \opt_{C(U)} + \alpha/2 \leq \opt_{\hypotheses} + \alpha.
    \end{equation*}

    {\bf Replicability.}
    Since the unsupervised samples $U$ and internal randomness of $\innerAlg$ are shared between two runs of the algorithm, both have the finite concept class $C(U)$ after Step 1 of the algorithm. 
    Replicability then follows from \Cref{thm:finite-class-replicable-learner}.
\end{proof}


We now argue that \Cref{thm:semi-supervised-repl-ub} is optimal.

\subsection{Lower Bounds for Shared Sample Complexity}

First, we give a bound on the shared sample complexity.

\SemiReplicableSharedLB*

\paragraph{Lower Bound for Infinite Littlestone Dimension}

As a first step, we show that $\frac{1}{\alpha}$ shared samples are necessary. This is essentially immediate from a corresponding lower bound for semi-privacy and the now standard replicable $\to$ private transformation of \cite{GhaziKM21,bun2023stability}, but we include the details for completeness.

\begin{theorem}
    \label{thm:threshold-semi-supervised-repl-shared-lb}
    Let $\hypotheses$ be any class with infinite Littlestone Dimension (for example, thresholds over $\R$).
    Then, any $0.0001$-semi-replicable $(\alpha, 0.0001)$-learner for $\hypotheses$ requires $\bigOm{\frac{1}{\alpha}}$ shared samples.
    The lower bound holds even when the shared samples are labeled.
\end{theorem}

We will obtain our lower bound via a reduction to semi-private learning.
We remark that our lower bound holds even when the shared samples are labeled, as is the case in semi-private learning, where lower bounds hold even when the public samples are labeled.
In this setting, an algorithm $\innerAlg$ given $m_{\pub}$ public samples and $m_{\priv}$ private samples is $(\eps, \delta)$-semi-private if it is private with respect to the private samples.
Formally, for any public data set $T \sim \distribution^{m_{\pub}}$ and any two neighboring datasets $S, S' \sim \distribution^{m_{\priv}}$ (see \Cref{def:private}) we have for any subset $Y$ of the outputs,
\begin{equation*}
    \Pr(\innerAlg(T; S) \in Y) \leq e^{\eps} \Pr(\innerAlg(T; S') \in Y) + \delta \text{.}
\end{equation*}
In particular, $\innerAlg(T; \cdot)$ is private for every public dataset $T$.
We say $\innerAlg$ is an $(\alpha, \beta, \eps ,\delta)$-agnostic semi-private learner if it is $(\eps, \delta)$-private and $(\alpha, \beta)$-accurate.
Lower bounds for semi-private learning follow from a data reduction lemma and lower bounds for private learning.

\begin{lemma}[Lemma 4.4 of \cite{alon2019limits}]
    \label{lem:public-data-reduction}
    Let $0 < \alpha \leq 0.01, \eps > 0, \delta > 0$. 
    Suppose there is an $(\alpha, \frac{1}{18} ,\eps, \delta)$-agnostic semi-private learner for a hypothesis class $\hypotheses$ with private sample size $n_{\priv}$ and public sample size $n_{\pub}$. Then, there is a $(100 n_{\pub} \alpha, \frac{1}{16}, \eps, \delta)$-private learner that learns any distribution realizable by $\hypotheses$ with input sample size $\lceil \frac{n_{\priv}}{10 n_{\pub}} \rceil$.
\end{lemma}

\begin{theorem}[Theorem 1 of \cite{alon2019private}]
    \label{thm:private-implies-finite-littlestone}
    Let $\hypotheses$ be a class with Littlestone Dimension $n$.
    Let $\innerAlg$ be a $(0.01, 0.01)$-accurate algorithm for $\hypotheses$ satisfying $(\eps, \delta)$-privacy with $\eps = 0.1$ and $\delta = \bigO{\frac{1}{m^2 \log m}}$.
    Then,
    \begin{equation*}
        m \geq \Omega(\log^* n) \text{.}
    \end{equation*}
    In particular, any class that is privately learnable has finite Littlestone Dimension.
\end{theorem}

In particular, whenever $n_{\pub} = o(\alpha^{-1})$, we must have $n_{\priv} = \bigOm{n_{\pub} \log^* n}$.
We show that any semi-replicable algorithm can be converted into a semi-private learner.

\begin{lemma}
    \label{lem:semi-private-reduction}
    Let $\beta, \rho > 0$ be sufficiently small constants.
    Let $\innerAlg$ be a $\rho$-semi-replicable $(\alpha, \beta)$-learner for $\hypotheses$ with sample complexity $(m_{u}, m_{s})$.
    
    Then, there is a $(\alpha, \beta \log(1/\beta), \eps, \delta)$-semi-private learner for $\hypotheses$ using $m_{\pub} = \bigO{m_{u} \log (1/\beta)}$ public samples and $m_{\priv} = \bigO{m_{s} \left( \frac{\log(1/\delta)}{\eps} + \log(1/\beta) \right) \log(1/\beta)}$ private samples.
\end{lemma}

\begin{proof}
    Suppose $\innerAlg$ is a $\rho$-semi-replicable $(\alpha, \beta)$-learner with sample complexity $(m_{u}, m_{s})$.
    We follow the framework of replicability-privacy reductions of \cite{GhaziKM21, bun2023stability}.
    A key tool we require is private selection.

    \begin{theorem}[Private Selection \cite{korolova2009releasing, bun2016simultaneous, BunDRS18}]
        \label{thm:private-selection}
        There exists some $c > 0$ such that for every $\eps, \delta > 0$ and $m \in \N$, there is an $(\eps, \delta)$-private algorithm $\privSelection$ that on input $S \in \domain^{m}$, outputs with probability $1$ an element $x \in \domain$ that occurs in $S$ at most $\frac{c \log(1/\delta)}{\eps}$ times fewer than the true mode of $S$.
        Moreover, the algorithm runs in $\poly(m, \log(|\domain|))$-time.
    \end{theorem}

    We now describe our algorithm.

    \IncMargin{1em}
    \begin{algorithm}
    
    \SetKwInOut{Input}{Input}\SetKwInOut{Output}{Output}\SetKwInOut{Parameters}{Parameters}
    \Input{$(\alpha, \beta, \rho)$-semi-replicable learner $\innerAlg$, sample access to $\distribution$.}
    \Parameters{$(\eps, \delta)$ privacy parameters and $(\eps, \delta)$-accuracy parameters.}
    \Output{$(\alpha, \beta, \eps, \delta)$-semi-private leaner.}

    Set $k_1 \gets O(\log(1/\beta))$ and $k_2 \gets \bigO{\frac{\log(1/\delta)}{\eps} + \log(1/\beta)}$.
    
    \For{$i \in [k_1]$}{
        Sample $T_{i} \sim \distribution^{m_{u}}$ and random string $r_{i}$.
        
        \For{$j \in [k_2]$}{
            Sample $S_{i, j} \sim \distribution^{m_{s}}$ using fresh samples and compute $y_{i, j} \gets \innerAlg((T_{i}, S_{i, j}); r_{i})$.
        }
    }

    \Return $\privSelection(\set{y_{i, j}}, \eps, \delta)$.
    
    \caption{$\semiRepltoSemiPriv(\innerAlg, \alpha, \beta, \eps, \delta)$}
    \label{alg:semi-repl-semi-prev-red}
    
    \end{algorithm}
    \DecMargin{1em}

    The stated algorithm has public sample complexity $m_{\pub} = \bigO{k_1 m_{u}} = \bigO{m_{u} \log(1/\beta)}$ and private sample complexity $m_{\priv} = \bigO{k_1 k_2 m_{s}} = \bigO{m_{s} \left( \frac{\log(1/\delta)}{\eps} + \log(1/\beta) \right) \log(1/\beta)}$.

    We argue that the algorithm is semi-private.
    Note that $\set{T_{i}}$ consists of public samples and $\set{S_{i, j}}$ consists of private samples.
    Consider two neighboring datasets $S, S'$ of private samples.
    Since fresh samples are drawn for each $(i, j)$, note that there is exactly one index $(i, j)$ where $S_{i, j}, S_{i, j}'$ are neighboring instead of identical. 
    Then, $\innerAlg((T_{i}, S_{i, j}), r_{i})$ differ in exactly one run.
    We thus inherit $(\eps, \delta)$-differential privacy from the guarantees of \Cref{thm:private-selection}.

    We now argue correctness via two events (as in \cite{bun2023stability}).
    Consider the set of output hypotheses $\set{y_{i, j}}$.
    There exists a constant $c$ such that the following two conditions hold:
    \begin{enumerate}
        \item With probability $1 - \frac{\beta}{2}$, some hypothesis appears $t_{1} := 2 c \left( \frac{\log(1/\delta)}{\eps} + \log(1/\beta) \right)$ times.
        \item With probability $1 - \frac{\beta \log(1/\beta)}{2}$, any element appearing $t_{2} := c \left( \frac{\log(1/\delta)}{\eps} + \log(1/\beta) \right)$ times is correct.
    \end{enumerate}
    Towards the first item, note that 
    \begin{equation*}
        \E_{r, T \sim \distribution^{m_{u}}} \E_{S, S' \sim \distribution^{m_{s}}} \left[ \ind\left[\innerAlg((T, S); r) \neq \innerAlg((T, S'); r) \right] \right] \leq \rho  \text{.}
    \end{equation*}
    By Markov's inequality, with probability $\frac{1}{2}$, a randomly chosen $(r, T)$ satisfy 
    \begin{equation*}
        \E_{S, S' \sim \distribution^{m_{s}}} \left[ \ind\left[\innerAlg((T, S); r) \neq \innerAlg((T, S'); r) \right] \right] \leq 2 \rho \text{.}
    \end{equation*}
    In particular, with probability $1 - \frac{\beta}{4}$, some $r_{i}, T_{i}$ satisfy the desired condition.
    Now, by a Chernoff bound, we can guarantee that with probability at least $1 - \frac{\beta}{4}$, the canonical element corresponding to $r_{i}, T_{i}$ appears at least $t_1$ times.
    
    Towards the second item, since $\innerAlg$ is $\beta$-accurate, we observe that in expectation at most a $\beta$-fraction of hypotheses are incorrect.
    Then, Markov's inequality implies that at most $b := \bigO{\frac{k_1 k_2}{\log(1/\beta)}} = O(k_2)$ hypotheses are incorrect with probability $1 - \beta \log(1/\beta)$.
    For a small enough (constant) choice of $\beta$, we have $b < t_{2}$, thus proving the claim.

    Thus, the mode occurs at least $t_1$ times, and via $\privSelection$ we select an element that occurs at least $t_{2}$ times, which we guarantee is correct.
\end{proof}

\begin{proof}[Proof of \Cref{thm:threshold-semi-supervised-repl-shared-lb}]
    Let $\beta, \rho$ be sufficiently small constants.
    Suppose we have a $\rho$-semi-replicable $(\alpha, \beta)$-learner for thresholds over $[n]$ with sample complexity $(m_{u}, m_{s})$.
    Then, via \Cref{lem:semi-private-reduction}, we have a $(\alpha, \beta \log(1/\beta), \eps, \delta)$-semi-private learner for thresholds over $n$ with $m_{\pub} = \bigO{m_{u} \log (1/\beta)}$ public samples and $m_{\priv} = \bigO{m_{s} \left( \frac{\log(1/\delta)}{\eps} + \log(1/\beta) \right) \log(1/\beta)}$ private samples.
    It suffices to show the following claim.
    In particular, since $\beta$ is small constant, applying \Cref{lem:public-data-reduction} we obtain an $(100 m_{\pub} \alpha, 0.01, \eps, \delta)$-private learner using $m := \bigO{\frac{m_{\priv}}{m_{\pub}}} = \bigO{\frac{m_{s}}{m_{u}}\left( \frac{\log(1/\delta)}{\eps} + \log(1/\beta) \right)}$ samples.
    Now, if $m_{u} = \litO{\frac{1}{\alpha}} = \litO{\frac{1}{\alpha \log(1/\beta)}}$, we obtain a $(0.01, 0.01, 0.1, \frac{1}{\poly(m)})$-private learner (by setting $\eps < 0.1$, $\delta < \frac{1}{\poly(m_{s})}$).
    Thus, \Cref{thm:private-implies-finite-littlestone} implies that
    \begin{equation*}
        \frac{m_{s}}{m_{u}} \log(m_{s}) = \bigOm{\log^* n} \text{.}
    \end{equation*}
    Thus $m_{s} \log(m_{s}) = \bigOm{\log^* n}$.
    In particular, any semi-private learner for $\hypotheses$ with infinite Littlestone dimension must have $m_{u} = \bigOm{\frac{1}{\alpha}}$.

\end{proof}

\paragraph{Lower Bound for Large VC Dimension}

Towards proving \Cref{thm:semi-supervised-repl-shared-lb}, we fix a class $\hypotheses$ with VC dimension $d$ and infinite Littlestone Dimension and show that any semi-private learner for $\hypotheses$ requires $\bigOm{\frac{d}{\alpha}}$ public samples.
This follows from a similar hardness amplification lemma as in \cite{BunNSV15}.

First, we define the notion of an empirical learner and note that it suffices to prove hardness against empirical learners.
This is similar to Lemma 5.8 of \cite{BunNSV15}.
For a semi-private learner, we define the total complexity $m = m_{\priv} + m_{\pub}$ as the sum of the private and public sample complexities.

\begin{definition}
    \label{def:empirical-semi-private-learner}
    Algorithm $\innerAlg$ is an $(\alpha, \beta)$-accurate empirical learner for a hypothesis class $\hypotheses$ over $\domain$ with (total) sample complexity $m$ if for every $h \in \hypotheses$ and for every database $D = ((x_i,h(x_i)),...,(x_m,h(x_m)))$, algorithm $\innerAlg$ outputs a hypothesis $\hat{h} \in \hypotheses$ satisfying $\Pr_{\innerAlg}(\err_{D}(\hat{h}) \leq \alpha) \geq 1 - \beta$.
    Here, $\err_{D}(\hat{h}) = \frac{1}{m} \sum_{i = 1}^{m} \ind[h(x_{i}) \neq \hat{h}(x_{i})]$.
\end{definition}


\begin{lemma}[Lemma 5.8 of \cite{BunNSV15}]
    \label{lem:empirical-semi-private-learner-equivalence}
    Any $(\alpha, \beta)$-accurate algorithm for empirically learning the class of thresholds with $m$ samples is also a $(2\alpha, \beta + \beta')$-accurate PAC learner for thresholds when given at least $\max(m, 4 \log (2/\beta')/\alpha)$ samples.
\end{lemma}

We say an hypothesis $h$ is $\alpha$-consistent with dataset $S$ if its empirical error $\errS(h)$ is at most $\alpha$. 

\begin{lemma}
    \label{thm:semi-private-dim-reduction}
    Let $\alpha, \beta, \varepsilon, \delta > 0$. 
    Let $\hypotheses$ be the class of thresholds over $[n]$, and assume there is a minimum element $0$ such that $h(0) = 1$ for every $h \in \hypotheses$. 
    Let $\hypotheses^{\wedge d} \subset \set{[n]^{d} \rightarrow \set{\pm 1}}$ be the class of conjunctions of $d$ thresholds: the set of functions $(x_{1}, \dotsc, x_{d}) \rightarrow \bigwedge_{i = 1}^{d} h_{i}(x_{i})$ where $h_{i} \in \hypotheses$.
    Then any $(\alpha, \beta, \eps, \delta)$-semi-private learner for $\hypotheses^{\wedge d}$ with $\litO{\frac{d}{\alpha}}$ public examples requires $\bigOm{d \log^* n}$ private examples.
\end{lemma}

\begin{proof}
    Suppose $\innerAlg_0$ is an $(\alpha, \beta, \eps, \delta)$-semi-private learner for $\hypotheses^{\wedge d}$ with $m_{\pub}$ public examples and $m_{\priv}$ private examples.
    We give an algorithm $\innerAlg_{1}$ that is a semi-private learner for the class of thresholds.
    
    \IncMargin{1em}
    \begin{algorithm}
    
    \SetKwInOut{Input}{Input}\SetKwInOut{Output}{Output}\SetKwInOut{Parameters}{Parameters}
    \Input{$(\alpha, \beta, \eps, \delta)$-semi-private learner $\innerAlg_{0}$. Databases $D_{\pub}$ of $n_{\pub}$ public examples (denoted $(x_{j}^{(\pub)}, y_{j}^{(\pub)})$) and $D_{\priv}$ of $n_{\priv}$ private examples (denoted $(x_{j}^{(\priv)}, y_{j}^{(\priv)})$). Sample access to $\distribution$.}
    \Output{Semi-private learner for thresholds.}

    Initiate $S_{\pub}, S_{\priv}$ as empty multisets.
    
    Sample $r \in [d]$ uniformly.

    For $i \in [d]$, let $f_i(x)$ map $x$ to vector $v \in \R^{d}$ with $v[i] = x$ and $v[j] = 0$ for $j \neq i$.

    \For{$i \in [d]$}{
        \If{$i = r$}{
            \For{$j \in [n_{\pub}]$}{
                Let $z^{(\pub)}_{(r, j)} \gets \left(f_{r}\left(x_{j}^{(\pub)}\right), y_{j}^{(\pub)}\right)$.
                Add $z^{(\pub)}_{(r, j)}$ to $S_{\pub}$.
            }

            \For{$j \in [n_{\priv}]$}{
                Let $z^{(\priv)}_{(r, j)} \gets \left(f_{r}\left(x_{j}^{(\priv)}\right), y_{ j}^{(\priv)}\right)$.
                Add $z^{(\priv)}_{(r, j)}$ to $S_{\priv}$.
            }
        }\Else{
            Let $D_{\pub}^{(i)} = \set{(x_{i, j}^{(\pub)}, y_{i, j}^{(\pub)})}, D_{\priv}^{(i)} = \set{(x_{i, j}^{(\priv)}, y_{i, j}^{(\priv)})}$ denote a fresh sample from $\distribution$. 
            
            \For{$j \in [n_{\pub}]$}{
                Let $z^{(\pub)}_{(i, j)} \gets \left(f_{i}\left(x_{i, j}^{(\pub)}\right), y_{i, j}^{(\pub)}\right)$.
                Add $z^{(\pub)}_{(i, j)}$ to $S_{\pub}$.
            }

            \For{$j \in [n_{\priv}]$}{
                Let $z^{(\priv)}_{(i, j)} \gets \left(f_{i}\left(x_{i, j}^{(\priv)}\right), y_{i, j}^{(\priv)}\right)$.
                Add $z^{(\priv)}_{(i, j)}$ to $S_{\priv}$.
            }
        }
    }

    Compute $h_0 \gets \innerAlg_0(S_{\pub}, S_{\priv})$

    Return $h: x \mapsto h_0(f_{r}(x))$.
    
    \caption{$\semiPrivDimRed(\innerAlg_0)$}
    \label{alg:semi-repl-dim-reduction}
    
    \end{algorithm}
    \DecMargin{1em}

    First, observe that $\semiPrivDimRed$ is $(\eps, \delta)$-semi-private.
    Note that a single change in $D_{\priv}$ leads to a single change in the multi-set $S_{\priv}$.
    Since $\innerAlg_{0}$ is $(\eps, \delta)$-semi-private, so is the output $h_0$. 
    Finally, the output $h$, and therefore $\semiPrivDimRed$, is $(\eps, \delta)$-semi-private by post-processing.

    We now argue the correctness of $\semiPrivDimRed$.
    Consider the execution of $\semiPrivDimRed$ on databases $D_{\pub}, D_{\priv}$ of size $n_{\pub}, n_{\priv}$, sampled from $\distribution$, where $\distribution$ is the hard meta-distribution for private threshold learning guaranteed by \Cref{thm:private-implies-finite-littlestone}.
    That is, $\distribution$ is a distribution over distributions such that any $(\eps, \delta)$-private algorithm with $\eps = 0.1$ and $\delta = O(1/ (m^{2} \log m))$ and $(0.01, 0.01)$-accurate over the meta-distribution $\distribution$ requires sample complexity $m \geq \Omega(\log^* n)$.
    Note that $\distribution$ is known to the algorithm, so we can draw fresh samples from $\distribution$ without any additional sample complexity.
    
    First, we argue that $\innerAlg_{0}$ is applied on multi-sets $S_{\pub}, S_{\priv}$ labeled by some hypothesis $h \in \hypotheses^{\wedge d}$.
    For each $i$, let $D_{\pub}^{(i)}, D_{\priv}^{(i)}$ denote the datasets sampled from $\distribution$, which are labeled by some hypothesis $h_{i} \in \hypotheses$.
    Define $h: x \rightarrow \bigwedge_{i = 1}^{d} h_{i}(x_{i})$. 
    Then, for all examples $(x_{i, j}, y_{i, j}) \in D_{\priv}^{(i)}, D_{\pub}^{(i)}$ where public or private,
    \begin{equation*}
        h(z_{i, j}) = h_1(0) \wedge \dotsc \wedge h_{i - 1}(0) \wedge h_{i}(x_{i, j}) \wedge h_{i + 1}(0) \wedge \dotsc \wedge h_{d}(0) = h_{i}(x_{i, j}) = y_{i, j}
    \end{equation*}
    as desired.
    Then, $\innerAlg_{0}$ with probability at least $1 - \beta$ finds an $\alpha$-consistent hypothesis $h_{0}$ with respect to $S_{\pub}, S_{\priv}$.
    We now analyze the error of $\innerAlg_{0}$.
    Let $\distribution_{i}$ denote the distribution on databases of size $n$ obtained by applying $f_{i}$ to each example in $D_{\pub}, D_{\priv} \sim \distribution$.
    Note that $S_{\pub} = \bigcup_{i = 1}^{d} D_{\pub}^{(i)}$ and $S_{\priv} = \bigcup_{i = 1}^{d} D_{\priv}^{(i)}$.
    Let $D^{(i)} = D_{\pub}^{(i)} \cup D_{\priv}^{(i)}$.
    Then, we may write the empirical error of $h_{0}$ as
    \begin{equation*}
        \alpha \geq \err_{S}(h_{0}) = \frac{1}{d} \sum_{i = 1}^{d} \err_{D^{(i)}}(h_{0}) \text{.}
    \end{equation*}
    Then, there are at most $\beta d$ indices on which $\err_{S_{i}}(h_{0}) \geq \frac{\alpha}{\beta}$.
    Since $r$ is uniformly chosen at random, and the points on every axis are distributed in exactly the same way (a dataset drawn from $\distribution$), the probability that $r$ is one of the $\beta d$ indices occurs with probability at most $\beta$.
    Thus, with probability $1 - 2 \beta$ we have that $\err_{S_{r}}(h_{0}) \leq \frac{\alpha}{\beta}$.
     Finally, we bound the error of $h$ on $D_{\pub}, D_{\priv}$.
    Note that for any point $(x_{j}, y_{j})$ (either a public or private sample),
    \begin{equation*}
        h(x_{j}) = h_{0}(f_{r}(x_{j})) = h_{0}(z_{r, j})
    \end{equation*}
    so that $\err_{D}(h) = \err_{S_{r}}(h_{0}) \leq \frac{\alpha}{\beta}$.

    In particular, we obtain a $\left( \frac{\alpha}{\beta}, 2 \beta, \eps, \delta \right)$-semi-private empirical learner for thresholds with $\frac{m_{\pub}}{d}$ public samples and $\frac{m_{\priv}}{d}$ private samples.
    Applying \Cref{lem:empirical-semi-private-learner-equivalence}, we obtain a $\left( \frac{2 \alpha}{\beta}, 3 \beta, \eps, \delta \right)$-semi-private learner for thresholds with $\frac{m_{\pub}}{d}$ public samples and $\frac{m_{\priv}}{d} + \frac{4 \log(2/\beta)}{\alpha}$ private samples.
    Suppose that $m_{\pub} = \litO{d/\alpha}$, so that we obtain a semi-private learner with $\litO{1/\alpha}$ public examples. 
    Then, from \Cref{lem:public-data-reduction} and \Cref{thm:private-implies-finite-littlestone} we have that 
    \begin{equation*}
        \frac{m_{\priv}}{d} + \frac{4 \log(2/\beta)}{\alpha} \geq \log^* n \text{.}
    \end{equation*}
    In particular, for sufficiently large $n$, we have $m_{\priv} = \bigOm{d \log^* n}$, as desired.
\end{proof}

Finally, we are ready to prove \Cref{thm:semi-supervised-repl-shared-lb}.

\begin{proof}[Proof of \Cref{thm:semi-supervised-repl-shared-lb}]
    Let $\beta, \rho$ be sufficiently small constants.
    Consider the class of thresholds over $[n]^{d}$, $\hypotheses^{\wedge d}$.
    Suppose we have a $\rho$-replicable $(\alpha, \beta)$-semi-replicable learner for $\hypotheses^{\wedge d}$ with sample complexity $(m_{u}, m_{s})$.
    Applying \Cref{lem:semi-private-reduction}, we obtain a $(\alpha, \beta \log(1/\beta), \eps, \delta)$-semi-private learner for $\hypotheses^{\wedge d}$ with $\bigO{m_{u} \log(1/\beta)}$ public samples and $\bigO{m_{s} \left(\frac{\log(1/\delta)}{\eps} + \log(1/\beta) \right) \log(1/\beta)}$ private samples.
    Finally, from \Cref{thm:semi-private-dim-reduction}, we may conclude that if $m_{u} = \litO{\frac{d}{\alpha \log(1/\beta)}}$, we must have 
    \begin{equation*}
        m_{s} \left(\frac{\log(1/\delta)}{\eps} + \log(1/\beta) \right) \log(1/\beta) = \bigOm{d \log^* n} \text{.}
    \end{equation*}
    This implies that over an infinite domain, thresholds over $\R$ require $\bigOm{\frac{d}{\alpha \log(1/\beta)}}$ shared sample complexity.
\end{proof}

\subsection{Lower Bound for Supervised Sample Complexity}

Finally, we give our near-matching lower bound on the supervised sample complexity.

\SemiReplicableLB*

Our lower bound follows from a reduction to the sign-one-way marginals problem (\Cref{def:sign-one-way-marginals}).
Recall that \cite{bun2023stability} show that any $0.0005$-replicable algorithm $(0.02, 0.002)$-accurately solving the sign-one-way marginals problem requires sample complexity $\tOm{d}$.
\cite{hopkins2025generative} showed that any $\rho$-replicable $(\alpha, 0.001)$-accurate algorithm requires $\tOm{d\rho^{-2}\alpha^{-2}}$ samples.

\begin{theorem}[\cite{hopkins2025generative}]
    \label{thm:sign-one-way-marginals-lb}
    Let $\alpha < 0.001$ be small constants.
    Any $\rho$-replicable algorithm $(\alpha, 0.001)$-accurately solving the sign-one-way marginals problem requires sample complexity $\bigOm{\frac{d}{\rho^{2} \alpha^{2} \log^{6} d}}$.
\end{theorem}

We now present a lower bound for semi-supervised learning given \Cref{thm:sign-one-way-marginals-lb}.
Let $\hypotheses$ be a VC class of dimension $d$.
Consider a set of $d$ shattered elements $\set{x_{1}, \dotsc,x_{d}}$.
We provide a generic lemma that says any algorithm that PAC learns $\hypotheses$ with respect to these distributions solves the sign-one-way marginals problem.


    

\begin{proof}[Proof of \Cref{thm:semi-supervised-repl-lb}]
    Assume without loss of generality $m = \bigOm{d \log d}$.\footnote{We will show that whenever $m = \bigOm{d \log d}$ that $m = \bigOm{\frac{d^2}{\rho^2 \alpha^2 \log^{6} d}}$. 
    Any algorithm on $O(d \log d)$ samples implies one between $O(d \log d)$ and $\bigO{\frac{d^2}{\rho^2 \alpha^2 \log^{6} d}}$, giving a contradiction for large enough $d$.}
    Consider a set of $d$ shattered elements $\set{x_1, \dotsc, x_{d}}$.
    Suppose $\innerAlg$ is a $\rho$-replicable $(\alpha, 0.0001)$-semi-supervised learner for $\hypotheses$ with sample complexity $m$.
    We design a $\rho$-replicable $(2 \alpha, 0.0002)$-accurate learner for the sign-one-way marginals problem with sample complexity $O(md)$.

    Suppose we receive $m'$ samples from $\distribution_{p}$ for some expectation vector $p$.
    We let the distribution $\distribution$ over $\domain$ be uniform.
    Thus, whenever $\innerAlg$ demands an unsupervised sample, we can easily generate one uniformly from $\distribution$.
    
    Whenever a supervised sample is required by $\innerAlg$, we uniformly sample $x_i \in \domain$ and label it according to a fresh sample from the $i$-th coordinate of $\distribution_{p}$.
    Note that with $m$ samples from $\distribution$, any element of $\domain$ is sampled $\frac{m}{d}$ times in expectation, and more than $C m / d$ times for some large constant $C$ with probability at most $\exp(-\Omega(m/d)) < \frac{0.0001}{d}$ since we have assumed $m = \bigOm{d \log d}$.
    By a union bound, no element is sampled more than $C m/d$ times with probability at least $1 - 0.0001$.
    We condition on this event.
    In particular, we as long as we have $m' = \bigO{m/d}$ samples from $\distribution_{p}$, we can successfully label every sampled point from the domain $\distribution$.
    Let $S$ denote the $m$ uniform samples from $\domain$ labeled with the corresponding coordinates from $\distribution_{p}$.
    Note that also the samples of $S$ are independent.

    Thus, let $f \gets \innerAlg(S)$ be the output of the semi-supervised learner.
    We condition on the event $\errD{f} \leq \errD(f_{\opt}) + \alpha$, which fails with probability at most $0.0001$ and output the vector $v$ with $v_i = f(x_i)$ for the sign-one-way marginals problem.
    It is clear that any output hypothesis $f$ has error
    \begin{equation*}
            \frac{1}{d} \left( \sum_{f(x_i) = +1} \frac{1 - p_i}{2} + \sum_{f(x_i) = -1} \frac{1 + p_i}{2} \right) = \frac{1}{d} \left( \sum_{i = 1}^{d} \frac{1 - |p_i|}{2} + \sum_{f(x_i) \neq \sign(p_i)} |p_i| \right).
    \end{equation*}
    Thus, the optimal hypothesis is $f_{\opt}(x_i) = \sign(p_i)$ and our hypothesis satisfies
    \begin{equation*}
        \frac{1}{d} \sum_{f(x_i) \neq \sign(p_i)} |p_i| = \frac{1}{d} \sum_{v_i \neq \sign(p_i)} |p_i| \leq \alpha \text{.}
    \end{equation*}
    Thus, the vector $v$ satisfies
    \begin{equation*}
        \frac{1}{d} \sum_{i = 1}^{d} |p_i| - v_i p_i = \frac{1}{d} \sum_{v_i \neq \sign(p_i)} 2 |p_i| \leq 2 \alpha.
    \end{equation*}
    In particular, we obtain a $(2 \alpha, 0.0002)$-accurate learner for the sign-one-way marginals problem.
    Furthermore, our algorithm for the sign-one-way marginals problem is $\rho$-replicable since the semi-supervised learner $\innerAlg$ is $\rho$-replicable.
    From \Cref{thm:sign-one-way-marginals-lb}, we know that $m = \bigOm{\frac{d}{\rho^2 \alpha^{2} \log^{6} d}}$ so that $m' = \bigOm{\frac{d^2}{\rho^2 \alpha^{2} \log^{6} d}}$.
\end{proof}

%% file: appendix_apx_repl.tex
\section{Additional Algorithms for Approximate Replicability}

In this appendix, we give our additional algorithms for approximately replicable learning.

\subsection{An Algorithm with Finite Randomness}

With \Cref{thm:apx-repl-chernoff}, we presented an algorithm that possibly uses an infinite amount of randomness (since we draw a fresh random string $r(x)$ for every $x \in \domain$) on infinite domains.
Below, we give an algorithm that uses a finite amount of randomness (regardless of the domain) but achieves worse dependence with respect to accuracy and replicability parameters.
Note that we still achieve the correct dependence on the VC dimension $d$.

\begin{theorem}[Formal \Cref{thm:apx-repl-finite-r-informal}]
    \label{thm:apx-repl-finite-r}
    Let $\Acal$ be an (agnostic) $(\alpha, \beta)$-learner on $m(\alpha, \beta)$ samples. 
    There exists
    \begin{enumerate}
        \item an (agnostic) $\rho$-pointwise replicable $(\alpha, \beta)$-learner 
        \item an (agnostic) $(\rho, \gamma)$-approximately replicable $(\alpha, \beta)$-learner
    \end{enumerate}
    with sample complexity
    \begin{align*}
        m(\alpha, \beta, \rho, \gamma) &= \bigtO{m(\alpha, \beta) \left( \frac{\log^{3}(1/\beta)}{\rho^2 \gamma^2 \alpha^3} + \frac{\log^{4}(1/\beta)}{\rho^{4} \alpha^{5}} \right)} \\
        &= \bigtO{\frac{d \log^{4}(1/\beta)}{\rho^2 \gamma^2 \alpha^{5}} + \frac{d \log^{5}(1/\beta)}{\rho^{4} \alpha^{7}}} \text{.}
    \end{align*}
    Furthermore, our algorithm runs in time linear in sample complexity with $\bigtO{\frac{\log^{3}(1/\beta)}{\rho^2 \gamma^2 \alpha^3} + \frac{\log^{4}(1/\beta)}{\rho^{4} \alpha^{5}}}$ oracle calls to $\Acal$.
\end{theorem}

A key tool we require is a new analysis of replicable hypothesis selection, a mechanism for selecting the best option in a set of candidates studied in \cite{GhaziKM21, hopkins2025generative}. 
In particular, we will show that existing algorithms for hypothesis selection can be made replicable even under the weaker assumption that across our two `independent' runs of the algorithm, each candidate may have slightly different rewards.

\begin{definition}[$\alpha$-Optimal Hypothesis Selection]
    Given hypotheses $f_{1}, \dotsc, f_{n}: X \rightarrow Y$ and sample access to distribution $\distribution$ on $X \times Y$, output index $i$ such that $\errD(f_{i}) \leq \min_{j} \errD(f_{j}) + \alpha$.
\end{definition}

\begin{definition}
    An algorithm $\innerAlg$ solving the $\alpha$-Optimal Hypothesis Selection problem is $(\rho, \tau)$-robustly replicable if, given sample access to distribution $\distribution$ and hypotheses $f_{1}, \dotsc, f_{n}, g_{1}, \dotsc, g_{n}: X \rightarrow Y$ satisfying $\dist_{\distribution}(f_{i}, g_{i}) < \tau$ for all $i \in [n]$, 
    \begin{equation*}
        \Pr_{S, S' \sim \distribution^{m}, r \sim R}\left(\innerAlg(S, f_{1}, \dotsc, f_{n}; r) \neq \innerAlg(S', g_{1}, \dotsc, g_{n}; r) \right) < \rho
    \end{equation*}
    where $m$ is the sample complexity of $\innerAlg$.
\end{definition}

We give a robustly replicable algorithm for hypothesis selection.

\begin{theorem}
    \label{thm:optimal-hypothesis-sel}
    Let $\rho, \alpha, \beta > 0$.
    Let $0 < \tau \ll \frac{\rho \alpha}{\log(n/\beta)}$ for some sufficiently small constant.
    There is a $(\rho + \beta, \tau)$-robustly replicable algorithm solving the $\alpha$-Optimal Hypothesis Selection problem with sample complexity $m = \bigO{\frac{\log^{2}(n/\beta)}{\tau^2}}$.
    Furthermore, the algorithm runs in time $O(mn)$.
\end{theorem}

We defer the proof to \Cref{sec:hyp-sel} and show how to use \Cref{thm:optimal-hypothesis-sel} to obtain replicable prediction. 
We give a brief overview of the intuition behind our algorithm.

Consider a (not necessarily replicable) $(\alpha, \beta)$-PAC learning algorithm $\Acal$ with $m := m(\alpha, \beta)$ samples.
It is not too hard to show that for any $\eta > 0$, if we generate $T = O(\eta^{-2})$ hypotheses $h_{1}, \dots, h_{t}$ on fresh samples of size $m$ and compute $\hat{p}(x) = \frac{1}{T} \sum_{t} h_{t}(x)$ and sample a single random threshold $r \sim \UnifD[-1, 1]$, the aggregated hypothesis $\tilde{h}(x) := \ind[\hat{p}(x) \geq r]$ is $\eta$-pointwise replicable (see e.g. \Cref{lem:p-estimate}).

In particular, in expectation (see \Cref{prop:non-uniform-to-approx}), two output hypotheses will not be replicable on an $\eta$-fraction of the distribution.
By taking $\eta \ll \rho \gamma$, Markov's inequality ensures that with probability $\rho$, two output hypotheses (on independent samples with the same shared randomness) will not be replicable on a $\gamma$-fraction of the distribution.
In other words, the above procedure gives a $(\rho, \gamma)$-approximately replicable learner.

Unfortunately, we can not guarantee the above procedure produces an $(\alpha, \beta)$-learner.
While the expected error of the above procedure is at most $\opt + \alpha$ (see \Cref{lem:exp-error-ub}), we need to ensure this holds with high probability.
To do so, we observe that given a random variable $X$ in $[0, 1]$ with expectation $\mu$, we have $\Pr(X < \mu + \alpha) \geq \alpha$ (see \Cref{lem:bounded-markov}).
In particular, if we sample $O(1/\alpha)$ independent hypotheses using our above procedure (using independent shared random strings), at least one will have error $\opt + O(\alpha)$ with high probability.
Then, our hypothesis selection algorithm will output an $\alpha$-optimal hypothesis among our $O(1/\alpha)$ candidate hypothesis, thus returning a hypothesis with error $\opt + O(\alpha)$.

To enforce replicability, we must ensure that \Cref{thm:optimal-hypothesis-sel} selects the same index over both runs.
Our above discussion then ensures that every pair of hypotheses in the same index are $\gamma$-close in distribution distance.
To conclude the proof, we observe that it suffices to set $\gamma \ll \rho \alpha$ to satisfy the input assumption of \Cref{thm:optimal-hypothesis-sel}.
Overall, our sample complexity is $O(m \alpha^{-1} \eta^{-2})$ where $\eta \ll \rho^2 \alpha^2$ so that the final sample complexity is roughly $O(d \alpha^{-7} \rho^{-4})$.

\begin{proof}[Proof of \Cref{thm:apx-repl-finite-r}]
    We are now ready to prove \Cref{thm:apx-repl-finite-r}

    \IncMargin{1em}
    \begin{algorithm}
    
    \SetKwInOut{Input}{Input}\SetKwInOut{Output}{Output}\SetKwInOut{Parameters}{Parameters}
    \Input{PAC-learner $\Acal$ with sample complexity $m(\alpha, \beta)$.}
    \Parameters{$\alpha$ accuracy, $\beta$ error probability, and $\rho$ replicability}
    \Output{$\rho$-pointwise replicable $(\alpha, \beta)$-accurate PAC-learner $\tilde{\Acal}$.}
    
    \caption{$\rLearnerFinite(\Acal, \alpha, \beta, \rho)$}
    \label{alg:non-uniform-replicability}

    Let $R \gets \frac{C \log(1/\beta)}{\alpha}, \gamma \gets \min\left(\gamma, \frac{\rho \alpha}{C \log(R/\min(\rho, \beta))}\right), \eta \gets \frac{\rho \gamma}{R}, T \gets \frac{C}{\eta^2}$ for a sufficiently large constant $C$.
    
    Run $\Acal$ on $RT$ fresh samples of size $m(\alpha, \beta)$, obtaining hypotheses $h_{i, t}: \domain \rightarrow \set{\pm 1}$ for $i \in [R], t \in [T]$.

    For any fixed $i \in [R]$, $x \in \domain$, define $\hat{p}_{i}(x) := \frac{1}{T} \sum_{t = 1}^{T} h_{i, t}(x)$.

    Draw $R$ random thresholds $r_1, \dots, r_{R} \in [-1, 1]$.

    For every $i \in [R]$, define hypotheses $\tilde{h}_{i}(x) = \begin{cases}
        1 & \hat{p}_{i}(x) \geq r_{i} \\
        -1 & \hat{p}_{i}(x) < r_{i}
    \end{cases}$

    Apply \Cref{thm:optimal-hypothesis-sel} to $(\rho, \gamma)$ select an $\alpha$-optimal $\tilde{h}_{i}$. Let $\hat{i} \in [R]$ denote the selected index.

    \Return $\tilde{h}_{\hat{i}}$.
    
    \end{algorithm}
    \DecMargin{1em}

    \paragraph{Step 1: Candidate hypotheses $\tilde{h}_{i}$ are close in distribution distance.}
    
    We begin by proving that (with probability $\geq 1 - \rho$) two runs of the algorithm will produce candidate hypotheses $\tilde{h}_{i}$ that agree on a large fraction of the domain under $\distribution$.
    For any fixed $x$, let $p(x) := \E_{S, \Acal}[\Pr(\Acal(S)(x) = 1)]$ denote the probability (over a random training dataset) that $\Acal$ outputs a hypothesis mapping $x \mapsto 1$.
    Over two runs of the algorithm, we obtain hypotheses $H^{(1)} := \set{h_{i, t}^{(1)}}$ and $H^{(2)} := \set{h_{i, t}^{(2)}}$ from $\Acal$.
    For $i \in [R]$ and $j \in \set{1, 2}$, denote $H_{i}^{(j)} = \set{h_{i, t}^{(j)}}$ from which we can compute estimates $\hat{p}_{i}^{(1)}, \hat{p}_{i}^{(2)}$ of $p(x)$.
    Note that $H^{(j)} = \bigcup_{i = 1}^{R} H_{i}^{(j)}$.
    Denote the candidate hypotheses $\tilde{h}_{i}^{(1)}, \tilde{h}_{i}^{(2)}$.
    We will prove the following lemma.

    \begin{lemma}
        \label{lem:cand-hyp-close}    
        Suppose $\eta \leq \frac{\rho \gamma}{R}$, then with probability $1 - \rho$, the following holds for all $i$:
        \begin{equation*}
            \left|\errD\left(\tilde{h}_{i}^{(1)}\right) - \errD\left(\tilde{h}_{i}^{(2)}\right)\right| \leq \distD\left(h_{i}^{(1)}, h_{i}^{(2)}\right) \leq \gamma \text{.}
        \end{equation*}
    \end{lemma}

    \begin{proof}[Proof of \Cref{lem:cand-hyp-close}]
        We begin by arguing that for any fixed $i \in [R]$ and $x$, the probability that a randomly sampled $r$ falls between $\hat{p}_{i}^{(1)}(x), \hat{p}_{i}^{(2)}(x)$ is small.
        The following lemma follows from an identical argument as \Cref{lem:p-estimate} where we replace $\rho$ with $\gamma$.

        \begin{lemma}
            \label{lem:p-estimate-i}
            For any fixed $i \in [R], x \in \domain$ and $r \sim \UnifD[-1, 1]$,
            \begin{equation*}
                \Pr_{H^{(1)}, H^{(2)}, r} \left( \min(\hat{p}_{i}^{(1)}(x), \hat{p}_{i}^{(2)}(x)) \leq r \leq \max(\hat{p}_{i}^{(1)}(x), \hat{p}_{i}^{(2)}(x)) \right) < \frac{\eta}{10} \text{.}
            \end{equation*}
        \end{lemma}
    
        Now, we bound the probability that a randomly sampled $r$ falls between a large fraction of $\hat{p}_{i}^{(1)}(x), \hat{p}_{i}^{(2)}(x)$.
        We call randomly sampled $r$ that fall between a $\gamma$-fraction of estimates $\gamma$-noisy, which we formally define below.
    
        \begin{definition}
            \label{def:ambiguous-random-string}
            Fix $i \in [R]$. 
            A random string $r$ is $\gamma$-noisy \wrt distribution $\distribution$ and hypotheses $H_{i}^{(1)}, H_{i}^{(2)}$ if 
            \begin{equation*}
                \Pr_{x \sim \distribution_{\domain}} \left( \min(\hat{p}_{i}^{(1)}(x), \hat{p}_{i}^{(2)}(x)) \leq r \leq \max(\hat{p}_{i}^{(1)}(x), \hat{p}_{i}^{(2)}(x)) \right) > \gamma 
            \end{equation*}
            where $\hat{p}^{(j)}_{i}(x) = \E_{h \in H_{i}^{(j)}}[h(x)] = \frac{1}{T} \sum_{h \in H_{i}^{(j)}} h(x)$ for $j \in \set{1, 2}$ and $T = \left|H_{i}^{(1)}\right| = \left|H_{i}^{(2)}\right|$.
        \end{definition}
    
        Now, we bound the probability that a randomly sampled $r$ is $\gamma$-noisy.
    
        \begin{lemma}
            \label{lem:r-noisy-prob}
            Fix $i \in [R]$. 
            For a randomly drawn $r \sim \UnifD[0, 1]$ and hypotheses $H_{i}^{(1)}, H_{i}^{(2)}$ generated from \Cref{alg:non-uniform-replicability}, 
            \begin{equation*}
                \Pr_{H_{i}^{(1)}, H_{i}^{(2)}, r} \left( r \text{ is $\gamma$-noisy \wrt } \distribution, H_{i}^{(1)}, H_{i}^{(2)} \right) < \frac{\eta}{10 \gamma} \text{.}
            \end{equation*}
        \end{lemma}
    
        \begin{proof}
            For notational simplicity, we denote $H^{(j)} := H_{i}^{(j)}$ and $\hat{p}^{(j)} := \hat{p}_{i}^{(j)}$ for $j \in \set{1, 2}$.
            Let $E(r, x, H^{(1)}, H^{(2)})$ denote the event that $\min(\hat{p}^{(1)}(x), \hat{p}^{(2)}(x)) \leq r \leq \max(\hat{p}^{(1)}(x), \hat{p}^{(2)}(x))$.
            Then,
            \begin{align*}
                \E_{H^{(1)}, H^{(2)}, r} \left[ \Pr_{x \sim \distribution} \left( E(x, r, H^{(1)}, H^{(2)}) \right) \right] &=  \E_{H^{(1)}, H^{(2)}, r} \E_{x \sim \distribution} \left[ \ind_{E(x, r, H^{(1)}, H^{(2)})} \right] \\
                &= \E_{x \sim \distribution} \E_{H^{(1)}, H^{(2)}, r} \left[ \ind_{E(x, r, H^{(1)}, H^{(2)})} \right] \\
                &= \E_{x \sim \distribution} \left[ \Pr_{H^{(1)}, H^{(2)}, r} \left( \min(\hat{p}^{(1)}(x), \hat{p}^{(2)}(x)) \leq r \leq \max(\hat{p}^{(1)}(x), \hat{p}^{(2)}(x)) \right) \right] \\
                &\leq \frac{\eta}{10} 
            \end{align*}
            where we have exchanged the order of integration by observing that we are taking the integral of a bounded function over a finite measure and the final inequality follows from \Cref{lem:p-estimate}.
            
            Let $E(r, H^{(1)}, H^{(2)})$ denote the event that $r$ is $\gamma$-noisy \wrt $\distribution$ and $H^{(1)}, H^{(2)}$. 
            Then, applying Markov's inequality to the non-negative random variable $\Pr_{x \sim \distribution} \left( E(x, r, H^{(1)}, H^{(2)}) \right)$, we obtain
            \begin{align*}
                \Pr_{H^{(1)}, H^{(2)}, r}\left( E(r, H^{(1)}, H^{(2)}) \right) &= \Pr_{H^{(1)}, H^{(2)}, r} \left( \Pr_{x \sim \distribution} \left( E(x, r, H^{(1)}, H^{(2)}) \right) > \gamma \right) \\
                &< \frac{\eta}{10 \gamma} 
            \end{align*}
            as desired.
        \end{proof}
    
        Thus, setting $\eta \leq \frac{\rho \gamma}{R}$, we can apply a union bound and assume that (with probability $1 - \rho$), for all $i \in [R]$, $r_{i}$ is not $\gamma$-noisy \wrt $\distribution, H_{i}^{(1)}, H_{i}^{(2)}$.
    
        \begin{claim}
            \label{clm:no-r-noisy}
            If $\eta \leq \frac{\rho \gamma}{R}$, then with probability at least $1 - \rho$, for all $i \in [R]$, $r_{i}$ is not $\gamma$-noisy \wrt $\distribution, H_{i}^{(1)}, H_{i}^{(2)}$.
        \end{claim}
    
        We now claim that whenever $r_{i}$ is not $\gamma$-noisy \wrt $H_{i}^{(1)}, H_{i}^{(2)}$, the hypotheses $\tilde{h}_{i}$ obtained from $r_{i}$ and $H_{i}^{(1)}, H_{i}^{(2)}$ disagree on a small fraction of the distribution and have similar errors.
    
        \begin{claim}
            \label{clm:not-noisy-error-ub}
            Fix $i \in [R]$.
            Suppose $r_{i}$ is not $\gamma$-noisy \wrt $\distribution, H_{i}^{(1)}, H_{i}^{(2)}$.
            For $j \in \set{1, 2}$, let $\tilde{h}^{(j)}_{i}$ be the hypothesis in \Cref{alg:non-uniform-replicability} that maps $x$ to $1$ if $\hat{p}_{i}^{(j)}(x) \geq r_{i}$ and $-1$ otherwise.
            Then,
            \begin{equation*}
                \left|\errD\left(\tilde{h}_{i}^{(1)}\right) - \errD\left(\tilde{h}_{i}^{(2)}\right)\right| \leq \distD\left(h_{i}^{(1)}, h_{i}^{(2)}\right) \leq \gamma \text{.}
            \end{equation*}
            Furthermore, the first inequality holds for any two hypotheses.
        \end{claim}
    
        \begin{proof}
            We begin with the first inequality.
            We can rewrite for any hypothesis $h$,
            \begin{equation*}
                \errD(h) = \Pr_{(x, y) \sim \distribution}(h(x) \neq y) = \E_{(x, y) \sim \distribution}[\ind[h(x) \neq y]] \text{.}
            \end{equation*}
            Then, for any two hypotheses $h_1, h_2$,
            \begin{align*}
                |\errD(h_1) - \errD(h_2)| &= \left| \E_{(x, y) \sim \distribution}[\ind[h_{1}(x) \neq y]] - \E_{(x, y) \sim \distribution}[\ind[h_{2}(x) \neq y]] \right| \\
                &= \left| \E_{(x, y) \sim \distribution}[\ind[h_{1}(x) \neq y] - \ind[h_2(x) \neq y]] \right| \\
                &\leq \E_{(x, y) \sim \distribution}[\left| \ind[h_{1}(x) \neq y] - \ind[h_2(x) \neq y] \right|] \\
                &\leq \E_{(x, y) \sim \distribution}[\ind[h_{1}(x) \neq h_2(x)] - \ind[h_2(x) \neq y]] \\
                &= \distD(h_1(x), h_2(x)) \text{.}
            \end{align*}
            Above, we have used the identity $\left| \ind[h_{1}(x) \neq y] - \ind[h_2(x) \neq y] \right| \leq \ind[h_1(x) \neq h_2(x)]$.
    
            We proceed to the second inequality.
            Since we fix $i$, we omit $i$ as before for notational simplicity.
            Note that $\tilde{h}^{(j)}$ disagree on $x$ if and only if $\min(\hat{p}^{(1)}(x), \hat{p}^{(2)}(x)) \leq r \leq \max(\hat{p}^{(1)}(x), \hat{p}^{(2)}(x))$.
            Thus, since $r$ is not $\gamma$-noisy \wrt $H^{(1)}, H^{(2)}$, we can upper bound
            \begin{align*}
                \distD\left(h^{(1)}, h^{(2)}\right) &= \Pr_{x \sim \distribution}\left(h^{(1)}(x) \neq h^{(2)}(x)\right) \\
                &= \Pr_{x \sim \distribution}\left(\min\left(\hat{p}^{(1)}(x), \hat{p}^{(2)}(x)\right) \leq r \leq \max\left(\hat{p}^{(1)}(x), \hat{p}^{(2)}(x)\right)\right) \\
                &\leq \gamma \text{.} \qedhere
            \end{align*}
        \end{proof}

        We note that the above discussion concludes the proof of \Cref{lem:cand-hyp-close}.
    \end{proof}

    \paragraph{Step 2: (At least) one candidate hypothesis is accurate.}
    So far, we have obtained a collection of hypotheses $\set{h^{(i)}_{r}}$ such that over two runs of the algorithm, every pair of $h^{(1)}_{r}, h^{(2)}_{r}$ agree on a large fraction of the data distribution $\distribution_{\domain}$.
    We now prove that \whp at least one such hypothesis is $O(\alpha)$-accurate.
    In the following, we focus on only one run of the algorithm, and therefore omit the superscript for notational simplicity.
    We will prove the following lemma.

    \begin{lemma}
        \label{lem:no-r-accurate-ub}
        Let $\tilde{h}_{1}, \dots, \tilde{h}_{R}$ denote the candidate hypotheses generated by \Cref{alg:non-uniform-replicability}.
        Then, with probability $1 - \beta$, $\min_{i = 1}^{R} \errD(\tilde{h}_{i}) < \opt + 2 \alpha$.
    \end{lemma}

    \begin{proof}[Proof of \Cref{lem:no-r-accurate-ub}]
        First, we place an upper bound on the expected error of our candidate hypotheses.

        \begin{lemma}
            \label{lem:exp-error-ub}
            Fix $i \in [R]$.
            Then, $\E_{\tilde{h}_{i}}\left[ \errD(\tilde{h}_{i})\right] = \E_{h \sim \Acal}[\errD(h)] \leq \opt + \alpha + \beta$.
        \end{lemma}
    
        \begin{proof}
            As before, since we fix $i$, we omit the subscript for notational simplicity.
            The lemma will follow from the fact that our algorithm (which produces candidate hypotheses $\tilde{h}$) is an unbiased estimator. 
            In particular, over the randomness of $\tilde{h}$, we have
            \begin{align*}
                \E_{\tilde{h}}\left[\errD(\tilde{h})\right] &= \E_{\tilde{h}, (x, y) \sim \distribution} \left[\ind[\tilde{h}(x) \neq y] \right]\\
                &= \E_{(x, y) \sim \distribution} \left[ \E_{\tilde{h}} \left[\ind[\tilde{h}(x) \neq y] \right] \right] \\
                &= \E_{(x, y) \sim \distribution} \left[ \Pr_{\tilde{h}} \left( \tilde{h}(x) \neq y \right) \right] \text{.}
            \end{align*}
            As before, we have exchanged the order of integration by Fubini's theorem.

            From our above calculation, it is enough to observe that for any fixed $x$ we have
            \begin{align*}
                \Pr_{\tilde{h}}\left( \tilde{h}(x) = 1 \right) = \Pr_{h \sim \Acal}(h(x) = 1) \text{.}
            \end{align*}
            We argue the above equality.
            Observe that $\Pr_{H, r}(\tilde{h}(x) = 1) = \E_{H} [\Pr_{r}(\tilde{h}(x) = 1)]$ where $r$ are sampled randomness and $H = \set{h_{i, t}}$ the sampled hypotheses.
            From the proof of \Cref{lem:exp-error-ub-pred}, we have
            \begin{equation*}
                \E_{H}[\Pr_{r}(\tilde{h}(x) = 1)] = \E_{H}[\hat{p}(x)] = \frac{1}{T} \sum_{t = 1}^{T} \Pr_{H}(h_{i, t}(x) = 1) = \Pr_{h \sim \Acal}(h(x) = 1) \text{.}
            \end{equation*}
            
            In particular, we have $\E_{H, r}[\errD(\tilde{h})] = \E_{h \sim \Acal}[\errD(h)] \leq \opt + \alpha + \beta$. 
        \end{proof}
    
        Now, we bound the probability that no hypothesis is $O(\alpha)$-accurate.
        We will use the following fact, whose proof we defer.
    
        \begin{lemma}
            \label{lem:bounded-markov}
            Let $X$ be a random variable with support in $[0, 1]$ and $\mu := \E[X]$.
            Then, for any $t \leq 1 - \mu$,
            \begin{equation*}
                \Pr(X < \mu + t) \geq t \text{.}
            \end{equation*}
        \end{lemma}

        Note that $\errD(\tilde{h}_{i})$ is a random variable with support in $[0, 1]$.
        Thus, for any fixed $i \in [R]$, \Cref{lem:exp-error-ub} and \Cref{lem:bounded-markov} with $t = \alpha$ implies
        \begin{equation*}
            \Pr\left( \errD(\tilde{h}_{i}) < \opt + 2 \alpha + \beta \right) \geq \alpha \text{.}
        \end{equation*}

        Finally, since all $\tilde{h}_{i}$ are generated independently (note that $H_{i}$ are independent training samples and $r_{i}$ are independent random strings), we conclude the proof by observing
        \begin{equation*}
            \Pr \left( \min_{i = 1}^{R} \errD(\tilde{h}_{i}) \geq \opt + 2 \alpha \right) < (1 - \alpha)^{R} < e^{- \alpha R} < \beta
        \end{equation*}
        as long as $R \geq \frac{C \log(1/\beta)}{\alpha}$ for sufficiently large constant $C$.
    \end{proof}

    \paragraph{Step 3: Replicable hypothesis selection.}
    Finally, we show that we can replicably select a near-optimal candidate hypothesis, thus obtaining a PAC learner that is both pointwise and approximately replicable.
    To do so, we directly apply \Cref{thm:optimal-hypothesis-sel}.
    Recall that we set $\gamma \leq \frac{\rho \alpha}{C \log(R/\min(\rho, \beta))}$ for a sufficiently large constant $\alpha$ so that \Cref{thm:optimal-hypothesis-sel} (along with \Cref{lem:cand-hyp-close}) implies that with probability $1 - \beta$ we $\rho$-replicably select an index $\hat{i} \in [R]$ such that 
    \begin{equation*}
        \errD(\tilde{h}_{\hat{i}}) \leq \min_{i = 1}^{R} \errD(\tilde{h}_{i}) + \alpha \leq \opt + 3 \alpha + \beta \text{.}
    \end{equation*}
    
    A union bound concludes that our algorithm is a $(3 \alpha + \beta, 2 \beta)$-learner.
    We now argue that our algorithm produces both a pointwise and approximately replicable learner. 
    The latter immediately follows by union bounding over the error of \Cref{lem:cand-hyp-close} and \Cref{thm:optimal-hypothesis-sel} so that we obtain a $(2\rho, \gamma)$-approximately replicable learner.
    The former follows from the fact that each $\tilde{h}_{i}$ is already an $\eta$-pointwise replicable learner from \Cref{lem:p-estimate}.
    By our choices of parameters for $\eta, \gamma, R$, observe that
    \begin{equation*}
        \eta \leq \frac{10 \rho \gamma}{R} \leq \frac{10 \rho^2 \alpha}{C R \log(R/\beta)} \leq \frac{10 \rho^2 \log(1/ \beta)}{C^2 \log(R/\min(\rho, \beta))} \leq \rho
    \end{equation*}
    for large enough constant $C$.
    In particular, by a union bound, whenever we replicably select an index $\hat{i} \in [R]$, we obtain a $2 \rho$-pointwise replicable learner.

    \paragraph{Sample complexity and runtime.}
    We now bound the sample complexity.
    Our algorithm requires samples to run $\Acal$ on $RT$ fresh samples of size $m(\alpha, \beta)$ and $\bigO{\frac{\log^2(R/\beta)}{\gamma^2}}$ samples to run \Cref{thm:optimal-hypothesis-sel}.
    Bounding the above two terms we obtain a sample complexity of
    \begin{align*}
        RTm(\alpha, \beta) &= \bigO{\frac{R m(\alpha, \beta)}{\eta^2}} \\
        &= \bigO{\frac{R^{3} m(\alpha, \beta)}{\rho^2 \gamma^2}} \\
        &= \bigO{m(\alpha, \beta) \left( \frac{R^{3}}{\rho^2 \gamma^2} + \frac{R^{3} \log(R/\min(\rho, \beta))}{\rho^{4} \alpha^2} \right)} \\
        &= \bigO{m(\alpha, \beta) \left( \frac{\log^{3}(1/\beta)}{\rho^2 \gamma^2 \alpha^3} + \frac{\log^{3}(1/\beta) \log(1/(\alpha\min(\rho, \beta)))}{\rho^{4} \alpha^{5}} \right)} \\
        &= \bigtO{m(\alpha, \beta) \left( \frac{\log^{3}(1/\beta)}{\rho^2 \gamma^2 \alpha^3} + \frac{\log^{4}(1/\beta)}{\rho^{4} \alpha^{5}} \right)} \text{.}
    \end{align*}
    Above, we have iteratively applied our setting of $T, \eta, \gamma, R$.
    Finally, since $m(\alpha, \beta) = \bigO{\frac{d + \log(1/\beta)}{{\alpha^2}}}$, we conclude that the overall sample complexity is
    \begin{equation*}
        \bigtO{\frac{d \log^{3}(1/\beta)}{\rho^2 \gamma^2 \alpha^{5}} + \frac{d \log^{4}(1/\beta)}{\rho^{4} \alpha^{7}} + \frac{\log^{4}(1/\beta)}{\rho^2 \gamma^2 \alpha^{5}} + \frac{\log^{5}(1/\beta)}{\rho^{4} \alpha^{7}}} = \bigtO{\frac{d \log^{4}(1/\beta)}{\rho^2 \gamma^2 \alpha^{5}} + \frac{d \log^{5}(1/\beta)}{\rho^{4} \alpha^{7}}} \text{.}
    \end{equation*}
    By inspecting our algorithm, we note that all steps run in linear time in sample complexity, except \Cref{thm:optimal-hypothesis-sel}, which requires time $\bigO{R\log^2(R/\beta)\gamma^{-2}}$.Note that this remains within our sample complexity, proving the desired runtime.
\end{proof}

To conclude, we prove the necessary lemmas.

\begin{proof}[Proof of \Cref{lem:bounded-markov}]
    Since $X \geq 0$ is non-negative, we have
    \begin{align*}
        \mu &= \int_{0}^{\infty} \Pr(X \geq x) \dx \\
        &= \int_{0}^{\mu + t} \Pr(X \geq x) \dx + \int_{\mu + t}^{\infty} \Pr(X \geq x) \dx \\
        &\geq \int_{0}^{\mu + t} \Pr(X \geq \mu + t) \dx \\
        &= (\mu + t) \Pr(X \geq (\mu + t)) 
    \end{align*}
    where the inequality follows from $\Pr(X \geq x) \geq \Pr(X \geq y)$ for $x \leq y$.
    Rearranging, we obtain $\Pr(X \geq \mu + t) \leq \frac{\mu}{\mu + t}$.
    Finally, we can conclude
    \begin{equation*}
        \Pr(X < \mu + t) = 1 - \Pr(X \geq \mu + t) \geq 1 - \frac{\mu}{\mu + t} = \frac{t}{\mu + t} \geq t 
    \end{equation*}
    where the final inequality follows from $\mu + t \leq 1$.
\end{proof}

\subsection{Hypothesis Selection}
\label{sec:hyp-sel}

We prove \Cref{thm:optimal-hypothesis-sel}.
Our algorithm closely follows \cite{GhaziKM21, hopkins2025generative}, using correlated sampling to sample from a distribution over the hypotheses given by the exponential mechanism from privacy \cite{McTalwar}.

\begin{proof}
    Consider \Cref{alg:robust-repl-hyp-sel}.

    \IncMargin{1em}
    \begin{algorithm}
    
    \SetKwInOut{Input}{Input}\SetKwInOut{Output}{Output}\SetKwInOut{Parameters}{Parameters}
    \Input{Hypotheses $\set{f_{i}}_{i = 1}^{n}$ and sample access to $\distribution$.}
    \Parameters{$\alpha$ accuracy, $\beta$ error probability, and $(\rho, \tau)$ robust replicability}
    \Output{$\alpha$-optimal hypothesis $f_{i}$.}
    Set $\expMechScoreScale \gets \frac{2 \log(2n/\beta)}{\alpha}$ and $C$ a sufficiently large constant.
    
    Draw $m \gets \frac{C}{\tau^2} \log \left( \frac{n}{\beta} \right)$ \iid samples from $\distribution$, denoted $S \gets ((x_{\ell}, y_{\ell}))_{\ell = 1}^{m}$.
    
    For each $i \in [n]$, empirically estimate $\err_{S}(f_{i}) \gets \frac{1}{m} \sum_{\ell = 1}^{m} \ind (f_{i}(x_{\ell}) \neq y_{\ell})$.
    
    Define the distribution $\hat{\distribution}$ on $[n]$ using the exponential mechanism, i.e. $i$ is drawn with probability proportional to $\exp(- t \cdot \err_{S}(f_{i})))$.
    
    \Return $\corrSamp(\hat{\distribution}; r)$ with shared internal randomness $r$.
    
    \caption{$\hypothesisSelection(\set{f_{i}}_{i = 1}^{n}, \alpha, \beta, \rho, \tau)$}
    \label{alg:robust-repl-hyp-sel}
    
    \end{algorithm}
    \DecMargin{1em}
    
    We separately argue the correctness and robust replicability of the algorithm.
    Condition on the event
    \begin{equation*}
        \left|\err_{S}(f_{i}) - \errD(f_{i})\right| \ll \frac{\tau}{\log(n/\beta)}
    \end{equation*}
    for all $i$.
    By our choice of $m = \bigO{\tau^{-2} \log(n/\beta)}$, this happens with probability at least $1 - \frac{\beta}{2}$ by a Hoeffding and union bound.
    
    We begin by arguing correctness.
    Denote by $i^* \in [n]$ the index of an optimal hypothesis i.e., $i^* = \arg \min_{i} \errD(f_{i})$.
    Then, the probability that a hypothesis $f_{i}$ is sampled is at most $\exp(- \expMechScoreScale (\err_{S}(f_{i}) - \err_{S}(f_{i^*}))$.
    If $\errD(f_{i}) > \errD(f_{i^*}) + \alpha$ then by the above conditioning, $\err_{S}(f_{i}) \geq \errD(f_{i}) - \frac{\tau}{4} > \errD(f_{i^*}) + \alpha - \frac{\tau}{4} \geq \err_{S}(f_{i^*}) + \alpha - \frac{\tau}{2}$.
    Thus, by our assumption $\tau < \alpha$, $\err_{S}(f_{i}) \geq \err_{S}(f_{i^*}) + \frac{\alpha}{2}$.
    Finally the probability any such $i$ is output is at most
    \begin{equation*}
        \exp\left(- \expMechScoreScale \alpha / 2 \right) \leq \frac{\beta}{2n}
    \end{equation*}
    so that by a union bound, the probability that any hypothesis with $\errD(f_{i}) > \errD(f_{i^*}) + \alpha$ is chosen is at most $\frac{\beta}{2}$.
    Combined with the conditioned event, the probability of an error is at most $\beta$. 

    Next, we analyze replicability.
    Let $S_1, S_2$ denote the samples observed by two independent runs of the algorithm $\innerAlg$ and $r$ the shared internal randomness.
    Let $P(i) = \Pr(\innerAlg(S_1; r) = i)$ and $Q(i) = \Pr(\innerAlg(S_2; r) = i)$.
    Then by the definition of the exponential mechanism
    \begin{align*}
        P(i) &= \frac{\exp(- \expMechScoreScale \cdot \err_{S_{1}}(f_{i}))}{\sum_{j} \exp(- \expMechScoreScale \cdot \err_{S_{1}}(f_{j}))} \\
        Q(i) &= \frac{\exp(- \expMechScoreScale \cdot \err_{S_{2}}(g_{i}))}{\sum_{j} \exp(- \expMechScoreScale \cdot \err_{S_{2}}(g_{j}))}
    \end{align*}

    We bound the total variation distance between $P$ and $Q$.
    Fix some $i \in [n]$.
    Then,
    \begin{align*}
        P(i) - Q(i) &= \frac{\exp(- \expMechScoreScale \cdot \err_{S_{1}}(f_{i}))}{\sum_{j} \exp(- \expMechScoreScale \cdot \err_{S_{1}}(f_{j}))} - \frac{\exp(- \expMechScoreScale \cdot \err_{S_{2}}(g_{i}))}{\sum_{j} \exp(- \expMechScoreScale \cdot \err_{S_{2}}(g_{j}))} \\
        &= \left( \frac{\exp(- \expMechScoreScale \cdot \err_{S_{2}}(g_{i}))}{\sum_{j} \exp(- \expMechScoreScale \cdot \err_{S_{2}}(g_{j}))} \left( \frac{\exp(- \expMechScoreScale \cdot \err_{S_{1}}(f_{i}))}{\exp(- \expMechScoreScale \cdot \err_{S_{2}}(g_{i}))} \frac{\sum_{j} \exp(- \expMechScoreScale \cdot \err_{S_{2}}(g_{j}))}{\sum_{j} \exp(- \expMechScoreScale \cdot \err_{S_{1}}(f_{j}))} - 1 \right) \right) .
    \end{align*}
    Combining our conditioned event and the triangle inequality, we bound
    \begin{align*}
        \left| \err_{S_{1}}(f_{i}) - \err_{S_{2}}(g_{i}) \right| &\leq \left| \err_{S_{1}}(f_{i}) - \err_{\distribution}(f_{i}) \right| + \left| \err_{\distribution}(f_{i}) - \err_{\distribution}(g_{i}) \right| + \left| \err_{\distribution}(g_{i}) - \err_{S_{2}}(g_{i}) \right|\\
        &\leq 2 \tau + \left| \errD(f_{i}) - \errD(g_{i}) \right| \\
        &\leq 2 \tau + \tau \\
        &\leq 3 \tau \text{.}
    \end{align*}
    In particular, the total variation distance can be bounded by
    \begin{align*}
        \frac{1}{2} \sum_{i} |P(i) - Q(i)| &= \frac{1}{2} \sum_{i} Q(i) \left( e^{4 t (3 \tau)} - 1 \right) = O(t \tau) = \bigO{\frac{\log(n/\beta) \tau}{\alpha}}
    \end{align*}
    where we use that $e^{x} < 1 + O(x)$ for small $x > 0$.
    Since we have assumed that $\tau \ll \rho \alpha/\log(n/\beta)$, the total variation distance is bounded by $\rho$.
    In particular, by Correlated Sampling (see \Cref{lemma:correlated-sampling}) we can ensure that the outputs of the algorithms differ by at most probability $O(\rho)$.
    Combined with a union bound on the conditioned events, the algorithms fail to be replicable with probability at most $O(\rho + \beta)$.
\end{proof}

\subsection{Approximately Replicable Learning for Thresholds}

In this section, we give a proper approximately replicable learning algorithm for thresholds with near-optimal sample complexity by combining our robust hypothesis selection method from the previous section with a basic high probability quantile estimation.

\begin{definition}[Thresholds]
    \label{def:1-d-thresholds}
    The class of thresholds consist of all functions $\hypotheses := \set{f_{t}: x \mapsto \ind[x > t] \mid t \in \R}$.
\end{definition}

Our learner satisfies the following formal guarantees.

\begin{proposition}[Formal \Cref{prop:thresholds-proper-informal}]
    \label{prop:thresholds-proper-formal}
    Let $\rho, \gamma, \alpha > 0$ and $0 < \beta < \rho$.
    There is an agnostic proper $(\rho, \gamma)$-approximately replicable $(\alpha, \beta)$-learner for thresholds with sample complexity 
    \begin{equation*}
        \bigO{\frac{\log(1/\beta)}{\gamma^2} + \frac{\log^3(1/\alpha\beta)}{\rho^2 \alpha^2}} \text{.}
    \end{equation*}
\end{proposition}

In \Cref{sec:apx-lower} we show this algorithm is optimal (up to log factors) when $\rho = \Theta(\gamma)$.

Our algorithm proceeds in two stages: first we estimate quantiles of the marginal distribution $\distribution_{\domain}$. Then, given that the quantiles give a sufficiently granular splitting of the marginal distribution (and they are learned up to sufficient accuracy), we use robustly replicable hypothesis selection to select an $\alpha$-optimal quantile.
In slightly more detail, if we learn the marginal $\distribution_{\domain}$'s $\alpha$-quantiles up to $\gamma \ll \rho \alpha$ accuracy, we are guaranteed (1) that at least one quantile is $\alpha$-optimal, and (2) every quantile is $\gamma$-close in classification distance across multiple runs of the algorithm.
Therefore, robustly replicable hypothesis selection (\Cref{thm:optimal-hypothesis-sel}) replicably outputs an $O(\alpha)$-optimal hypothesis.

\begin{proof}
    We now formally describe our algorithm. 

    \IncMargin{1em}
    \begin{algorithm}
    
    \SetKwInOut{Input}{Input}\SetKwInOut{Output}{Output}\SetKwInOut{Parameters}{Parameters}
    \Input{Sample access to distribution $\distribution$ over $\R \times \set{0, 1}$.}
    \Parameters{$(\rho, \gamma)$ replicability parameters and $(\alpha, \beta)$-accuracy parameters.}
    \Output{$(\rho, \gamma)$-approximately replicable $(\alpha, \beta)$-accurate agnostic PAC leaner.}

    Fix $K \gets \frac{3}{\alpha}$ and $\tau \ll \min\left(\gamma, \frac{\rho \alpha}{\log(K/\beta)} \right)$ for some sufficiently small constant.
    
    Draw $m \gets \bigO{\frac{\log(1/\beta)}{\tau^2}}$ samples from $\distribution$, denoted $S = (x_1, \dotsc, x_{m})$.
    
    Sort $S$ (with ties broken arbitrarily).
    
    \For{$0 \leq i \leq K$}{
        Set $\hat{t}_{i} \gets x_{\frac{i m}{K}}$ and $h_{i}: x \mapsto \ind[x \geq t_{i}]$.
    }

    Set $h_{-1}: x \mapsto 0$ and $h_{K + 1}: x \mapsto 1$.

    \Return $\hypothesisSelection\left( \set{h_{i}}_{i = -1}^{K+ 1}, \frac{\alpha}{2}, \frac{\beta}{2}, \frac{\rho}{3}, \tau \right)$.
    
    \caption{$\thresholdLearner(\alpha, \beta, \rho, \gamma)$}
    \label{alg:threshold-alg}
    
    \end{algorithm}
    \DecMargin{1em}

    Let $\tau, K > 0$ to be parameters to be fixed later.
    We will assume that $K$ is a multiple of $\frac{1}{\alpha}$.
    For simplicity, we may assume $m$ is a multiple of $\frac{1}{\alpha}$ (since we will have $m > \frac{1}{\alpha}$, this comes at most a constant factor increase in sample complexity).
    Let $F(x)$ denote the cumulative distribution function of the marginal distribution $\distribution_{\domain}$, i.e. $F(x) := \Pr_{X \sim \distribution_{\domain}}(X \leq x)$.
    For all $0 \leq i \leq K$, let $t_{i} := F(i/K)$ denote the $i/K$-th quantile of $\distribution_{\domain}$.
    Since $\hat{t}_{i}$ are the empirical $i/K$-th quantiles of our sample $S$ drawn \iid from $\distribution_{\domain}$, we use the Dvoretsky-Kiefer-Wolfowitz (DKW) Inequality to bound the accuracy of our estimates.

    \begin{theorem}[Dvoretsky-Kiefer-Wolfowitz Inequality \cite{dvoretzky1956asymptotic, massart1990tight}]
        \label{thm:dkw}
        Let $\distribution$ be a distribution on $\R$ with cumulative distribution function $F$.
        Let $X_{1}, \dotsc, X_{m}$ be drawn \iid from $\distribution$ and $F_{m}(x) := \frac{1}{m} \sum_{i = 1}^{m} \ind[X_{i} \leq x]$.
        Then
        \begin{equation*}
            \Pr \left( \sup_{x \in \R} |F_{n}(x) - F(x)| > \tau \right) < 2 e^{-m \tau^2} \text{.}
        \end{equation*}
    \end{theorem}

    In particular, by our choice of $m$, we have that with probability $1 - \beta$, 
    the estimated quantiles $\hat{t}_{i}$ satisfy
    \begin{equation}
        \label{eq:threshold-quantile-est}
        \Pr_{z \sim \distribution} \left( z \in (\min(t_{i}, \hat{t}_{i}) , \max(t_{i}, \hat{t}_{i})) \right) < \frac{\tau}{20} \text{.}
    \end{equation}

    We claim that our algorithm is correct.
    Let $h^*: x \mapsto \ind[x \geq t^*]$ be the optimal hypothesis.
    We claim that there exists a hypothesis $h_{i}$ such that $\Pr_{z}(h_{i}(z) \neq h^*(z)) \leq \frac{1}{K} + 0.1 \tau$.
    By construction (if we set $\hat{t}_{-1} \gets - \infty$ and $\hat{t}_{K + 1} \gets +\infty$) there exists some $-1 \leq i \leq K$ for which $\hat{t}_{i} \leq t^* \leq \hat{t}_{i + 1}$.
    If $\hat{t}_{i} = \hat{t}_{i + 1}$, then $h^*$ is in fact included in our hypothesis set.
    Otherwise, $\hat{t}_{i} < \hat{t}_{i + 1}$ and we have
    \begin{align*}
        \Pr_{z}(h_{i}(z) \neq h^*(z)) &\leq \Pr_{z}(h_{i}(z) \neq h_{i + 1}(z)) \\
        &\leq \Pr_{z}(\hat{t}_{i} < z < \hat{t}_{i + 1}) \\
        &\leq \frac{1}{K} + \frac{\tau}{10}
    \end{align*}
    where the final inequality follows from the definition that $\hat{t}_{i}, \hat{t}_{i + 1}$ are separated by $\frac{m}{K}$ samples and \eqref{eq:threshold-quantile-est}.
    Thus, setting $K = \frac{3}{\alpha}$ and $\tau < \alpha$, there exists a hypothesis $h_{i}$ that is within $\frac{\alpha}{2}$ of $h^*$ i.e. we ensure that we obtain a $\frac{\alpha}{2}$-optimal hypothesis among the set of candidate hypotheses.
    In particular, applying \Cref{thm:optimal-hypothesis-sel}, we output a hypothesis $h_{j}$ with 
    \begin{equation*}
        \errD(h_{j}) \leq \errD(h_{i}) + \frac{\alpha}{2} \leq \opt + \frac{\alpha}{2} \text{.}
    \end{equation*}

    Now, we argue that our algorithm is approximately replicable. 
    Consider two runs of the algorithm, and let $\set{\hat{t}_{i}^{(1)}}, \set{\hat{t}_{i}^{(2)}}$ denote the estimated quantiles in each run of the algorithm, and $\set{h_{i}^{(1)}}, \set{h_{i}^{(2)}}$ the corresponding hypotheses.
    Conditioned on \eqref{eq:threshold-quantile-est}, we have that with probability at least $1 - \beta$, the following holds for all $i$:
    \begin{align*}
        \Pr_{z \sim \distribution} \left( h_{i}^{(1)}(z) \neq h_{i}^{(2)}(z) \right) &= \Pr_{z} \left( z \in (\min(\hat{t}_{i}^{(1)}, \hat{t}_{i}^{(2)}), \max(\hat{t}_{i}^{(1)}, \hat{t}_{i}^{(2)}))\right) \\
        &= \int \left| \ind[x \leq \hat{t}_{i}^{(1)}] - \ind[x \leq \hat{t}_{i}^{(2)}] \right| \dx \\
        &\leq \int \left| \ind[x \leq \hat{t}_{i}^{(1)}] - \ind[x \leq t_{i}] \right| + \left| \ind[x \leq t_{i}] - \ind[x \leq \hat{t}_{i}^{(2)}] \right| \dx \\
        &= \sum_{j = 1}^{2} \Pr_{z} \left( z \in (\min(t_{i}, \hat{t}_{i}^{(j)}), \max(t_{i}, \hat{t}_{i}^{(j)}))\right) \\
        &< \tau \text{.}
    \end{align*}
    In particular, we satisfy the assumption of \Cref{thm:optimal-hypothesis-sel} that $\distD(h_{i}^{(1)}, h_{i}^{(2)}) < \tau$ for all $i$.
    Since $\hypothesisSelection$ is $\left(\frac{\rho}{2}, \tau\right)$-robustly replicable, we obtain via a union bound that the same index is selected with probability at least $1 - \beta - \frac{\rho}{2} \geq 1 - \rho$.
    The proof concludes by setting $\tau \leq \gamma$.

    To bound the sample complexity, note that quantile estimation requires $\tau < \min(\alpha, \gamma)$ and applying \Cref{thm:optimal-hypothesis-sel} requires $\tau \ll \frac{\rho \alpha}{\log(K/\beta)}$.
    Thus, the final sample complexity is 
    \begin{equation*}
        \bigO{\frac{\log(1/\beta) \log^2(1/\alpha \beta)}{\rho^2 \alpha^2} + \frac{\log(1/\beta)}{\gamma^2} + \frac{\log^3(1/\alpha\beta)}{\rho^2 \alpha^2}} = \bigO{\frac{\log(1/\beta)}{\gamma^2} + \frac{\log^3(1/\alpha\beta)}{\rho^2 \alpha^2}} \text{.}
    \end{equation*}
\end{proof}

\subsection{Approximately Replicable Realizable Learning}

In this section, we give a simple algorithm with improved sample complexity in the `realizable' PAC learning setting, giving the second half of \Cref{thm:apx-repl-informal}.
Here correctness is only required when the distribution $\distribution$ is realizable, but we still ensure approximate replicability over \textit{all} input distributions. 

\begin{restatable}{theorem}{RealizableApxRepl}
    \label{thm:realizable-apx-repl}
    Let $\innerAlg$ be a realizable $(\alpha, \beta)$-learner with sample complexity $m(\alpha, \beta)$.
    There is a $(\rho, \gamma)$-approximately replicable realizable $(\alpha, \beta)$-learner with sample complexity 
    \begin{equation*}
        \bigO{\frac{d + \log(1/\min(\rho, \beta))}{\rho^2 \min(\alpha, \gamma)}} \text{.}
    \end{equation*}
\end{restatable}

\begin{proof}
    Our algorithm proceeds in two steps.
    First, we replicably estimate $\opt$ to determine whether $\distribution$ is realizable.
    If not, we can replicably output an arbitrary hypothesis (since we have no correctness requirement in this case).
    Otherwise, since every accurate solution is close to the optimal hypothesis, it suffices to run a (non-replicable) realizable PAC learner.
    We formalize this in \Cref{alg:apx-repl-realizable}.

    \IncMargin{1em}
        \begin{algorithm}
        
        \SetKwInOut{Input}{Input}\SetKwInOut{Output}{Output}\SetKwInOut{Parameters}{Parameters}
        \Input{Realizable PAC learner $\innerAlg$. Sample access to $\distribution$.}
        \Parameters{$(\rho, \gamma)$-approximate replicability, $\alpha$ accuracy, $\beta$ error.}
        \Output{$(\rho, \gamma)$-approximately replicable realizable $(\alpha, \beta)$-learner.}

        {\bf Step 1: Replicably Estimate $\opt$.}

        Collect $m_{1} = \bigO{\frac{d + \log(1/\min(\rho, \beta))}{\rho^2 \alpha}}$ samples from $\distribution$, denoted $S$.

        Draw a random threshold $r \in [0.1\alpha, 0.2\alpha]$.

        Compute $\hat{\opt} := \min_{h \in \hypotheses} \errS(h)$.

        \If{$\hat{\opt} > r$}{
            \Return $h: x \rightarrow 1$.
        }
        
        {\bf Step 2: Running a Realizable Learner.}

        \Return $\hat{h} \gets \arg\min_{h \in \hypotheses} \errS(h)$.
        
        \caption{$\apxReplRealizable(\innerAlg)$}
        \label{alg:apx-repl-realizable}
        
        \end{algorithm}
        \DecMargin{1em}

        We prove the correctness and replicability of our algorithm.
        We begin with the following lemma, which states that Step 1 replicably decides whether the distribution is realizable.
        
        \begin{lemma}
            \label{lem:repl-reject}
            At the end of Step 1, the following hold:
            \begin{enumerate}
                \item If $\opt = 0$, then $\hat{\opt} < r$ with probability $1 - \beta$.
                \item If $\opt > \alpha$, then $\hat{\opt} \geq r$ with probability $1 - \beta$.
                \item The output is $11 \rho$-replicable.
            \end{enumerate}
        \end{lemma}

        \begin{proof}
            The argument is similar to the general sample complexity for realizable learning with ERM.
            In this argument, we require a strengthening of the guarantees of realizable PAC learning.
            
            Let $S$ consist of a dataset of $m$ \iid samples from $\distribution$.
            \begin{align*}
                E(h) &:= \left((\errD(h) < \alpha) \wedge |\errS(h) - \errD(h)| > \rho \alpha\right) \vee \left(\errD(h) \geq \alpha \wedge \errS(h) < \frac{\alpha}{3} \right) \\
                B(S) &:= \set{\exists h \given E(h) } \text{.}
            \end{align*}
            In particular, we ensure that every hypotheses either has large empirical error (and $\hat{\opt} \geq r$) or any near-optimal hypothesis is estimated up to error $\rho \alpha$.
            Let $S'$ consist of $m$ \iid samples from $\distribution$ (independent of $S$) and $\sigma \in \set{\pm 1}^{m}$ an independent random vector.
            Note that to prove the lemma, it suffices to show $\Pr(B(S)) < \min(\rho, \beta)$.
            In particular, if $\opt = 0$, we have $\hat{\opt} < \rho \alpha < 0.1 \alpha < r$.
            If $\opt > \alpha$, we have $\hat{\opt} > \frac{\alpha}{3} > 0.2 \alpha > r$.
            Finally, $\hat{\opt}$ (if it lies in $[0.1 \alpha, 0.2 \alpha]$, lies in an interval of width $\rho \alpha$, so that $r$ falls in this interval with probability at most $10 \rho$.
            We union bound over the probability of $B(S)$.

            We now bound the probability of $B(S)$.
            Given $S, S', \sigma$, define $T$ to consist of an example from $S$ when $\sigma = +1$ and $S'$ when $\sigma = -1$.
            Define the following events:
            \begin{align*}
                E'(S, S', h) &:= \left(\left(|\errS(h) - \err_{S'}(h)| > \frac{\rho \alpha}{2}\right) \wedge (\errS(h) < \alpha)\right) \\
                &\quad \vee \left(\left(\err_{S'}(h) > \frac{\alpha}{2}\right) \wedge \left(\err_{S}(h) < \frac{\alpha}{3}\right)\right) \\ 
                B'(S, S') &:= \set{\exists h \given E'(h)} \\
                B''(S, S', \sigma) &:= \set{\exists h \given E'(T, T', h) \text{ where $T, T'$ correspond to $S, S', \sigma$}} \text{.}
            \end{align*}

            \begin{claim}
                \label{clm:}
                When $m \gg \frac{1}{\rho^2 \alpha}$, for some sufficiently large constant, $\Pr(B'(S, S') | B(S)) > \frac{1}{2}$.
            \end{claim}

            \begin{proof}
                Suppose $B(S)$ holds i.e. $E(h)$ holds for some $h$.
                We split into two cases.
                Suppose $\errD(h) \leq \alpha$ and $|\errS(h) - \errD(h)| > \rho \alpha$.
                Then, $\E[\err_{S'}(h)] = \errD(h) \leq \alpha$.
                By a standard Chernoff bound, 
                \begin{equation*}
                    Pr\left(|\err_{S'}(h) - \errD(h)| > \frac{\rho \alpha}{2} \right) < 2\exp\left(-\frac{\rho^2 \alpha m}{12}\right) < \frac{1}{2} \text{.}
                \end{equation*}
                By the triangle inequality, we have
                \begin{equation*}
                    |\errS(h) - \err_{S'}(h)| > \frac{\rho \alpha}{2} \text{.}
                \end{equation*}
                On the other hand, suppose $\errD(h) > \alpha$ and $\errS(h) < \frac{\alpha}{3}$.
                Following an identical argument, we have $\E[\err_{S'}(h)] > \alpha$ and
                \begin{equation*}
                    \Pr\left( \err_{S'}(h) < \frac{\alpha}{2} \right) < \exp \left( - \frac{m \alpha}{12} \right) < \frac{1}{2} \text{.}
                \end{equation*}
            \end{proof}

            Thus, we have
            \begin{equation*}
                \Pr(B(S)) = \frac{\Pr(B'(S, S') \cap B(S))}{\Pr(B'(S, S') | B(S))} \leq \frac{\Pr(B'(S, S'))}{\Pr(B'(S, S') | B(S))} \leq 2 \Pr(B'(S, S')) \text{.}
            \end{equation*}

            Observe that since $S, S'$ and $T, T'$ are identically distributed, $\Pr(B'(S, S')) = \Pr(B''(S, S', \sigma))$.
            We thus hope to bound $\Pr(B''(S, S', \sigma))$.
            It suffices to bound the probability of $B''(S, S', \sigma)$ conditioned on any fixed $S, S'$.

            \begin{claim}
                For any fixed $S, S', h$, 
                \begin{equation*}
                    \Pr_{\sigma}\left( E'(T, T', h) | S, S' \right) < \exp\left( - \bigOm{ \rho^2 \alpha m} \right) \text{.}
                \end{equation*}
            \end{claim}

            \begin{proof}
                Consider the predictions on $S, S'$, denoted $\set{h(x_{i})}$ and $\set{h(x_{i}')}$.
                Note that $\sigma$ independently distributes each sample to $T, T'$ uniformly.
                Consider the set of indices where exactly one of $h(x_{i}), h(x_{i}')$ is correct.
                We may assume this set has size $k \geq \rho \alpha m$ (otherwise $E'(T, T', h)$ cannot happen).
                Then, the difference in the errors $\err_{T}(h), \err_{T'}(h)$ is distributed as $X \sim \frac{\BinomD(k, 0.5)}{m}$.
                Consider two cases.
                Suppose $k < 10 \alpha m$, so that a standard Chernoff bound implies
                \begin{align*}
                    \Pr_{\sigma}(E(T, T', h)|S, S') &< \Pr(|X - \E[X]| > \rho \alpha m) \\
                    &= \Pr\left(|X - \E[X]| > \frac{2 \rho \alpha m}{k} \frac{k}{2} \right) \\
                    &< \exp\left( - \bigOm{ \frac{\rho^2 \alpha^2 m^2}{k}} \right) \\
                    &< \exp\left( - \bigOm{ \rho^2 \alpha m} \right) \text{.}
                \end{align*}
                On the other hand, when $k \geq 10 \alpha m$, we bound the probability that $\err_{T}(h) < \alpha$.
                By a Chernoff Bound, we have
                \begin{align*}
                    \Pr_{\sigma}(E(T, T', h)|S, S') < \Pr(\err_{T}(h) < \alpha) < \exp\left( - \bigOm{k} \right) < \exp\left( - \bigOm{\alpha m} \right) \text{.}
                \end{align*}
            \end{proof}

            Finally, for any fixed $S, S'$, we bound the probability that there exists any hypothesis $h$ with $E(T, T', h)$ holds.
            Now, for a fixed $S, S'$, we can apply Sauer's Lemma to union bound over all labellings of $S \cup S'$, rather than all hypotheses in $\hypotheses$.

            \begin{lemma}[Sauer's Lemma]
                \label{lem:sauer}
                Let $\hypotheses$ be a class with VC Dimension $d$ and $S \in \domain^{m}$.
                Let $\Pi(\hypotheses)$ denote the set of labellings of $S$ under $\hypotheses$.
                Then,
                \begin{align*}
                    |\Pi(\hypotheses)| = m^{O(d)} \text{.}
                \end{align*}
            \end{lemma}

            Then, for any $S, S'$, we have
            \begin{align*}
                \Pr_{\sigma}\left(B''(S, S', \sigma) | S, S' \right) &< m^{O(d)} e^{-\bigOm{\rho^2 \alpha m}} < \min(\beta, \rho)
            \end{align*}
            for $m \gg \frac{1}{\rho^2 \alpha} \left( d + \log(1/\min(\rho, \beta)) \right)$.
            This completes the proof of \Cref{lem:repl-reject}.
        \end{proof}

        We now continue with the proof of \Cref{thm:realizable-apx-repl}.
        Correctness is clear, since we proceed to Step 2 and return a hypothesis with empirical error at most $0.1 \alpha$ (since $\errS(h^*) < \rho \alpha < 0.1 \alpha$).
        We thus argue replicability.
        Conditioned on the success of Step 1, we immediately have replicability if the algorithm returns $x \rightarrow 1$.
        Thus, we assume that the algorithm proceeds to Step 2, and thus assume $\opt < \alpha$.
        In particular, there exists a hypothesis with $\errD(h) < \alpha$ so that $\errS(h) \leq 1.1 \alpha$.
        In two runs of the algorithm, we obtain $\hat{h}_{1}, \hat{h}_{2}$ both with empirical error at most $1.1 \alpha$.
        Following identical arguments as in \Cref{lem:repl-reject}, we can prove that both hypotheses have true distributional error at most $2 \alpha$.
        Then, we conclude with the triangle inequality
        \begin{equation*}
            \Pr(\hat{h}_{1}(x) \neq \hat{h}_{2}(x)) \leq \Pr(\hat{h}_{1}(x) \neq y) \leq 2 \alpha
        \end{equation*}
        since $\hat{h}_{1}(x) = \hat{h}_{2}(x)$ when both are correct.
\end{proof}

%% file: appendix_omitted.tex
\section{Omitted Proofs}
\label{sec:omitted}

We give the omitted proof of \Cref{thm:0-1-bias-estimation}.
The proof is a simple modification of the lower bound from \cite{hopkins2024replicability}.

\begin{proof}[Proof of \Cref{thm:0-1-bias-estimation}]
    We will construct a meta-distribution $\calM$ over biases $p \in [-1, 1]$ such that any $\rho$-replicable algorithm with error probability less than $0.01$ with respect to $\calM$ requires $\bigOm{\frac{1}{\rho^2}}$ samples.
    Let $\innerAlg$ be a $\rho$-replicable algorithm for $\set{\pm 1}$-bias estimation with succeeds with probability $0.9$ on $m$ samples.
    First, following \cite{hopkins2024replicability}, we fix a random string $r$ such that $\innerAlg(; r)$ has error probability at most $0.04$ and is $4\rho$-replicable with respect to $\calM$.
    For any bias $p$, let $\acc(p) := \Pr_{S \sim \Rad(p)^{m}}(\innerAlg(S; r) = 1)$ denote the acceptance probability of $\innerAlg$ given samples from $\Rad(p)$.
    Then, since $\acc(p)$ is a continuous function (see Claim 3.10 of \cite{hopkins2024replicability}) there exists $p^*$ such that $\acc(p^*) = \frac{1}{2}$.
    
    Our goal is now to prove that there is in fact a (sufficiently wide) interval $I^*$ where $\acc(p) \in (1/3, 2/3)$.
    To do so, we follow the mutual information framework of \cite{hopkins2024replicability}.
    First, we need a standard fact that says any correct algorithm must see samples correlated with the answer (see e.g. \cite{diakonikolaskane2016}).

    \begin{claim}
        \label{clm:mi-alg-lb}
        Let $X \sim \Rad(0)$ and $A$ be a random variable possibly correlated with $X$.
        If there exists a (randomized) function $f$ so that $f(A) = X$ with at least 51\% probability, then $I(X:A) \geq 2 \cdot 10^{-4}$.
    \end{claim}

    Next, we argue that samples from coins with similar bias are indistinguishable.

    \begin{claim}[Claim 3.9 of \cite{hopkins2024replicability}]
        \label{clm:sample-mi-ub}
        Let $m \geq 0$ be an integer and $-1 \leq a < b \leq 1$. Let $X \sim \Rad(0)$ and $Y$ be distributed according to $\Rad(a)^{m}$ if $X = 1$ and $\Rad(b)^{m}$ if $X = -1$.
        Then, 
        \begin{equation*}
            I(X:Y) = \bigO{\frac{m(b - a)^2}{\min\left(1+a, 1+b, 1- a, 1- b\right)}} \text{.}
        \end{equation*}
    \end{claim}

    Assume without loss of generality that $p^* \leq 0$.
    Our goal is to find an interval $I^* := I^*(p^*) = [p^*, q^*(p^*)]$ such that $|\acc(p^*) - \acc(q^*)| \leq 0.1$.
    Consider the following experiment: let $X, Y$ be distributed as in \Cref{clm:sample-mi-ub} where $a = p^*, b = q^*$.
    Suppose $|\acc(p^*) - \acc(q^*)| \geq 0.1$.
    Then, for some universal constant $C'$, there is a $C'm$ sample algorithm (that runs $\innerAlg(; r)$ on $C'$ fresh samples) which successfully guesses $X$ with 51\% probability. 
    We will choose $q^*$ such that no $C'm$ sample algorithm can guess $X$ with 51\% probability, so that $|\acc(p^*) - \acc(q^*)| \geq 0.1$.
    We proceed with case analysis on $p^*$.
    \begin{enumerate}
        \item $p^* = -1$.
        Note that $Y \sim \Rad(p^*)^{C' m}$ always yields $-1$ samples, while $Y \sim \Rad(q^*)^{C' m}$ yields all $-1$ samples with probability $\left(1 - \frac{q^* + 1}{2}\right)^{m} \geq \exp\left( - C'm(q^* + 1) \right) \geq 0.999$ for $q^* \leq -1 + \frac{1}{C'm}$ for some sufficient constant factor.

        \item $-1 < p^* \leq 0$.
        From \Cref{clm:sample-mi-ub}, we have
        \begin{equation*}
            I(X:Y) = \bigO{\frac{m(q^* - p^*)^2}{1 + p^*}} < 2 \cdot 10^{-4}
        \end{equation*}
        if $q^* \leq p^* + c\sqrt{\frac{1 + p^*}{m}}$ for some sufficiently small constant $c > 0$.
    \end{enumerate}
    Then, consider $\calM$ that samples uniformly from the following set of points $B$ defined $p_{0} = -1, p_{1} = \frac{c}{m} - 1$ and $p_{i} = p_{i} + c \sqrt{\frac{1 + p_{i}}{m}}$ for $i > 1$ where $c$ is a sufficiently small constant until $p_{k} > 0$ for some integer $k$.
    Observe that we have thus obtained $O(\sqrt{m})$ points $B$ that intersect $I^*(p^*)$ for any $p^* \in [-1, 0]$.
    To handle the case that $p^* > 0$, we additionally add all points $-p_{i}$ to $B$ (i.e. $1, 1 - \frac{c}{m}, \dots$).
    
    In particular, for any random string $r$, $\innerAlg(; r)$ is not $0.1$-replicable on the $\Rad(p)$ for $p \in I^*(p^*) \cap B$.
    Then, if we define $\calM$ to pick a bias uniformly from $B$, $\innerAlg(; r)$ is not $\frac{c}{\sqrt{m}}$-replicable with respect to $\calM$.
    Since we assumed $\innerAlg(; r)$ is $4\rho$-replicable, we obtain $m = \bigOm{\frac{1}{\rho^2}}$.
\end{proof}

\subsection{Domain Size Reduction}
\label{app:domain-red}

We show how to remove the assumption of finitely supported distributions for approximately replicable learners.


\begin{proof}[Proof of \Cref{prop:domain-reduction-apx-repl}]
    Suppose we have an arbitrary distribution $\distribution_0$ (possibly with uncountable support) over domain $\domain_0$.
    Let $N$ be some arbitrarily large integer (depending on $m, \alpha, \beta, \rho, \gamma$) where $m := m(\alpha, \beta, \rho, \gamma)$ is the sample complexity of our algorithm $\Acal$ that is a $(\rho, \gamma)$-pointwise replicable $(\alpha, \beta)$-learner over arbitrary finitely supported distributions over $\domain_{0}$.
    Consider the following algorithm $\tilde{\Acal}$:
    \begin{enumerate}
        \item Draw $N$ \iid samples from $\distribution_{0}$ denoted $\domain$.
        \item Run $\Acal$ on $m$ \iid samples from the uniform distribution over $\domain$ and return $h \gets \Acal$.
    \end{enumerate}
    As written, this algorithm has sample complexity $N$. However, we can reduce this to $m$ by drawing the samples in the first step `lazily', in the sense that we only draw samples from $\distribution_{0}$ when $\Acal$ requests them from $\domain$ in the second step. In other words, we can think of $\tilde A$ as sampling $m$ \textit{indices} $i_1, \dots, i_{m} \in [N]$, and then for each distinct $i_j$ sampling from $\mathcal D_0$ to fix the value of $\domain_{i_j}$. The resulting sample at these indices is then fed into $\mathcal A$ to produce the final hypothesis as above. Notably, the sample fed into $\mathcal A$ is equi-distributed with the algorithm that has fully sampled $\domain$ in the first step above, since our draws from $\distribution_0$ are $\iid$


    We first observe that as long as $N$ is sufficiently large, the produced sample is $\iid$ from $\distribution_{0}$ with high probability. In particular, this is true as long as $\tilde \Acal$ draws \emph{collision-free} indices, that is if $i_{j}, i_{k}$ are distinct for all $j \neq k$.

    \begin{claim}
        \label{clm:sample-simulation}
        Suppose $N > m^2/\min(\beta, \rho)$.
        With probability $1 - \min(\beta, \rho)$, $\Acal$ draws a collision free sample and obtains $m$ \iid samples from $\distribution_{0}$.
    \end{claim}

    \begin{proof}
        Note that for any $j \neq k$, we have $\Pr(i_j = i_k) = \frac{1}{N}$ so that linearity of expectation yields 
        \begin{equation*}
            \E\left[ \sum_{j \neq k} \ind[i_j = i_k] \right] \leq \frac{m^2}{N} \text{.}
        \end{equation*}
        Applying Markov's inequality, we have that the number of collisions exceeds $1$ with probability $\frac{m^2}{N} < \min(\beta, \rho)$ if $N > \frac{m^2}{\min(\beta, \rho)}$.
        In particular, $\Acal$ draws a collision-free sample, so all indices $i_1, \dots, i_m$ are independent.
        To conclude the proof, we note that every sample in $\domain$ is an \iid sample from $\distribution_{0}$.
    \end{proof}

    Now, we argue that $\tilde{\Acal}$ is a $(3 \rho, 3 \gamma)$-approximately replicable $(4 \alpha, 4 \beta)$-learner on $\distribution_{0}$.
    In the following, we condition on the event that $\Acal$ draws a collision-free sample and obtains $m$ \iid samples from $\distribution_{0}$.
    In fact, we will assume that two runs of $\Acal$ draw collision-free samples (this can be ensured by replacing $N > m^2/\min(\rho, \beta)$ with $4m^2/\min(\rho, \beta)$).
    In particular, for $N$ sufficiently large, we have that over two independent runs of $\Acal$ drawing indices $i_1, \dots, i_m$ and $i_{m + 1}, \dots, i_{2m}$, we have that $i_1, \dots, i_{2m}$ are all distinct.
    We begin with approximate replicability.

    \begin{claim}
        \label{clm:domain-red-apx-repl}
        For sufficiently large $N \gg \frac{m^2}{\min(\rho ,\beta)} + \frac{\log(1/\rho)}{\gamma^2}$, $\tilde{\Acal}$ is $(3 \rho, 3 \gamma)$-approximately replicable.
    \end{claim}

    \begin{proof}
        We use a hybrid argument to argue that $\tilde{\Acal}$ is $(3 \rho, 3 \gamma)$-approximately replicable.
        In particular, if $S^{(1)}, S^{(2)}$ are disjoint subsets of indices in $[N]$, then to each run of the algorithm $\tilde{\Acal}$, the following two random processes are indistinguishable:
        \begin{enumerate}
            \item Draw $N$ samples from $\distribution_{0}$ denoted $\domain^*$, give $\domain^*[S^{(1)}]$ and $\domain^*[S^{(2)}]$ to independent runs of $\Acal$.
            \item Draw two independent samples of size $N$ from $\distribution_{0}$, denoted $\domain^{(1)}, \domain^{(2)}$, give $\domain^{(1)}[S^{(1)}]$ and $\domain^{(2)}[S^{(2)}]$ to independent runs of $\Acal$.
        \end{enumerate}
        At a high level, our goal is to prove that the above two processes are in fact indistinguishable to $\Acal$.
        This will allow us to establish that the latter process, which is is exactly the process of running $\tilde{\Acal}$ on two independent samples, is indistinguishable to the former process, where we can conclude that $\Acal$ is approximately replicable on the uniform distribution over the shared $N$ sample domain $\domain^*$.
        We give a formal argument below. 
        
        Formally, let $\tilde{h}^{(1)}, \tilde{h}^{(2)}$ denote the hypotheses output over two runs of $\tilde{\Acal}$.
        Let $\domain^{(1)}, \domain^{(2)}$ denote the $N$ hypothetical samples drawn by two runs of $\tilde{\Acal}$.
        Let $\domain^*$ to be another independent $N$ \iid samples from $\distribution_{0}$.
        Consider the scenario that $\tilde{\Acal}$ over two runs draws a \emph{shared} $\domain^*$ and then runs $\Acal$ on independent collision-free sub-samples $\domain^*[S^{(1)}]$ and $\domain^*[S^{(2)}]$.
        We now claim that conditioned on $\tilde{\Acal}$ drawing collision-free samples over two runs, that for any two hypotheses $\tilde{h}^{(1)}, \tilde{h}^{(2)}$: 
        \begin{align}
            \label{eq:domain-red-equiv}
            &\Pr_{\substack{\domain^* \\ S^{(1)}, S^{(2)} \\ r}} \left( \tilde{\Acal}(\domain^*; S^{(1)}, r) = \tilde{h}^{(1)}\text{, } \tilde{\Acal}(\domain^*; S^{(2)}, r) = \tilde{h}^{(2)} \right) \\
            &\quad\quad\quad\quad\quad = \Pr_{\substack{\domain^{(1)}, \domain^{(2)} \\ S^{(1)}, S^{(2)} \\ r}} \left( \tilde{\Acal}(\domain^{(1)}; S^{(1)}, r) = \tilde{h}^{(1)}\text{, } \tilde{\Acal}(\domain^{(2)}; S^{(2)}, r) = \tilde{h}^{(2)} \right) \nonumber
        \end{align}
        where $\domain^*, \domain^{(1)}, \domain^{(2)}$ are independent draws of $N$ samples from $\distribution_{0}$, $S^{(1)}, S^{(2)}$ are independent (collision-free) samples of subsets of $[N]$, and $r$ is the shared internal randomness of $\Acal$.
        Before proving \eqref{eq:domain-red-equiv}, let us complete the proof of approximate replicability. 
        
        In particular, first observe that by \eqref{eq:domain-red-equiv}, it is enough to bound the probability the left-hand procedure that samples a shared domain $\domain^*$ produces hypotheses $\tilde h^{(1)},\tilde h^{(2)}$ that are $3\gamma$-close over $\distribution_0$ by $2\rho$ since this procedure is equi-distributed with the true right-hand procedure we'd like to analyze so long as $S^{(1)},S^{(2)}$ are collision-free, which occurs except with probability much less than $\rho$.

        To bound the approximate replicability of the procedure with shared $\domain^*$, we may now appeal to the fact that $\Acal$ is $(\rho,\gamma)$ approximately replicable. In particular, we have that the disagreement over $\domain^*$ is low with high probability:    
        \begin{equation*}
            \Pr_{\domain^*, S^{(1)}, S^{(2)}, r} \left( \Pr_{i \sim [N]}(\tilde{h}^{(1)}(x_i) \neq \tilde{h}^{(2)}(x_{i})) > \gamma \right) <  \rho,
        \end{equation*}
        where $x_i$ denotes the $i$-th example of $\domain^*$.
        
        To conclude the proof, we argue that disagreement over $\domain^*$ is close to disagreement over $\distribution_{0}$ with high probability.
        To see this, note that for the two output hypotheses $\tilde{h}^{(1)}, \tilde{h}^{(2)}$,
        \begin{align*}
            &Pr_{i \sim [N], \domain^*}(\tilde{h}^{(1)}(x_i) \neq \tilde{h}^{(2)}(x_{i})) = \frac{1}{N} \sum_{i = 1}^{N}  \sum_{i} \ind[\tilde{h}^{(1)}(x_i) \neq \tilde{h}^{(2)}(x_{i})] \\
            \\
            &\quad\quad\quad\quad= \frac{1}{N}  \sum_{i \in S^{(1)} \cup S^{(2)}} \ind[\tilde{h}^{(1)}(x_i) \neq \tilde{h}^{(2)}(x_{i})] + \frac{1}{N} \sum_{i \not\in S^{(1)} \cup S^{(2)}} \ind[\tilde{h}^{(1)}(x_i) \neq \tilde{h}^{(2)}(x_{i})] \\
            &\quad\quad\quad\quad= \frac{1}{N}  \sum_{i \in S^{(1)} \cup S^{(2)}} \ind[\tilde{h}^{(1)}(x_i) \neq \tilde{h}^{(2)}(x_{i})] + \frac{N - 2m}{N} \frac{1}{N - 2m} \sum_{i \not\in S^{(1)} \cup S^{(2)}} \ind[\tilde{h}^{(1)}(x_i) \neq \tilde{h}^{(2)}(x_{i})] 
        \end{align*}
        where the second summation is a sum of independent Bernoulli variables with mean $\dist_{\distribution_{0}}(\tilde{h}^{(1)}, \tilde{h}^{(2)})$ since $\tilde{h}^{(1)}, \tilde{h}^{(2)}$ depend only on $\domain^*[S^{(1)}], \domain^*[S^{(2)}]$ and $\domain^*[[N] \setminus (S^{(1)} \cup S^{(2)})]$ consists of $N - 2m$ \iid samples from $\distribution_{0}$.
        In particular, any standard concentration bound ensures that for $N$ sufficiently large, e.g. $N - 2m \geq \frac{C \log(1/\rho)}{\gamma^2}$ for some large constant $C$, ensures that with probability $1 - \rho$, the empirical disagreement over $\domain^*[[N] \setminus S^{(1)} \cup S^{(2)}]$ is at most $\gamma$ more than $\dist_{\distribution_{0}}(\tilde{h}^{(1)}, \tilde{h}^{(2)})$.
        Finally, a union bound allows us to conclude that with probability $1 - 3 \rho$,
        \begin{align*}
            \dist_{\distribution_{0}}(\tilde{h}^{(1)}, \tilde{h}^{(2)}) &\leq \frac{1}{N - 2m} \sum_{i \not\in S^{(1)} \cup S^{(2)}} \ind[\tilde{h}^{(1)}(x_i) \neq \tilde{h}^{(2)}(x_{i})] + \gamma \\
            &\leq \frac{\gamma}{1 - \frac{2m}{N}} + \gamma \leq 3 \gamma 
        \end{align*}
        if $N \geq 4m$, as desired.

        It thus remains to prove \eqref{eq:domain-red-equiv}.
        By definition of $\tilde{\Acal}$, we write
        \begin{align*}
             Z &:= \Pr_{\substack{\domain^* \\ S^{(1)}, S^{(2)} \\ r}} \left( \tilde{\Acal}(\domain^*; S^{(1)}, r) = \tilde{h}^{(1)}\text{, } \tilde{\Acal}(\domain^*; S^{(2)}, r) = \tilde{h}^{(2)} \right) \\
             &= \Pr_{\substack{\domain^* \\ S^{(1)}, S^{(2)} \\ r}} \left( \Acal(\domain^*[S^{(1)}]; r) = \tilde{h}^{(1)}\text{, } \Acal(\domain^*[S^{(2)}]; r) = \tilde{h}^{(2)} \right)
        \end{align*}
        since $\tilde{\Acal}$ returns the output of $\Acal$ on the subsample of $\domain^*$ indexed by $S^{(i)}$.
        Next, if we denote $\domain^* \uparrow_{i} S$ to denote replacing the indices of $\domain^*$ corresponding to $S \subset [n]$ with samples from $\domain^{(i)}$, then
        \begin{align*}
            Z &=  \underset{\substack{S^{(1)}, S^{(2)} \\ r}}{\E} \left[ \underset{\substack{ \domain^* \\ \domain^{(1)}, \domain^{(2)}}}{\E} \left[ \ind\left[ \Acal(\domain^*[S^{(1)}]; r) = \tilde{h}^{(1)}\text{, } \Acal(\domain^*[S^{(2)}]; r) = \tilde{h}^{(2)} \right] \right] \right] \\
            &= \underset{\substack{\domain^* \\ S^{(1)}, S^{(2)} \\ r}}{\E} \left[ \Pr_{\domain^{(1)}, \domain^{(2)}} \left( \Acal(\domain^*[S^{(1)}]; r) = \tilde{h}^{(1)}\text{, } \Acal(\domain^*[S^{(2)}]; r) = \tilde{h}^{(2)} \right) \right] \\
            &= \underset{\substack{\domain^* \\ S^{(1)}, S^{(2)} \\ r}}{\E} \left[ \Pr_{\domain^{(1)}, \domain^{(2)}} \left( \Acal((\domain^* \uparrow_1 [N] \setminus S^{(1)})[S^{(1)}]; r) = \tilde{h}^{(1)}\text{, } \Acal(\domain^*[S^{(2)}]; r) = \tilde{h}^{(2)} \right) \right] \\
            &= \underset{\substack{\domain^* \\ S^{(1)}, S^{(2)} \\ r}}{\E} \left[ \Pr_{\domain^{(1)}, \domain^{(2)}} \left( \Acal((\domain^* \uparrow_1 [N] \setminus S^{(1)})[S^{(1)}]; r) = \tilde{h}^{(1)}\text{, } \Acal((\domain^* \uparrow_2 [N] \setminus S^{(2)})[S^{(2)}]; r) = \tilde{h}^{(2)} \right) \right] \\
        \end{align*}
        where the first equality follows from the total probability rule as neither event depends on $\domain^{(1)}, \domain^{(2)}$ and the second from writing expectation of indicator as probability. 
        The third equality follows from the total probability rule as the same output $\tilde{h}^{(1)}$ is produced for every choice of $\domain^{(1)}$ (conditioned on a fixed choice of $\domain^*, S^{(1)}$).
        The last equality follows from a symmetric argument applied to $\tilde{h}^{(2)}$.

        Next, note that if we replace $(\domain^* \uparrow_{1} [N] \setminus S^{(1)})$ with $(\domain^* \uparrow_{1} [N] \setminus S^{(1)}) \uparrow_{1} S^{(1)} = \domain^{(1)}$, then the distribution of samples seen by $\Acal$ stays the same, as $\domain^*, \domain^{(1)}$ are identically distributed as \iid draws from $\distribution_{0}$.
        Furthermore, this replacement does not affect the output $\tilde{h}^{(2)}$.
        Thus, this replacement maintains the output distribution, or
        \begin{align*}
            Z &= \underset{\substack{\domain^* \\ S^{(1)}, S^{(2)} \\ r}}{\E} \left[ \Pr_{\domain^{(1)}, \domain^{(2)}} \left( \Acal(\domain^{(1)}[S^{(1)}]; r) = \tilde{h}^{(1)}\text{, } \Acal((\domain^* \uparrow_2 [N] \setminus S^{(2)})[S^{(2)}]; r) = \tilde{h}^{(2)} \right) \right] \\
            &= \underset{\substack{\domain^* \\ S^{(1)}, S^{(2)} \\ r}}{\E} \left[ \Pr_{\domain^{(1)}, \domain^{(2)}} \left( \Acal(\domain^{(1)}[S^{(1)}]; r) = \tilde{h}^{(1)}\text{, } \Acal(\domain^{(2)}[S^{(2)}]; r) = \tilde{h}^{(2)} \right) \right] 
        \end{align*}
        where the second equality follows from a similar argument for replacing $(\domain^* \uparrow_{2} [N] \setminus S^{(2)})$ with $\domain^{(2)}$.
        Finally, we again apply the total probability rule to remove the expectation over $\domain^*$ which neither event now depends on, thus establishing
        \begin{align*}
            Z &= \Pr_{\substack{\domain^{(1)}, \domain^{(2)} \\ S^{(1)}, S^{(2)} \\ r}} \left( \Acal(\domain^{(1)}[S^{(1)}]; r) = \tilde{h}^{(1)}\text{, } \Acal(\domain^{(2)}[S^{(2)}]; r) = \tilde{h}^{(2)} \right) \\
            &= \Pr_{\substack{\domain^{(1)}, \domain^{(2)} \\ S^{(1)}, S^{(2)} \\ r}} \left( \tilde{\Acal}(\domain^{(1)}; S^{(1)}, r) = \tilde{h}^{(1)}\text{, } \tilde{\Acal}(\domain^{(2)}; S^{(2)}, r) = \tilde{h}^{(2)} \right)
        \end{align*}
        proving \eqref{eq:domain-red-equiv}, where we use the definition of $\tilde{\Acal}$ in the final equality.
    \end{proof}

    Finally, we argue that $\tilde{\Acal}$ is correct.

    \begin{claim}
        \label{clm:domain-red-pac}
        For sufficiently large $N \gg \frac{m^2}{\min(\rho, \beta)} + \frac{m}{\alpha} + \frac{\log(1/\beta)}{\alpha^2}$, $\tilde{\Acal}$ is an $(4 \alpha, 4 \beta)$-learner.
    \end{claim}

    \begin{proof}
        By a union bound over \Cref{clm:sample-simulation} and the correctness of $\Acal$ over finitely supported distributions, with probability $1 - 2 \beta$, $\Acal$ produces hypothesis $\tilde{h}$ such that
        \begin{equation*}
            \Pr_{i \sim [N], \domain} \left( \tilde{h}(x_i) \neq y_i \right) < \min_{h \in \hypotheses} \err_{\domain}(h) + \alpha \text{.}
        \end{equation*}
        Note that for any hypothesis $h$,
        \begin{align*}
            \err_{\domain}(h) &= \Pr_{i \sim [N]} \left( h(x_i) \neq y_i \right) \\
            &= \frac{1}{N} \sum_{i = 1}^{N} \Pr(h(x_i) \neq y_i) \\
            &= \frac{1}{N} \sum_{i \in S} \Pr(h(x_i) \neq y_i) + \frac{1}{N} \sum_{i \not\in S} \Pr(h(x_i) \neq y_i) \\
            &= \frac{1}{N} \sum_{i \in S} \Pr(h(x_i) \neq y_i) + \frac{N - m}{N} \frac{1}{N - m} \sum_{i \not\in S} \Pr(h(x_i) \neq y_i) \text{.}
        \end{align*}
        Here, for any \emph{fixed} hypothesis, we observe that the empirical error is a sum of $N$ independent Bernoulli variables with mean $\err_{\distribution_{0}}(h)$, while for $\tilde{h}$, the second summand in the final line is a sum of independent Bernoulli variables since $\domain[[N] \setminus S]$ consists of \iid samples drawn from $\distribution_{0}$ independent of $\tilde{h}$.
        
        First, we observe that by uniform convergence (\Cref{thm:uniform-convergence}) there exists sufficiently large $N$ such that we have that for every $h \in \hypotheses$,
        \begin{equation*}
            |\err_{\domain}(h) - \err_{\distribution_{0}}(h)| \leq \alpha
        \end{equation*}
        with probability $1 - \beta$.

        Second, as in \Cref{clm:domain-red-apx-repl}, we observe that for the output hypothesis $\tilde{h}$, for sufficiently large $N \gg \frac{\log(1/\beta)}{\alpha^2}$, we have with probability $1 - \beta$,
        \begin{align*}
            \err_{\distribution_{0}}(\tilde{h}) &\leq \frac{1}{N - m} \sum_{i \not\in S} \Pr(\tilde{h}(x_i) \neq y_i) + \alpha  \\
            &\leq \frac{\err_{\domain}(\tilde{h})}{1 - \frac{m}{N}} + \alpha \\
            &\leq \err_{\domain}(\tilde{h}) \left(1 + \frac{2m}{N} \right) + \alpha \\
            &\leq \err_{\domain}(\tilde{h}) + 2 \alpha 
        \end{align*}
        where we use $\frac{1}{1 - x} \leq 1 + 2x$ for $x \leq 0.5$ and $N \geq 2m$ in the third inequality, and $\err_{\domain}(\tilde{h}) \frac{2m}{N} \leq \frac{2m}{N} \leq \alpha$ in the fourth inequality.
        Thus, by a union bound, we have with probability $1 - 4 \beta$, that
        \begin{align*}
            \err_{\distribution_{0}}(\tilde{h}) &\leq \err_{\domain}(\tilde{h}) + 2 \alpha \\
            &\leq \min_{h \in \hypotheses} \err_{\domain}(h) + 3 \alpha \\
            &\leq \min_{h \in \hypotheses} \err_{\distribution_{0}}(h) + 4 \alpha \\
            &\leq \opt + 4 \alpha \text{.}
        \end{align*}
    \end{proof}
    This concludes the proof of \Cref{prop:domain-reduction-apx-repl}.
\end{proof}

\subsection{Concentration Inequalities}
\label{app:concentration}

We use the following standard concentration inequalities.

\begin{theorem}[Hoeffding's Inequality]
    \label{thm:hoeffding}
    Let $X_1, \dots , X_{n}$ be independent random variables in $[0, 1]$.
    Let $S = X_1 + \dots + X_n$ and $\mu = \E[S]$.
    Then for any $t > 0$,
    \begin{equation*}
        \Pr(|X - \mu| > t) < \exp(-t^2/2n) \text{.}
    \end{equation*}
\end{theorem}

\begin{restatable}[Chernoff Inequality]{theorem}{ChernoffScaled}
    \label{thm:chernoff}
    Let $\gamma > 0$.
    Let $X_1, \dots , X_{n}$ be independent random variables in $[0, \gamma]$.
    Let $S = X_1 + \dots + X_n$ and $\mu = \E[S]$.
    Then for any $0 < t < 1$,
    \begin{equation*}
        \max(\Pr(X < (1 - t) \mu), \Pr(X > (1 + t)\mu)) < \exp\left( - \frac{t^2 \mu}{3 \gamma} \right) \text{.}
    \end{equation*}
    Furthermore, for $t \geq 1$,
    \begin{equation*}
        \Pr(X > (1 + t)\mu) 
        < \exp \left( - \frac{t^2 \mu}{\gamma (2 + t)} \right) <  \exp \left( - \frac{t \mu}{3 \gamma} \right) \text{.}
    \end{equation*}
\end{restatable}

We prove \Cref{thm:chernoff} with a standard modification of the standard Chernoff bound.

\begin{proof}[Proof of \Cref{thm:chernoff}]
    We restate the standard Chernoff inequality.
    \begin{theorem}[Chernoff Inequality]
        \label{thm:chernoff-standard}
        Let $\gamma > 0$.
        Let $X_1, \dots , X_{n}$ be independent random variables in $[0, 1]$.
        Let $X = X_1 + \dots + X_n$ and $\mu = \E[S]$.
        Then for any $0 < t < 1$,
        \begin{equation*}
            \max(\Pr(X < (1 - t) \mu), \Pr(X > (1 + t)\mu)) < \exp\left( - \frac{t^2 \mu}{3} \right) \text{.}
        \end{equation*}
        Furthermore, for all $t > 0$,
        \begin{equation*}
            \Pr(X > (1 + t)\mu)) < \left( \frac{e^{t}}{(1 + t)^{1 + t}} \right)^{\mu} <  \exp\left( - \frac{t^2 \mu}{2 + t} \right) \text{.}
        \end{equation*}
    \end{theorem}
    Now, if $X_{1}, \dots, X_{n}$ are in $[0, \gamma]$, consider $Y_{i} = X_i/\gamma$ for $i \in [n]$ so that $Y_i \in [0, 1]$ and 
    \begin{equation*}
        \E[Y] = \E[X]/\gamma = \mu/\gamma \text{.}
    \end{equation*}
    Note that $X < (1 - t) \mu$ if and only if $Y < (1 - t) \mu / \gamma$, as shown below:
    \begin{align*}
        X = X_1 + \dots + X_{n} &< (1 - t) \mu \\
        (X_{1} + \dots + X_{n}) / \gamma &< (1 - t) \mu / \gamma \\
        Y_{1} + \dots + Y_{n} &< (1 - t) \E[Y] \text{.}
    \end{align*}
    A similar argument shows $X > (1 + t) \mu$ if and only if $Y > (1 + t) \mu / \gamma$. 
    Thus, we apply \Cref{thm:chernoff-standard} with $\E[Y] = \mu / \gamma$ to conclude the desired bound.
\end{proof}

We use the following concentration bound on countable sets.

\CountableTail*

\begin{proof}[Proof of \Cref{thm:countable-tail-bound}]
    We follow the proof of the standard Hoeffding inequality.
    Fix a $\lambda > 0$ and observe
    \begin{equation*}
        \E\left[ e^{\lambda(S - \mu)} \right] = \E\left[ \exp \left( \sum_{i = 1}^{\infty} \lambda X_i - \lambda \E[X_i]\right) \right] = \E\left[ \prod_{i = 1}^{\infty} \exp \left( \lambda (X_i- \E[X_i]) \right) \right] \text{.}
    \end{equation*}
    Since $X_{i} \in [0, b_i]$, Hoeffding's Lemma implies that for each $i$,
    \begin{equation*}
        0 \leq \exp \left( \lambda (X_i - \E[X_i]) \right) \leq \exp(\lambda^2 b_i^2/8) \text{.}
    \end{equation*}
    In particular, since the partial products are bounded by the constant $\exp(\lambda \sum_{i} b_i) < \exp(\lambda^2/8)$, we can exchange the limit and the expectation to obtain
    \begin{equation*}
        \lim_{n \rightarrow \infty} \E\left[ \prod_{i = 1}^{n} \exp \left( \lambda (X_i- \E[X_i]) \right) \right] = \lim_{n \rightarrow \infty} \prod_{i = 1}^{n} \E\left[ \exp \left( \lambda (X_i- \E[X_i]) \right) \right] < \prod_{i = 1}^{n} \exp\left( \lambda^2 b_i^2/8\right) \text{.}
    \end{equation*}
    where we use the fact that each $X_i$ is independent.
    Applying Markov's inequality, we have
    \begin{equation*}
        \Pr(S - \mu > t) = \Pr \left( e^{\lambda (S - \mu)} > e^{\lambda t} \right) < \frac{\exp\left(\lambda^2 \sum_{i} b_i^2/8\right)}{\exp(\lambda t)} \text{.}
    \end{equation*}
    Set $\lambda = \frac{4t}{\sum_{i} b_i^2}$ so that
    \begin{equation*}
        \Pr(S - \mu > t) < \exp \left( - \frac{2 t^2}{\sum_{i = 1}^{\infty} b_i^2} \right) \text{.}
    \end{equation*}
    Finally, we upper bound $\sum_{i = 1}^{\infty} b_i^2 \leq \nu \sum_{i = 1}^{\infty} b_i < \nu$ to obtain the desired bound.
    A similar bound can be derived for $\Pr(S - \mu < -t)$.
\end{proof}